\documentclass{article}
\usepackage[margin = 1in]{geometry}
\usepackage[utf8]{inputenc} 
\usepackage[T1]{fontenc}    
\usepackage{hyperref}       
\usepackage{url}            
\usepackage{booktabs}       
\usepackage{amsfonts}       
\usepackage{nicefrac}       
\usepackage{microtype}      
\usepackage{xcolor}         
\usepackage{xcolor} 
\usepackage{ulem}

\usepackage{enumerate}

\usepackage{hyperref}

\usepackage{amsmath}
\usepackage{amsthm}
\usepackage{amssymb}
\usepackage{amsfonts}
\usepackage{bbm}
\usepackage{graphicx} 
\usepackage{wrapfig}
\usepackage{empheq}
\usepackage{caption}
\usepackage{mathrsfs}
\newcommand{\CMscr}{}
\let\CMscr=\mathscr
\usepackage[mathscr]{euscript}
\usepackage[title, toc]{appendix}
\DeclareSymbolFontAlphabet{\mathrsfs}{rsfs}

\DeclareMathAlphabet{\mathcalligra}{T1}{calligra}{m}{n}

\DeclareMathOperator{\Tr}{Tr}

\newtheorem{theorem}{Theorem}%
\newtheorem{lemma}{Lemma}%
\newtheorem{proposition}[theorem]{Proposition}%
\newtheorem{remark}{Remark}%
\newtheorem{definition}{Definition}
\newtheorem{assumption}{Assumption}
\newtheorem{corollary}{Corollary}

\def\N{\mathbb{N}}
\def\R{\mathbb{R}}

\def \a {\alpha}
\def \b {\beta}

\def\dif{\mathop{}\!\mathrm{d}}

\def\EE{{\mathbb E}\,}

\def\defas{\stackrel{\text{def}}{=}}
\def\eqd{\stackrel{d}{=}}

\DeclareDocumentCommand{\Prto} {o} {
  \IfNoValueTF {#1}
  {\overset{\Pr}{\longrightarrow}}
  { \xrightarrow[ #1 \to \infty]{\Pr }}
}

\DeclareDocumentCommand{\Asto} {o} {
  \IfNoValueTF {#1}
  {\overset{\text{\rm a.s.}}{\longrightarrow}}
  { \xrightarrow[ #1 \to \infty]{\text{\rm a.s.} }}
}

\DeclareDocumentCommand{\law} {o} {
  \IfNoValueTF {#1}
  {\overset{\text{law}}{=}}
  { \xrightarrow[ #1 \to \infty]{\Pr }}
}

\newcommand{\scrD}{\mathscr{D}}

\newcommand{\ls}{{\ell^{\star}}}
\newcommand{\lp}{{\ell^{+}}}
\newcommand{\lpp}{{\bar{\ell}}}
\DeclareMathOperator{\E}{\mathbb{E}}

\newcommand{\e}{\varepsilon}
\newcommand{\cA}{\mathcal{A}}

\newcommand{\cE}{\mathcal{E}}
\newcommand{\cF}{\mathcal{F}}
\newcommand{\cG}{\mathcal{G}}

\newcommand{\cK}{\mathcal{K}}

\newcommand{\cM}{\mathcal{M}}
\newcommand{\cV}{\mathcal{V}}
\newcommand{\cN}{\mathcal{N}}
\newcommand{\cR}{\mathcal{R}}
\newcommand{\cL}{\mathcal{L}}

\def\cm{\mathcalligra{m}}
\def\csM{\mathscr{M}}
\def\csN{\mathscr{N}}
\def\csB{\mathscr{B}}
\def\csD{\mathscr{D}}
\def\csF{\mathscr{F}}
\def\csS{\mathscr{S}}

\def\csZ{\mathscr{Z}}
\def\csX{\mathscr{X}}

\def\csQ{\mathscr{Q}}
\def\csV{\CMscr{V}}
\def\csL{\CMscr{L}}
\def\csU{\CMscr{U}}
\def\csm{\mathfrak{m}}
\def\r{\theta}
\def\w{\mathrm{w}}

\def\htQ{{\tau}_Q}

\DeclareMathOperator{\grad}{Grad}
\DeclareMathOperator{\hess}{Hess}
\newcommand{\beq}{ \begin{equation} }
\newcommand{\eeq}{ \end{equation} }

\newcommand{\ip}[1]{\langle {#1} \rangle }
\newcommand{\ifl}[1]{\lfloor {#1} \rfloor }
\newcommand{\ipa}[1]{\left\langle {#1} \right\rangle }

\newcommand{\tr}{\text{Tr}}

\newcommand{\vertiii}[1]{{\left\vert\kern-0.25ex\left\vert\kern-0.25ex\left\vert #1 
    \right\vert\kern-0.25ex\right\vert\kern-0.25ex\right\vert}}

\def\Dif{\mathop{}\!\mathrm{D}}
\newcommand{\op}{ \text{op} }

\graphicspath{{Fig/}}

\theoremstyle{thmstyletwo}%

\numberwithin{equation}{section}

\begin{document}


\title
{Exact Dynamics of Multi-class Stochastic Gradient Descent}

\author{ Elizabeth Collins-Woodfin\ORCID{0000-0002-4018-0769}
\address{\orgdiv{Department of Mathematics}, \orgname{University of Oregon}, 
\state{Eugene}, \country{OR, USA}}}
\author{Inbar Seroussi\ORCID{0000-0001-5209-5839}
\address{\orgdiv{School of Mathematical Science and Computer Science}, \orgname{Tel-Aviv University}, 
\country{Israel}}}
\author{Elizabeth Collins-Woodfin and Inbar Seroussi}

\maketitle
\begin{abstract}
 We develop a framework for analyzing the learning dynamics of high-dimensional problems trained using one-pass stochastic gradient descent (SGD) with data from multiple anisotropic classes. Our main theorem provides exact expressions for quantities of interest, including the risk and the overlap with the true signal, in terms of a deterministic system of ODEs, valid in the high-dimensional limit. 
The theorem holds for a broad class of optimization problems and extends to settings where the number of classes grows with dimension. To illustrate its utility, we investigate in detail the effect of the data's anisotropic structure on the problems of binary logistic regression and least-squares (LS) loss. We study the LS in a linear multiclass setup and derive a learning-rate threshold that depends on the average eigenvalue of the covariance matrices. In the binary logistic regression, we study three cases: isotropic covariances, data covariance matrices with a large fraction of zero eigenvalues (denoted as the \textit{zero-one model}), and covariance matrices with power-law spectra. We show that a structural phase transition occurs. In particular, for the zero-one model and the power-law model with sufficiently large power, SGD aligns more closely with values of the class mean that are projected onto the ``clean directions'' (i.e., directions of smaller variance). This is supported by analytical studies and numerical simulations, which show the exact asymptotic behavior of the loss in the high-dimensional limit. The effects of data anisotropy that we demonstrate are likely to hold beyond these examples and illustrate one application of the broader theorem that we prove. 
\end{abstract}


\section{Introduction}
Stochastic optimization algorithms are applied to high-dimensional data and parameter spaces. Understanding the dynamics and generalization of these algorithms remains a central challenge in modern machine learning. While significant progress has been made on isotropic Gaussian data distributions, real-world datasets exhibit substantially more structure: they are comprised of multiple classes, each with anisotropic noise. This anisotropy and non-Gaussianity can significantly affect algorithm convergence and stability, yet most theoretical analyses either assume spherical Gaussians or consider only a small number of classes with isotropic covariance structure.

In this work, we study the dynamics of stochastic gradient descent (SGD) on risk minimization problems over Gaussian mixture models (GMMs) with anisotropic class-wise covariance structures. Our framework handles multiple classes, including a number of classes that can grow with input dimension, and data with non-zero mean, addressing more realistic datasets.

GMMs are particularly well-suited for high-dimensional analysis due to their mathematical tractability and surprising universality. GMM universality in high-dimension means that the performance of many algorithms (like generalized linear models) depends asymptotically only on the first- and second-order moments of the data; see for example \cite{dandi2024universality}.  Recent work has shown that the Gaussian mixture assumption captures key behaviors across diverse machine learning tasks. \cite{loureiro2021learning, dandi2024universality} demonstrated that empirical learning curves follow GMM behavior even for non-Gaussian data. \cite{seddik2020random} showed that deep learning representations in generative adversarial nets (GANs) can be characterized by their first two moments for a wide range of classifiers in high dimensions. Furthermore, GMMs serve as effective toy models for understanding the training dynamics of neural network final layers through the neural collapse phenomenon \cite{papyan2020prevalence}, where representations from the penultimate layer often behave like linearly separable Gaussian mixtures.

Despite their widespread use, most theoretical studies of GMMs are restricted to isotropic covariances and a fixed, small number of classes. This limits their applicability to real datasets, which typically exhibit anisotropic structure within each class and may contain many classes. Understanding how SGD behaves under these more realistic conditions is essential for bridging the gap between theory and practice.

\paragraph{Main contributions} 
\begin{itemize}
    \item We study the high-dimensional limit of stochastic gradient descent (SGD) in the proportional regime where the feature dimension $d$ is linearly proportional to the number of samples $n$ for data generated from a multi-class Gaussian mixture model with general covariance and mean. We derive a closed set of equations for the norm of the iterates and their inner products (overlaps) with the class means. This allows us to predict the learning curve of SGD for general means and covariances. 
    
    \item We show that the deterministic equivalent of the SGD limit is also valid in the setting in which the class number grows logarithmically with the feature dimension ($\ls  = O(\log (d))$). 
    \item We study the effect of different covariance matrix structures and means on the dynamics of SGD in the setting of binary logistic regression. In particular, we provide exact asymptotics of the loss and overlap as a function of the number of samples for three different models. We quantify how the rate of decay depends on the learning rate and the structure of the covariance matrix.
    \item In models with power-law decay on the covariance spectrum, we identify a structural phase transition with respect to the power-law exponent in binary logistic regression. We suspect that this phase transition holds in a larger generality for other models as well.     
\end{itemize}
An important technical challenge addressed in our paper 
is obtaining the limiting dynamics in a setting where a finite set of summary statistics is not sufficient to close the equations.  In many high-dimensional settings, the dynamics are actually controlled by a finite set of scalar observables (e.g. projections of the learned parameters to certain ground truth directions).  Examples of settings that require only a finite set of such observables include online SGD with isotropic Gaussian data or finite data distributions.  In contrast, when the data have a more complicated covariance structure (e.g. power law), a finite set of scalar observables is not sufficient.  Rather, one would need projections of the parameters to all Krylov subspaces generated by the covariance matrix with the mean and parameter vectors.  The method developed in \cite{collinswoodfin2023hitting}, and extended here to multi-class problems, uses resolvents to encode this infinite set of scalar observables as a finite set of function-valued observables.  The technical details are explained later, but this is a crucial tool that enables us to analyze the affects of different covariance structures on SGD dynamics.

\subsection*{Related work}
\paragraph{Deterministic dynamics of stochastic algorithms in high-dimensions.}
The literature on  
deterministic dynamics 
for isotropic Gaussian data, has a long history \cite{saad1995dynamics,biehl1994line, biehl1995learning, saad1995exact}.  
These results have been rigorously proven and extended to other models under the isotropic Gaussian assumption \cite{goldt2019dynamics,wang2019solvable,arnaboldi2023highdimensional, arous2022high} and have been used to study sample complexity in some high-dimensional problems \cite{arous2021online,Bruno,damian2023smoothing,dandi2024benefits}. Extensions to
multi-pass SGD with small mini-batches \cite{PPAP01} as well as momentum \cite{LeeChengPaquettePaquette} have also been studied. 
Other works have studied high-dimensional limits 
from the perspective of dynamical mean field theory and related methods \cite{mignacco2020dynamical,gerbelot2022rigorous,celentano2021highdimensional,chandrasekher2021sharp,bordelon2022learning}.

High-dimensional analysis of streaming SGD for multi-class Gaussian data with isotropic covariance was done in \cite{arous2022high,arous2025local,refinetti2021classifying}, in which they also study the geometry and the spectral properties of the Hessian \cite{arous2025local}. 
In \cite{mignacco2020dynamical}, the authors study GMM with isotropic covariance in multi-pass SGD in the proportionate batch setting using a physics technique known as dynamical mean field theory. 

Recently, significant contributions have been made in understanding the effects of data of one class with a non-identity covariance matrix on the training dynamics \cite{CollinsWoodfinPaquette01,balasubramanian2023high,goldt2022gaussian,Yoshida, Goldt,collinswoodfin2023hitting}. 
The non-identity covariance modifies the optimization landscape and affects convergence properties, as discussed in \cite{collinswoodfin2023hitting}. This work extends the findings of \cite{collinswoodfin2023hitting} to Gaussian mixture data with non-identity covariance and mean, examining the impact of non-identity covariance within and between classes on the dynamics of SGD.

\paragraph{Information theoretic bounds and other estimators.} 

Information-theoretic bounds for Gaussian mixtures have been studied in both
supervised classification, where training labels are observed, and
unsupervised clustering, where labels are unknown. In supervised
high-dimensional Linear Discriminant Analysis (LDA), \cite{cai2019high}
derives minimax lower bounds for the excess misclassification risk relative
to the oracle Fisher rule. Related minimax bounds for Gaussian classification
were obtained in the isotropic setting by \cite{li2017minimax}, and more
recently for binary anisotropic high-dimensional mixtures with shared
covariance by \cite{minsker2025classification}. In the unsupervised setting,
minimax misclustering rates are known for isotropic mixtures
\cite{azizyan2013minimax,loffler2021optimality} and anisotropic mixtures
\cite{chen2024achieving,huang2025minimax}; see also
\cite{lesieur2016phase} for phase transitions in high-dimensional Gaussian
mixture clustering. These works provide statistical benchmarks for
classification or label recovery, whereas our focus is on the dynamics of a
specific training algorithm.

Closest to our setting are works deriving precise high-dimensional asymptotics
for estimators trained on Gaussian-mixture data. In the binary setting,
\cite{mai2019high} analyzes regularized and unregularized empirical risk
minimization for logistic and square losses under a shared positive-definite
covariance. Related precise asymptotics for ridge regression and regularized
discriminant analysis with general feature covariance were obtained in
\cite{dobriban2018high}. For multiclass models,
\cite{thrampoulidis2020theoretical} studies linear classifiers in the
proportional regime with a fixed number of classes and isotropic Gaussian
features, while \cite{loureiro2021learningGMM} studies convex empirical risk minimization (ERM) for
$K$-class Gaussian mixtures with positive-definite class covariances, with
explicit simplifications in the jointly commuting case. Compared with
these works, we analyze the supervised one-pass SGD trajectory itself, rather
than the final ERM estimator, and we allow nonconvex losses,
class-dependent possibly singular commuting covariance matrices, and a number
of classes growing with the input dimension.

\paragraph{Notation}
\begin{itemize}
\item We say an event holds \textit{with overwhelming probability, w.o.p.,} 
if there is a function $\omega: \N \to \R$ with $\omega(d)/\log d \to
\infty $ so that the event holds with probability at least
$1-e^{-\omega(d)}$. 
\item 
For  normed vector spaces $\mathcal{A}$, $\mathcal{B}$  with norms $\|\cdot\|_{\mathcal{A}}$ and $\|\cdot\|_{\mathcal{B}}$, respectively, and for $\alpha \geq 0$, we say a function $F \, : \, \mathcal{A} \to \mathcal{B}$ is \textit{$\alpha$-pseudo-Lipschitz} ($\alpha$-PL) with constant $L$ if for any $A, \hat{A} \in \mathcal{A}$, we have 
\[
\|F(A) - F(\hat{A})\|_{\mathcal{B}} \le L\|A-\hat{A}\|_{\mathcal{A}} (1 + \|A\|_{\mathcal{A}}^{\alpha} + \|\hat{A}\|_{\mathcal{A}}^{\alpha} ).
\]
\item We write $f(t) \asymp g(t)$ if there exist \textit{absolute} constants $C, c > 0$ such that $c  g(t) \le f(t) \le C  g(t)$ for all $t$. 
\end{itemize}
\paragraph{Tensor inner products and norms.}
We briefly review our notation for tensor computations and refer the reader to Section 3 of \cite{collinswoodfin2023hitting} for a comprehensive overview of tensors in this context.  Given vector spaces $\cV_1,\cV_2$ equipped with inner products $\ip{\cdot,\cdot}_{\cV_1},\ip{\cdot,\cdot}_{\cV_2}$ and given $a_1,a_2\in\cV_1$ and $b_1,b_2\in\cV_2$, we define the inner product of simple tensors by
\begin{equation}
    \ip{a_1\otimes b_1,\; a_2\otimes b_2}_{\cV_1\otimes\cV_2}\defas\ip{a_1,a_2}_{\cV_1}\ip{b_1,b_2}_{\cV_2}
\end{equation}
and this is extended to an inner product on all tensor by bilinearity.  For higher tensors, we define a partial contraction $\ip{A,B}_{\cV}$ where $A$ and $B$ are contracted along their first $\cV$ axis and the resulting dimension has the shape of the uncontracted axes of $A$ followed by those of $B$.  For example, given $A,B\in\cV_1\otimes\cV_2$, we have
\begin{equation}
    \ip{A,B}_{\cV_1}=A^\top B\in\cV_2^{\otimes2},\quad
    \ip{A,B}_{\cV_2}=AB^\top \in\cV_1^{\otimes2},\quad
    \ip{A,B}_{\cV_1\otimes\cV_2}=\Tr(A^\top B).
\end{equation}
When the subscript is omitted, it indicates contraction over the full space.  Finally, we use $\|\cdot\|$ and $\|\cdot\|_{\text{op}}$ to denote Hilbert-Schmidt norm and operator norm respectively. Given a function $f: \R_{+}\to \R_{+} $ we denote the $L_1$ norm of the function $\|f\|_1 = \int_0^\infty  f(x)\dif x$.

\subsection{Model Set-up}

We consider feature input data $a\in\R^d$ that comes from a Gaussian mixture with $\ell^\star$ number of classes.
At the training stage, each data point comes to us with its 
target $y\in \R^m$ that encodes the class label. A data point belongs to the class $i\in[\ls]$ with probability $p_i,$ such that $\sum_{i=1}^\ls p_i = 1.$ The number of classes $\ls$ can either be constant or $O(\log (d)).$
We consider the optimization problem 
\begin{align}
\min_{X\in\R^{d\times \ell}} \{\cL(X) = \E_{(a,y)}[l(X^\top a;y)]\}.
\end{align}
where $l: \R^\ell \times \R^m \to \R$ is the loss function and $X\in \R^{d\times \ell}$ is the matrix of learned parameters.

We now introduce the key assumptions for our model.

\begin{assumption}[Data features]
\label{ass:data}
 The distribution of a data point, conditioned on being in class $i$, is $a \mid  i \sim \mathcal{N}(\mu_i, K_i)$. The covariance matrices $K_i \in \mathbb{R}^{d \times d}$ are bounded in $\ell_2$ operator norm independent of $d$ (i.e., $\| K_i \|_{\textup{op}} \le C$); and $\{K_i\}$ commute. In addition, we scale $\mu_i$ such that $\sum_{i=1}^\ls p_i\|\mu_i\|^2\le A$ for some constant $A>0$.  In the case of growing $\ls$, we also require the means to be orthogonal.
\end{assumption}

We emphasize that the matrices $K_i$ are not necessarily positive definite. In fact, our analysis of the examples will demonstrate the effect of the zero eigenvalue direction on the optimization process. 
\begin{assumption}[Target model]\label{ass:target}
We consider two models for the target function: 
\begin{enumerate}[(A)]
    \item\label{ass:Hard_label} (\text{{Hard} label})  $y\mid i = i$ with $i\in [\ls]$  
    \item \label{ass:Soft_label} (\text{{Soft} label}) There exists a fixed matrix $X^\star \in \R^{d\times \ls}$ such that $\|X^\star\|< \sqrt{\ls}C$ for some $C>0$ independent of $d$, and a link function $g_i: \R^\ls \to \R^m$ for all $i\in [\ls]$. The target function is then $y |i  = g_i (\ip{X^\star, a_i}_{\R^d}; \e)$ with $\e\in \R^\ls$ some independent sub-Gaussian noise with  zero mean independent of the class and the input feature $a$. For example, the softmax function, $g_i(r^\star_{i})={e^{r^\star_{i,i}}}/{\sum_{j=1}^{\ell^\star}e^{r^\star_{i,j}}}$,  with $r^\star_{i} = \ip{X^\star, a_i}_{\R^d} \in \R^{\ls}$, and $r^\star_{i,j}$ is it $j$th coordinate. 
    
\end{enumerate}
\end{assumption}
We use the terms ``hard label'' and ``soft label'' to distinguish the two cases in our setup: case~(A) where the target is the class label, and case~(B) where the target is a function of projections of the data to a ground truth $X^\star$. We note that these terms have other meanings in the machine learning literature (e.g., one-hot encodings vs. probability distributions); our usage here refers specifically to the rigidity vs. flexibility of the target specification in our framework.

Throughout the paper, we will denote the matrix concatenating all means by $\mu \defas \mu_1\oplus \mu_2\cdots \oplus \mu_\ls.$ In order to keep a unified notation for hard and soft labels, we will sometimes use the notation
\begin{equation}
    \hat{X}\defas\begin{cases}
        X \oplus X^\star \in \R^{\lpp }\text{ with }\lpp = \ell+\ls & \text{(soft label setting),}\\
        X\in\R^{\lpp }\text{ with }\lpp = \ell &\text{(hard label setting).}
    \end{cases}
\end{equation}

Following the setup above, the population loss function can be written as a function of the parameters $\hat{X}$ with an expectation over the input feature $a$ the class label $i$, and the noise  $\e$: 
\begin{align}
\cL(X) = \E_{(a,i,\e)}[f_i(\hat{X}^\top a;\e)].
\end{align}
where $\{f_i:\R^{\ell+\ls}\to\R\}_{i=1}^{\ls}$ are functions satisfying growth and smoothness as described in assumptions \ref{ass:pseudo_lipschitz} and \ref{ass:risk_fisher} below. This notation, which exploits the problem's symmetry, will be used throughout the paper. This structure applies to both soft-label and hard-label models. In the case of soft label we take $f_i=f$ since the true label is not given and the influence of the class enters through the structure of the label. 

\begin{assumption}[Pseudo-Lipschitz (PL) $f_i$] 
\label{ass:pseudo_lipschitz}
The function $f_i :\R^\lpp \to \R$ is $\alpha$-PL for all $i\in [\ls]$ with respect to $\hat{X}^\top a$ and $\e$.  If $\ls$ growing with $d$ the functions $\{f_i\}$ are Lipschitz (i.e. $\alpha =0$). 
\end{assumption}
 
\begin{assumption}[Risk and its derivatives]\label{ass:risk_fisher} 

Let $B_i\in \R^{\lpp \times \lpp}$ be a positive-semidefinite matrix and let $m_i\in \R^{\lpp }$, and $z\sim \mathcal{N}(0, I_{\lpp}),$ for $i\in [\ls].$ 
The functions   
$\E_z[\nabla f_i(\sqrt{B_i} z + m_i; \e)], $ $\E_z[\nabla f_i(\sqrt{B_i} z + m_i, \e)^{\otimes 2}],$ and $\E_z[\nabla^2 f_i(\sqrt{B_i} z + m_i;\e)], $ are Lipschitz with respect to $B_i, m_i$ for all $i\in[ \ls],$  
with Lipschitz constants, $L_1, L_2, L_{22}$ respectively.  
\end{assumption}

   Assumption \ref{ass:risk_fisher} can be replaced by the weaker but rather more technical Assumption \ref{ass:risk_fisher_U}.

\subsection{Algorithmic set-up}
We consider an online/streaming stochastic gradient descent algorithm with iterates $\{X_k\in \R^{d\times \ell}\}$ given by the recurrence
\begin{align}\label{eq:SGD_update_rule}
X_{k+1}=X_k-\frac{\gamma_k}{d}a_{k+1,I_{k+1}}\otimes \nabla_x f_{I_{k+1}}(r_{k, I_{k+1}};\e_{k+1}),
\end{align}
where the $k$th data point is expressed as
$a_{k,I_k}=\sqrt{K_{I_k}}v_k+\mu_{I_k}$
and $I_k\in\{1,...,\ell^\star\}$ is a randomly chosen class index,  $v_k\sim\cN(0,I_d)$, and $r_{k,I_{k+1}} = \ip{\hat{X}_k, a_{k+1, I_{k+1}}}_{\R^d}.$ 
The notation $\nabla_x$ means that we are taking the gradient with respect to the variables associated with $\ip{X,a}_{\R^d}$ and not those coming from $\ip{X^\star,a}_{\R^d}$ or $\e$. 
The learning rate $\gamma_k$ is uniformly bounded. 

\section{Main Result: Deterministic dynamics for SGD \label{sec:main_result_rho_ODEs_thm}}
\paragraph{Time parametrization} To derive deterministic dynamics, we make the following change to continuous time by setting
\[
k \text{ iterations of SGD} = \lfloor td \rfloor, \quad \text{where $t \in \mathbb{R}$ is the continuous time parameter}.
\]
where $\lfloor \cdot \rfloor$ is the floor function. In addition, we define $X_{td} \defas X_{\lfloor td \rfloor}$ and $ \gamma(t) \defas \gamma_{\lfloor td \rfloor}. $  

\paragraph{Eigen-decomposition} Let $\{(\lambda_\rho^{(i)},u_\rho)\}_{\rho\leq d}$ denote the eigenvalue-eigenvector pairs for the covariance matrix $K_i$, noting that the eigenvectors will be the same for all covariance matrices due to the commuting assumption.  Then, for the index $\rho$ eigenspace, define the projected norm and mean overlap
\begin{equation}
V_\rho(X_{td})\defas dX_{td}^\top u_\rho u_\rho^\top X_{td},\qquad m_{\rho,j}(X_{td})\defas dX_{td}^\top u_\rho u_\rho^\top \mu_j\quad \text{for all }j\in[\ls].
\end{equation}
We also define
\begin{equation}
    B_i(X_{td})\defas \frac1d X_{td}^\top K_iX_{td}=\frac1d\sum_{\rho=1}^d\lambda_{\rho}^{(i)} V_{\rho}(X_{td}).
\end{equation}
\paragraph{Differential equations} We present the deterministic quantities $\csV_{\rho}(t),\csm_{\rho,j}(t),\csB_i(t)$ that approximate the SGD quantities $V_\rho(X_{td}),m_{\rho,j}(X_{td}),B_i(X_{td})$ respectively, where these deterministic equivalents solve the following system of differential equations (for hard labels):
\begin{equation}\label{eq:V_m_rho_2}\begin{split}
    \frac{\dif \csV_\rho}{\dif t}
    &=\sum_{i=1}^\ls p_i\Big(-2\gamma(t)\left( \csV_\rho\lambda_\rho^{(i)}\EE[\nabla^2f_i(\r_{t, i},\e)]+\csm_{\rho,i}\otimes \EE[\nabla f_i(\r_{t,i},\e)]\right)\\
    &\qquad\qquad+{\gamma(t)^2}(\lambda_\rho^{(i)}+\ip{\mu_i,u_\rho}^{2})\EE[\nabla f_i(\r_{t,i},\e)^{\otimes2}]\Big)\\ 
    \frac{\dif \csm_{\rho,j}}{\dif t}&=-\gamma(t)\sum_{i=1}^\ls p_i\left(\lambda_\rho^{(i)}\EE[\nabla^2f_i(\r_{t,i},\e)]\csm_{\rho,j}+d\ip{\mu_i,u_\rho}\ip{u_\rho,\mu_j}\EE[\nabla f_i(\r_{t,i},\e)]\right),
\end{split}\end{equation}
where $(\r_{t,i},\e)$ is a continuous model for $(r_{td,i},\e_{\lfloor td\rfloor+1})$ and has distribution
\begin{equation}
    \r_{t,i}= 
\sqrt{\mathscr{B}_i(t)}v +\frac1d\sum_{\rho=1}^d\csm_{\rho,i}(t) \text{ with } v \sim \mathcal{N}(0, I_\lpp)\text{ and }\e\sim\mathcal{N}(0,\sigma^2I_{\ls})\text{ independent}. 
\end{equation}
The limiting loss is then given by $\csL(t) = \sum_{i=1}^\ls p_i\EE[f_i(\r_{t,i},\e)]. $ Importantly, we note that, using the definition of $\r_{t,i}$ above, \eqref{eq:V_m_rho_2} is an autonomous system for $\{\csV_{\rho},\csm_{\rho,j}\}_{\rho\in[d],j\in[\ls]}$. We derive similar equations in the soft labels setting, see \eqref{eq:V_m_rho_2_soft} in the Appendix \ref{sec:integrodiff}. Furthermore, it follows from the Lipchitz condition in Assumption \ref{ass:risk_fisher} that the system above has a unique solution.  Under a weaker assumption (see Assumption \ref{ass:risk_fisher_U}) the system has a \textit{locally} unique solutions on some set $\csU$.

Finally, taking $\CMscr{M}\defas\bigoplus_{i=1}^{\lpp}\csm_{\rho,i}$, we encode all of these quantities in the compact notation
\[
\CMscr{Z}_\rho(t) \stackrel{\text { def }}{=}
\begin{bmatrix}
\CMscr{V}_{\rho}(t) & \CMscr{M}_{\rho}(t) \\
\CMscr{M}_{\rho}^\top(t) & \tilde{\mu}_{\rho}^{\otimes 2}
\end{bmatrix} 
\]
where 
 {$\tilde{\mu}_\rho = \ip{\mu, u_\rho}$ 
 and 
$\CMscr{Z}_{\rho}(t)$ is the deterministic continuous analogue of $ dW_{td}^\top u_\rho u_\rho^\top W_{td}$ where $W=\hat{X}\oplus\mu$.

Our main result shows that the above deterministic quantities are close to the stochastic ones defined by the iterates of SGD. In fact, we can show that the deterministic dynamics is close in the sense of Theorem \ref{thm:main_risk_m_v} to the stochastic one on a more general infinite class of functions defined by Definition \ref{ass:statistic}. This class of functions also includes the loss function, $\mathcal{L}(X)$, and its deterministic equivalent $\csL(t)$.

\begin{definition}
[Structure
of the statistics, $\varphi$] \label{ass:statistic} Denote by $W \defas\hat{X}\oplus \mu
\in \R^{d\times (\bar{\ell}+\ls)}$, and $q(K_1\dots,K_{\ls} )
$. The statistic satisfies a composite structure, 
\begin{equation*}
\begin{gathered}
\varphi(X) = 
g( W^\top q(K) W ) 
\end{gathered}
\end{equation*}
where ${g} \, : \, \R^{(\bar{\ell}+\ls)\times (\bar{\ell}+\ls)} \to \mathbb{R}$ is $\alpha$-pseudo-Lipschitz function
and $q$  
a polynomial (In the case of $\ls$ growing with $d$, we require the sum of the absolute value of the coefficients of $q$ to be bounded).
\end{definition}

\begin{theorem}[Learning curves]\label{thm:main_risk_m_v} Suppose 
that the assumptions above hold with $\alpha \le 1$. For any function satisfying Definition \ref{ass:statistic}, any $\e \in (0, \frac{1}{2})$ and $T>0$, with overwhelming probability, 
\begin{equation*}
\sup_{0\leq t\leq T} 
| \varphi(X_{\lfloor td \rfloor})- \phi(t)|
\leq Cd^{-\varepsilon}.
\end{equation*}
with $\phi(t) \defas g \bigg (\frac{1}{d} \sum_{\rho=1}^d \CMscr{Z}_\rho(t) q(\lambda_\rho^{(1)},\dots, \lambda_\rho^{(\ls)}) 
 \bigg ).$    
\end{theorem}
\begin{remark}
  The statistics $\varphi(X_{\lfloor td\rfloor})$ can also be compared to $\varphi(\csX_t)$ where $\{\csX_t\}_{t\geq0}$ is another stochastic process, called homogenized SGD (HSGD), with randomness driven by a $d$-dimensional Brownian motion.  See Section \ref{sec:GMM_homogenized} for precise formulas. The statistic $\varphi(\csX_t)$ will also converge to $\phi(t)$, which can be shown in a similar manner to the approach used in \cite{collinswoodfin2023hitting}.  We do not focus on the HSGD here, since it converges to the same deterministic equivalent as SGD.  However, one could consider other learning rate scalings in which statistics of SGD do not concentrate, but can be modeled in distribution by HSGD.  We do not pursue this type of analysis but refer the reader to 
  \cite{arous2022high} and the follow-up works.  
\end{remark}
   
The proof of Theorem \ref{thm:main_risk_m_v} is provided in Section \ref{sec:proof_main_thm}.  The proof involves the following key steps:
\begin{enumerate}
\item The first step of the proof is a Doob decomposition of the process with respect to quadratic functions. This is done in Section \ref{sec:doob_decomp}. This decomposition is a key step that motivates the differential equation above, and the one presented in Section \ref{sec:integrodiff}.
\item In Section \ref{sec:integrodiff}, we present an alternate formulation of the differential equations above in terms of resolvents and show that the two representations are equivalent. The resolvent formulation is a powerful tool that is useful in the discretization of the function space and for providing an autonomous dynamics. It is similar to the one developed in \cite{collinswoodfin2023hitting} for the single-class case, except that our case requires a product of the resolvents of each $K_i$, and we also need to incorporate the means $\mu_i$, which were not present in \cite{collinswoodfin2023hitting}.
\item The heart of the proof comes in Section \ref{sec:approx_sol}. Using the resolvent formulation and integro-differential equations introduced in Section \ref{sec:integrodiff}, the proof of Theorem \ref{thm:main_risk_m_v} boils down to proving Proposition \ref{prop:approx_sol}, which says that our resolvent statistic, evaluated on SGD, is an ``approximate solution'' to the integrodifferential equation in some suitable sense. This argument again follows the approach of \cite{collinswoodfin2023hitting}, but becomes delicate here, particularly in the case where the number of classes is growing with $d.$ 
\item Finally, getting from Proposition \ref{prop:approx_sol} to Theorem \ref{thm:main_risk_m_v} requires removing a stopping time, which is accomplished in Section \ref{sec:supporting_lemmas} (in this section we also verify the stability of solutions). In Section \ref{sec:error_bds} we verify bounds for martingale and error terms that appear in the proof of Proposition \ref{prop:approx_sol}.
\end{enumerate}

\section{Illustrative example: Binary logistic regression}\label{sec:example_binarylogistic} 

Many examples fall within our framework.  We mention a few briefly and provide a deeper analysis for the binary logistic regression.  While this example is simpler than many others, it yields differential equations that are tractable to study analytically, providing an illustration of how our methods can be used to study the effects of anisotropy on SGD dynamics for Gaussian mixture data.  For more complicated examples, we see similar effects numerically in simulations of the equations given in our theorem.

In the set-up for binary logistic regression, we consider just two classes $i\in\{1,2\}$ and we denote by $\{(\lambda_\rho^{(i)},u_\rho)\}_{\rho\leq d}$  the eigenvalue-eigenvector pairs of $K_i$ (note that the eigenvectors are the same for all classes since we assume that the matrices commute).  For ease of computation, we also take the means for classes 1 and 2 to be $+\mu$ and $-\mu$ respectively with $\|\mu\| = O(1)$, noting that this set-up is equivalent to any other choice of means in the binary case by a coordinate shift. 

To study the effect of anisotropy on the SGD dynamics in this setting, we compare three models for the data distribution: 
\begin{enumerate}
    \item \textit{Identity model} -- Covariance $K_1=K_2=I_d$ 
    ;
    \item \textit{Zero-one model} -- All eigenvalues of $K_1$ and $K_2$ are either 0 or 1 
    ;
    \item \textit{Power-law model} -- Eigenvalues of $K_1$ and $K_2$ follow a power-law and entries of mean $\mu$ also follow a power-law (with different exponent).
\end{enumerate}
The set-ups for the zero-one model and power-law model, as well as our results for each, are described in further detail below.  We note that the identity model has been well-studied via other methods (e.g., see \cite{arous2022high}), but we briefly describe it here for the purpose of comparison with the anisotropic models.

Figure \ref{fig:logistic_concentration_GMM} presents the learning curves (loss vs. SGD iteration/$d$) for all models across different learning rates, showing both the SGD output and our matching theoretical predictions via the ODEs. The most important observation, which we will develop further when analyzing each model analytically, is that for the identity model and the mild power-law, the loss saturates relatively quickly, depending on the learning rate. In contrast, for the zero-one model and the extreme power-law, the loss continues to decrease toward zero at a relatively slow rate.
\begin{table}
 \centering
\noindent\renewcommand{\arraystretch}{1.4}
\begin{tabular}{|p{3in}|p{2.4in}|}
\hline
\textbf{Zero-one model \& 
\newline Extreme power-law} $\a\in[\b+1, 2\b)$  
\begin{itemize}
\item Risk decays to zero at a polynomial rate 
\[t^{-c_1}<\csL(t)< t^{-c_2}\]
In the zero-one model, $\csL(t)\asymp t^{-1}$.
\item Norm of SGD iterates grows at log rate
\end{itemize}
& 
 \textbf{Identity model \& 
 \newline Mild power-law} $\a<\b+1$
\begin{itemize}
\item Risk bounded away from zero
\item Norm of iterates remains bounded
\end{itemize}\\
\hline
\end{tabular}
\caption{Summary of our results for the zero-one, identity, and power-law models under binary logistic regression (see Propositions \ref{prop:zero-one},\ref{prop:power_law_good_regime},\ref{prop:power_law_bad_regime} for more detail). This is also supported by numerical simulation. See Figure \ref{fig:power_law_los} and Figure \ref{fig:power_law_m} for power-law models, and Figure \ref{fig:zero_one_identity_los_m} for identity and zero-one models. }
\end{table}

\begin{figure}[t]
\centering

\begin{minipage}[t]{0.4\textwidth}
    \centering
    \includegraphics[width=\linewidth]{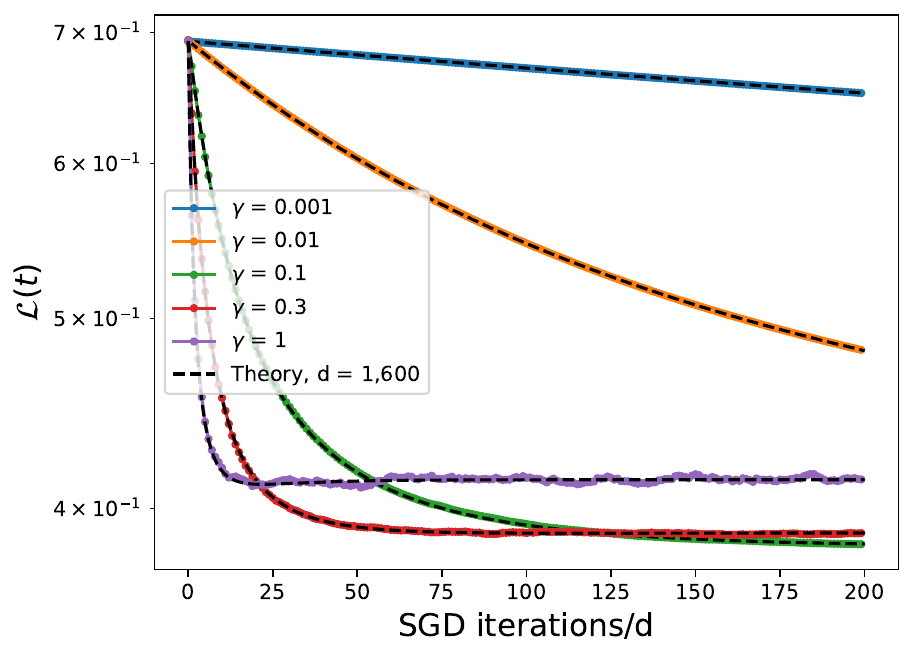}
    \textit{(a) Identity covariance}
\end{minipage}\hspace{0.1cm}
\begin{minipage}[t]{0.4\textwidth}
    \centering
    \includegraphics[width=\linewidth]{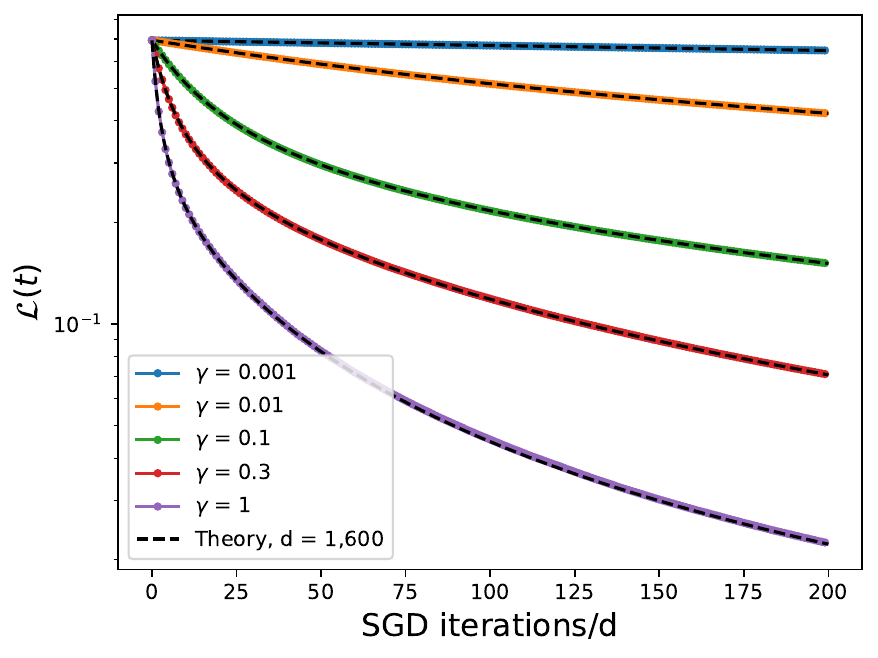}
    \textit{(b) Zero-one model}
\end{minipage}

\vspace{0.35cm}

\begin{minipage}[t]{0.4\textwidth}
    \centering
    \includegraphics[width=\linewidth]{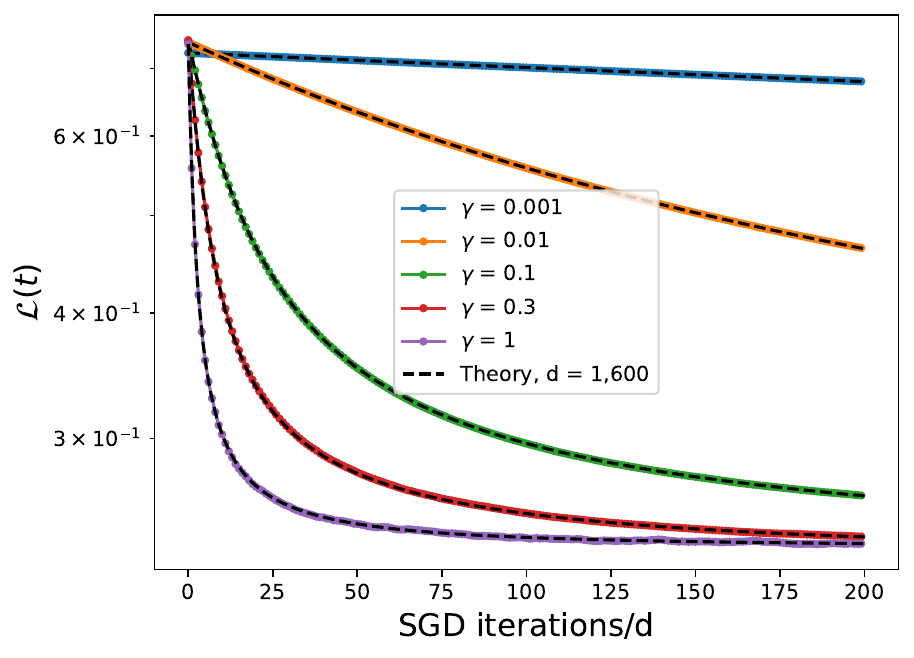}
    \textit{(c) Mild power-law $\beta=1$}
\end{minipage}\hspace{0.1cm}
\begin{minipage}[t]{0.4\textwidth}
    \centering
    \includegraphics[width=\linewidth]{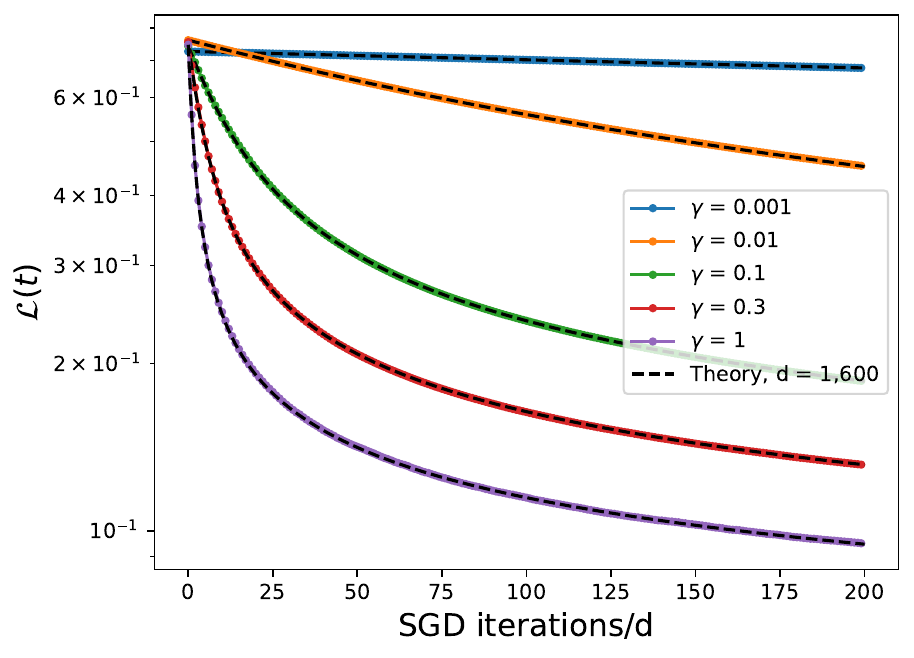}
    \textit{(d) Extreme power-law $\beta=0.2$}
    \vspace{0.35cm}
\end{minipage}

\caption{\textbf{Population risk concentration} in binary logistic regression across learning rates $\gamma$.
(a) Identity covariance.  
(b) Zero-one model where the mean lies in the null space of the covariance.  
(c,d) Power-law diagonal covariances $\lambda_\rho^{(1)}=(\rho/d)^{1.1}$ and $\lambda_\rho^{(2)}=(\rho/d)^{1.5}$, with mean $\mu_\rho = d^{-1/2}(\rho/d)^{\beta}$ for $\beta = 1$ and $0.2$.}
\label{fig:logistic_concentration_GMM}
\end{figure}

To gain a more precise understanding of each of these models, we study the dynamics of $\csm(t)$, $\csV(t)$, and $\csB_i(t)$, which are the deterministic equivalents (by our theorem) of $\mu^\top X_{\lfloor td\rfloor}$,  $\|X_{\lfloor td\rfloor}\|^2$, and $X_{\lfloor td\rfloor}^TK_iX_{\lfloor td\rfloor}$ respectively.  Intuitively, $\csm(t)$ and $\csV(t)$ tell us about the overlap of the iterate vector with the mean direction and the norm of the iterate vector, and how they evolve over time.  Computationally, these quantities are defined as
\[
\csm(t)=\frac1d\sum_{\rho=1}^d\csm_{\rho,1}(t),\qquad \csV(t)=\frac1d\sum_{\rho=1}^d\csV_\rho(t),\qquad \csB_i(t)=\frac1d\sum_{\rho=1}^d\lambda_\rho^{(i)}\csV_\rho(t),
\]
where $\{\csm_{\rho,j}(t),\csV_\rho(t)\}_{\rho=1}^d$ solve the system of ODEs given in \eqref{eq:V_m_rho_2} and, for the binary case, we only need to consider the class index $j=1$ for $\csm$, since our set-up has the symmetry $\csm_{\rho,2}=-\csm_{\rho,1}$. We denote the normalized overlap at time $t$ by $$\text{Alignment}:= \frac{\csm(t)}{\sqrt{\csV(t)}}$$

These quantities enable us to predict the exact dynamics of the loss $\mathcal{L}(X_{\lfloor td\rfloor})$ through its continuum equivalent $\csL(t)$. Using this framework, we provide an asymptotic analysis that determines whether the loss converges to zero or to a positive constant and characterizes the scaling behavior of the loss and these quantities as a function of the number of iterations (which corresponds to the number of samples). 
This scaling behavior is of its own interest as it provide a scaling law for nonlinear models and the scaling depends on the structure of the data i.e. covariance matrices and means. 

\subsection{Zero-one model (subspace of perfect classification)\label{subsec:zero_one_result}}

\begin{figure}[htbp]
\centering
\vspace{-0.2cm} 
\begin{minipage}[t]{0.23\textwidth}
    \centering
    \includegraphics[height=3.5cm]{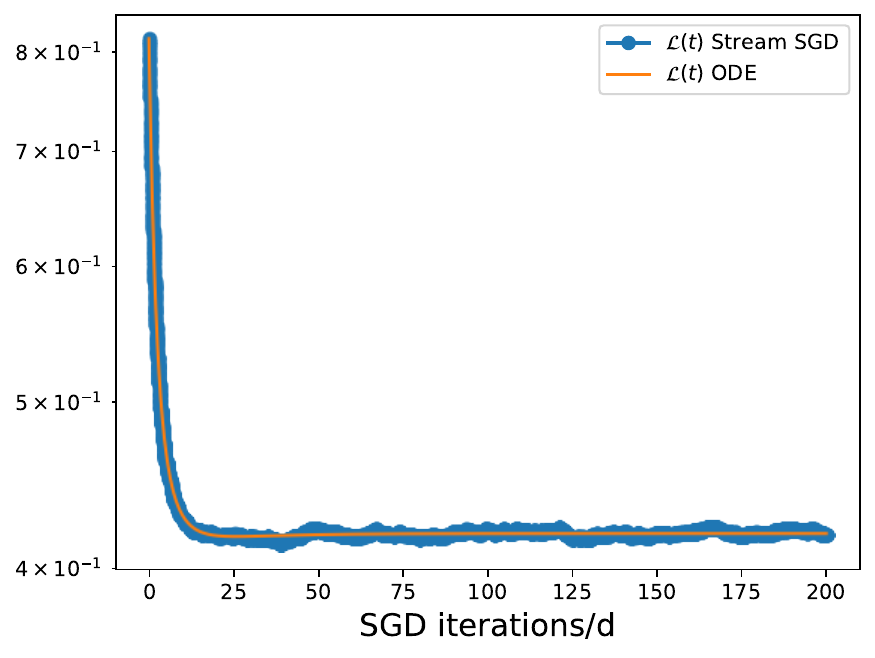}
    \medskip
    \textit{(a) Identity - $\csL(t)$ }
\end{minipage}\hspace{0.1\textwidth}%
\begin{minipage}[t]{0.23\textwidth}
    \centering
    \includegraphics[height=3.5cm]{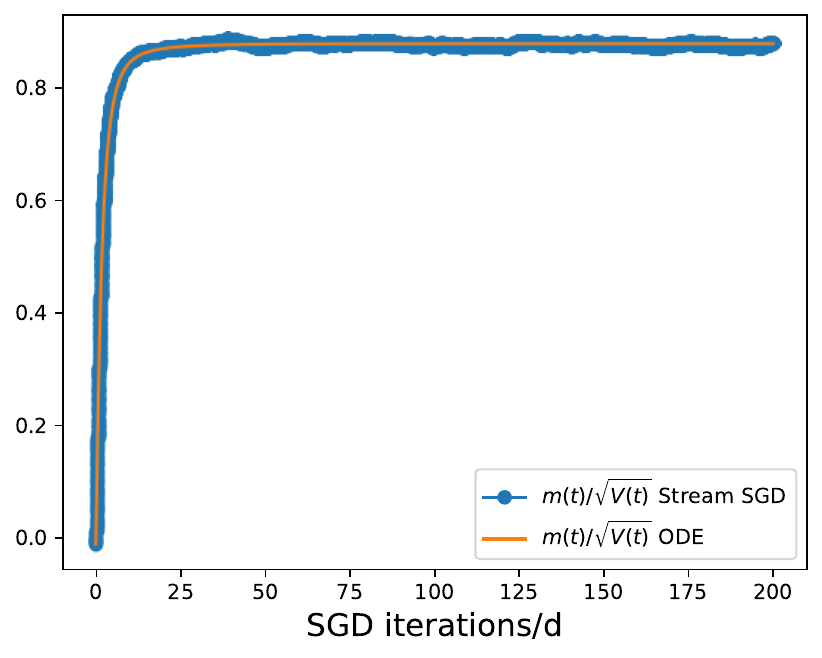}
    \medskip
   \textit{(b) Identity - 
   Alignment
   }
\end{minipage}\hspace{0.1\textwidth}%
\begin{minipage}[t]{0.23\textwidth}
    \centering
    \includegraphics[height=3.5cm]{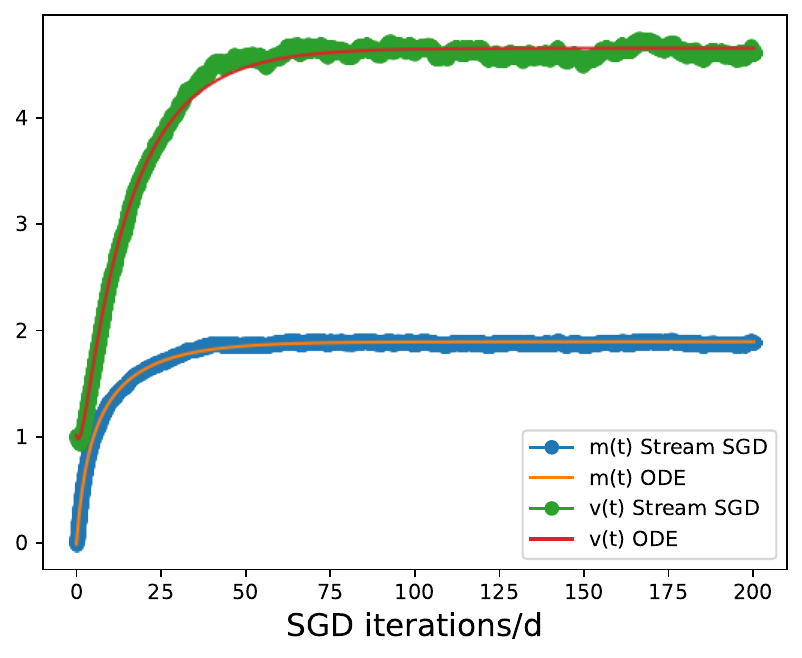}
    \medskip
     \textit{(c) Identity - $\mathfrak{m}(t)$ and $\sqrt{\mathfrak{V}(t)}$}
    
\end{minipage}

\vspace{0.4cm} 

\begin{minipage}[t]{0.23\textwidth}
    \centering
    \includegraphics[height=3.5cm]{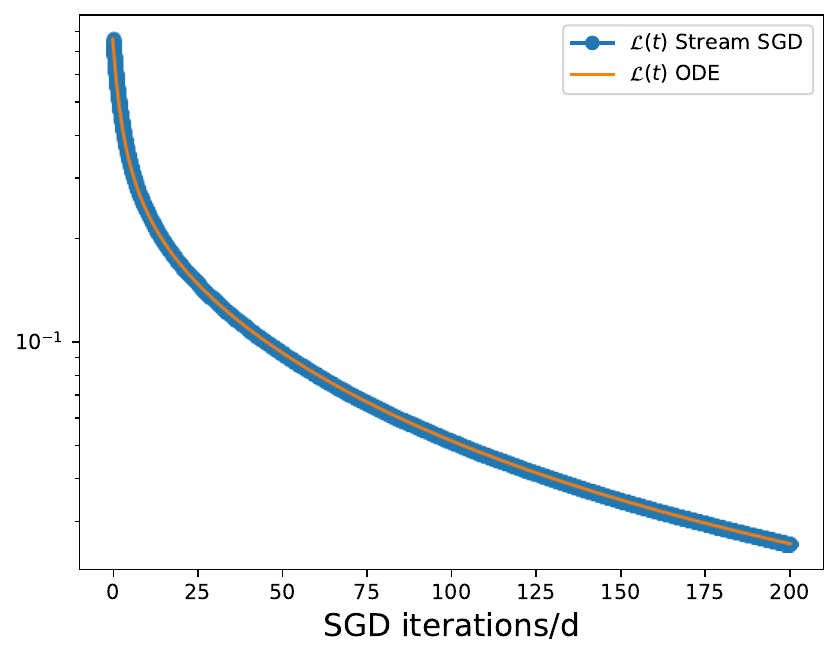}
    \medskip
    \textit{(d) Zero-one model- $\csL(t)$}
\end{minipage}\hspace{0.1\textwidth}%
\begin{minipage}[t]{0.23\textwidth}
    \centering
    \includegraphics[height=3.5cm]{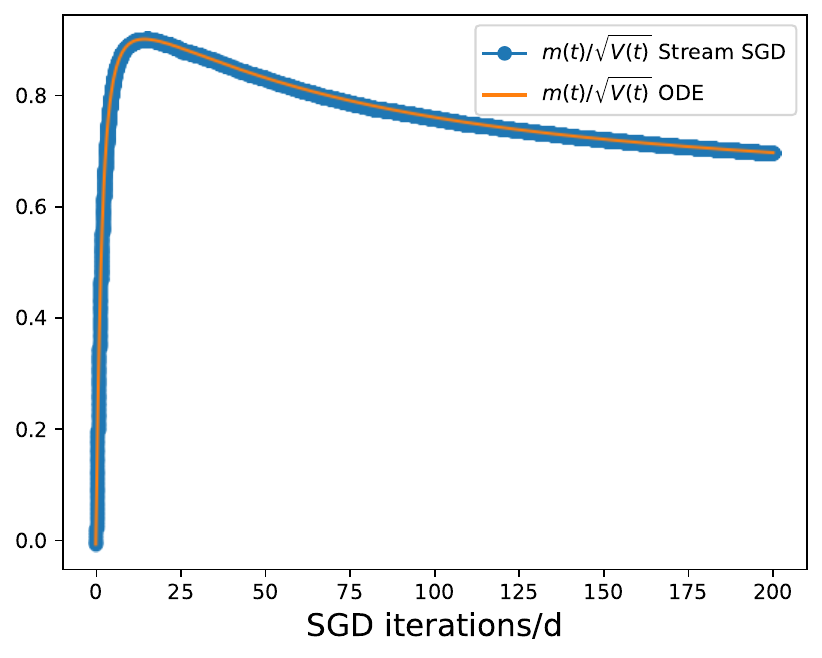}
    \medskip
   \textit{(e) Zero-one model- 
   Alignment
   }
\end{minipage}\hspace{0.1\textwidth}%
\begin{minipage}[t]{0.23\textwidth}
    \centering
    \includegraphics[height=3.5cm]{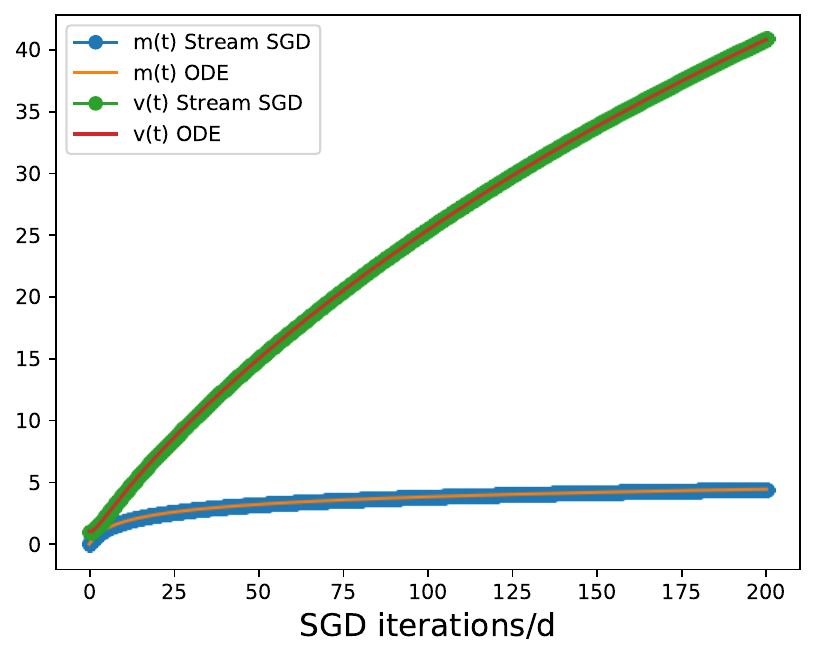}
    \medskip
      \textit{(f) Zero-one model- $\mathfrak{m}(t)$ and $\sqrt{\mathfrak{V}(t)}$}
   
\end{minipage}
\vspace{0.1cm} 
\caption{\textbf{Identity vs. Zero-one model} with $d=1000$ and $\gamma=0.9$.  
Panels show $\mathcal{L}(t)$, $\mathfrak{m}(t)$, $\mathfrak{V}(t)$, and the alignment, defined as the normalized ratio $\mathfrak{m}(t)/\sqrt{\mathfrak{V}(t)}$.  
Each curve compares SGD simulations with the corresponding ODE predictions.}
\label{fig:zero_one_identity_los_m}
\vspace{-0.2cm} 
\end{figure}

\begin{definition}[Zero-one model]

The covariance matrices $K_1,K_2$ have all eigenvalues equal to either $0$ or $1$.  Then the index set $[d]$ can be partitioned as
\begin{equation}
    [d]=I_{00}\cup I_{01}\cup I_{10}\cup I_{11},\quad\text{where }I_{jk}=\{i\leq d\;|\;\lambda^{(1)}_i=j,\;\lambda^{(2)}_i=k\}.
\end{equation}
We further assume that $I_{00}\neq\emptyset$ and there exists $i\in I_{00}$ such that $\mu^\top u_i\neq0$, meaning that there exist eigenvector directions in which both $K_1,K_2$ have eigenvalue 0, and the mean $\mu$ has non-zero overlap with one or more of these directions.
\end{definition}
From an information theoretical perspective, if $X$ aligned with the projection of $\mu$ into the eigenspace associated with $I_{00}$, this would yield perfect classification (because the two classes have zero variance and distinct means in that direction).  This is in sharp contrast with the identity model, in which there is non-trivial variance in every direction and perfect classification is impossible.

While the zero-one set-up is not an accurate model for real-world data, it is an interesting toy example to study the dynamics that occur when variance is significantly larger in some directions than in others.  For isotropic data, one expects the classifier to naturally prefer the direction of the largest difference in the class means.  For anisotropic data, one expects that both the direction of greatest difference in means and the directions of smallest variance will play a role.  To capture this idea of ``directions of small variance,'' we denote the projection of $\csm(t)$ into the eigenspace associated with $I_{ij}$ by
\begin{equation}
    m_{(ij)}(t)=\sum_{\rho\in I_{ij}}\csm_\rho(t)
\end{equation}
so we have $m_{(00)}(t)+m_{(01)}(t)+m_{(10)}(t)+m_{(11)}(t)=\csm(t)$. Likewise, we define $\mu_{(ij)}$ to be the projection of $\mu$ into the span of eigenvectors $\{u_\rho:\lambda^{(1)}_\rho=i,\lambda^{(2)}_\rho=j\}.$ Intuitively, $m_{(00)}$ tells us about the alignment of $X$ with the projection of $\mu$ into the zero-variance subspace.
Our results for the zero-one model are summarized in the following proposition, whose proof appears in Section \ref{sec:proof_zero-one}. 

\begin{proposition}\label{prop:zero-one}
    Consider the zero-one model with $\|\mu_{(ij)}\|^2\to\frac14\|\mu\|^2$ as $d\to \infty$ and the additional symmetry constraints that $p_1=p_2=\frac12$ and $|I_{jk}|=\frac14$ for all $j,k\in\{0,1\}$ and $X_0=0$.  We also impose a technical assumption \ref{ass:W_1W_2ab}, which is explained later.  Then we have the following bounds:
    \begin{enumerate}
     \item\label{prop:zero-one_loss} For the loss,    \[\csL(t)\asymp t^{-1}.\] This asymptotic for the deterministic equivalent of the loss implies an analogous asymptotic for the empirical loss in high dimension as a corollary of Theorem \ref{thm:main_risk_m_v}.
        \item\label{prop:zero-one_m/v2} For the alignment with $\mu$, we have the following asymptotic for large $d$ and $t$:
        \[\frac{\csm(t)}{\sqrt{\csV(t)}}=\frac12\left(1+O((\log t)^{-1})\right).\]
        Recall that $\csm(k/d)$ and $\csV(k/d)$ model $X_k^\top \mu$ and $\|X_k\|^2$ respectively.  Applying Theorem \ref{thm:main_risk_m_v}, this result implies that, when $d$ is large, the projection to the mean, $X_k^\top \mu/\|X_k\|$, will approach roughly $\frac12$ as $k/d$ grows. 
        \item\label{prop:zero-one_mij} For alignment with the zero-variance eigenspace, we have the following large-$t$ asymptotics:
        \[m_{(00)}=\log t+O(1),\qquad m_{(01)},m_{(10)},m_{(11)}\asymp1,\]
        which implies that the distance from $X_k/\|X_k\|$ to the zero-variance eigenspace is $O((\log(k/d))^{-1})$ as $k/d$ becomes large.
    \end{enumerate}
\end{proposition}

\begin{remark}
    The proof of Proposition \ref{prop:zero-one}, which we defer to Section \ref{sec:proof_zero-one}, relies on the technical assumption \ref{ass:W_1W_2ab}. Intuitively, this assumption implies that the logistic weights concentrate along the limiting trajectory. We justify this assumption numerically and heuristically, and we believe that it holds generically in the identity, zero-one, and power-law models, although we do not prove this analytically. In the identity model, we provide a proof that does not rely on the assumption.
\end{remark}

     \begin{figure}[h]
     \centering
         \includegraphics[scale = 0.45]{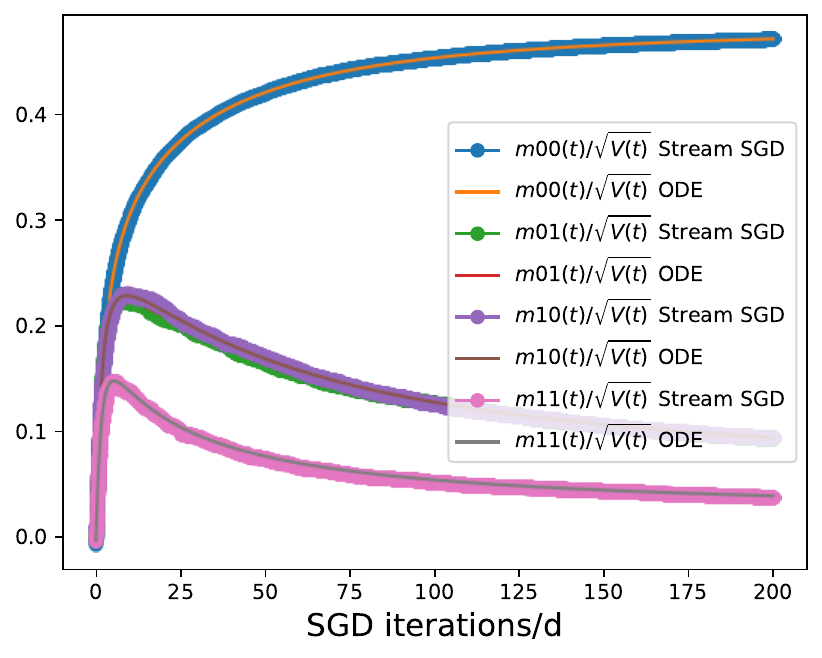}        
         \caption{\textbf{Zero-one model subspace alignment} with $d=1000$ and $\gamma=0.9$}
         \label{fig:alignment_proj_ZO}         
     \end{figure}

 Figure \ref{fig:zero_one_identity_los_m} supports these theoretical findings. It compares the loss curves and alignments of the identity and zero-one models. In particular, display \ref{fig:zero_one_identity_los_m}(d) shows the slow polynomial decay of the loss toward zero. Display \ref{fig:zero_one_identity_los_m})(e) presents the alignment, which converges to a constant, and finally, display \ref{fig:zero_one_identity_los_m}(f) presents the overlap itself, showing logarithmic growth. The behavior of the projected components of the overlap is presented in Figure \ref{fig:alignment_proj_ZO}. The complementary analytical results for the identity model are presented in the next section, Proposition \ref{prop:power_law_good_regime}.

\begin{remark}
    In the zero-one model, we observe two stages of learning. In the first stage, the iterate aligns with the direction of the mean. This happens at approximately the same rate as in the identity model. In the second stage, the iterate tends to align more with the projection of the mean into the subspace of zero variance of the data. This convergence is much slower than the first one. We suspect that this is due to the fact that, in this setting, there exists a separating subspace. We note that several works have studied gradient descent (GD) with logistic loss on separable data. Interestingly, they showed that GD minimizes the training loss and also maximizes the margin \cite{soudry2018implicit} in a rate  $O(1/\log(t))$ with $t$ being the iteration of GD\cite{bartlett2017spectrally}. Several works have improved that rate by properly normalizing the gradient update \cite{ji2021characterizing,nacson2019convergence} and using momentum techniques \cite{ji2021fast}, improving that rate to $O(1/t^2)$. 
\end{remark}

\begin{figure}[htbp]
    \centering
    \includegraphics[width=0.4\linewidth]{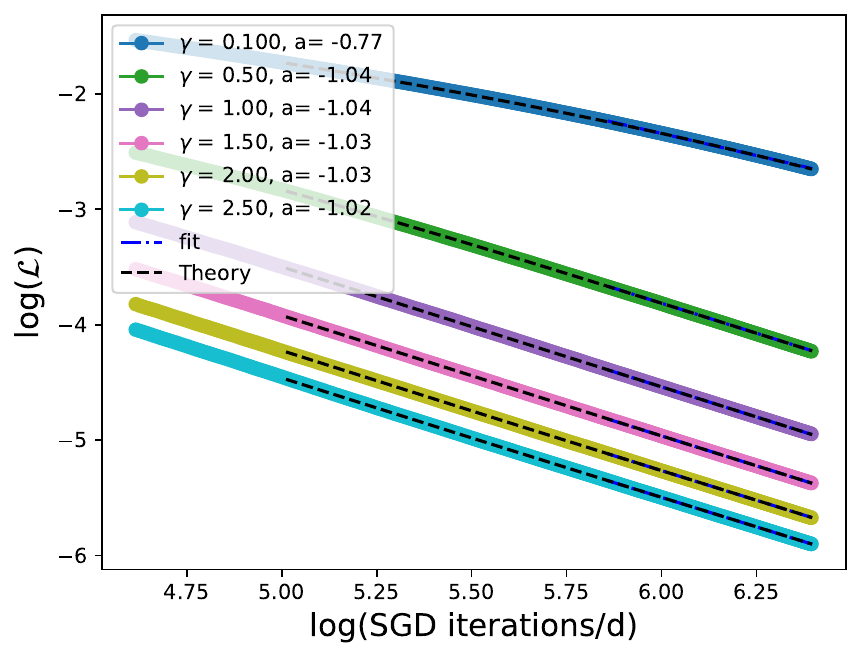}
    \caption{\label{fig:zero_one_rate_vs_gamma}\textbf{Zero-one model - loss scaling law} remain the same for different  learning rates $\gamma$. The value $a$ represents the power obtained i.e. $\csL(t) \asymp t^{a}$. The figure supports the rate found in Proposition \ref{prop:zero-one}, which is independent of the learning rate.} 
    \label{fig:placeholder}
\end{figure}

\begin{remark}
 Figure \ref{fig:zero_one_rate_vs_gamma} plots the decay rate of the loss in the zero-one model as a function of different learning rates, suggesting that after sufficiently many SGD iterates (large $t$) the rate of decay of the loss  (samples/parameter's dimension) is proportional to $t^{-1}$, regardless of the learning rate, supporting Proposition \ref{prop:zero-one}(\ref{prop:zero-one_loss}).
\end{remark}
\subsection{Power-law and identity models asymptotic analysis}
We now move beyond the zero-one model to a more realistic data distribution, i.e., one where the magnitudes of the eigenvalues are varied and are non-zero, but can be arbitrarily close to zero as the dimension grows.
\begin{definition}[Power-law model]\label{ass:power_law_mu_lam_con}
The eigenvalues of the matrix $K_i$ and the projections of the mean $\mu$ on the eigenvectors of $K_i$ satisfy, 
\begin{equation}
    \lambda_\rho^{(i)}=\left(\frac{\rho}{d}\right)^{\a_i};\qquad
    \tilde{\mu}_\rho^{(i)}=\frac{1}{d}\left(\frac{\rho}{d}\right)^{\b_i} \quad {\beta_i}, \a_i\ge 0.
\end{equation} 
where $\tilde{\mu}_\rho^{(i)}\defas (\mu_i^\top u_\rho)^2$ for all $i\in [\ls]$.
\end{definition}
As in the zero-one section \ref{subsec:zero_one_result}, we consider binary logistic regression with class mean $-\mu$ and $\mu$ so we have $\beta_1=\beta_2 = \b$. For simplicity, we focus on the case where both classes have the same covariance with exponent $\a$. 

\begin{remark}[The role of $\a$ and $\beta$]

For all $\a>0,$ this set-up has a limiting spectral density for $K_i$ (as $d\to\infty$) supported on the interval $(0,1]$ with the density function $f(\lambda)=\frac1\a\lambda^{\frac1\a-1}$.  This is integrable for all $\a>0$. For $\a>1$, it is unbounded near zero, which is our intended power-law set-up with eigenvalues accumulating near 0.  In contrast, the density for $0<\a\leq1$ is bounded and is an increasing function on the interval $(0,1]$ (except in the case $\a=1$ where it is the uniform density function). Finally, the case of $\a=0$ corresponds to a point mass $1$, i.e., identity covariance. $\b$ characterizes the alignment of the mean with the eigenvalues of the covariance. When the mean aligns with directions in which there is an accumulation of zero eigenvalues, depending on the values of $\a$ and $\b$, perfect classification is possible. We demonstrate that SGD converges to this solution.     
\end{remark}

We distinguish two regimes according to the effective signal-to-noise ratio $\mu^\top K^{-1}\mu.$
This quantity provides an information-theoretic baseline: for the
symmetric two-class Gaussian mixture, by Fisher's classical linear discriminant rule, the Bayes classification error is
$\Phi(-\sqrt{\mu^\top K^{-1}\mu})$ where $\Phi$ is the cumulative distribution function
of the standard normal distribution; see \cite{cai2019high} in the context of high dimension. In the mild power-law regime, $\mu^\top K^{-1}\mu$ remains
bounded as $d\to\infty$, and hence the population classification error cannot
vanish (This corresponds to the regime in which $\a<\b+1$).  This regime is qualitatively similar to the identity-covariance model,
where the noise is spread across all directions. Proposition
\ref{prop:power_law_good_regime} shows that the limiting SGD risk remains
bounded away from zero, consistent with this information-theoretic
obstruction. The corresponding numerical results are shown in Figure
\ref{fig:power_law_los}(b) for the loss and Figure
\ref{fig:power_law_m}(b) for the alignment with the true mean.

In the extreme power-law regime, by contrast,
$\mu^\top K^{-1}\mu\to\infty$ as $d\to\infty$ (and $\a\geq\b+1$), so the Bayes classification
error vanishes. Intuitively, the signal has a sufficiently large projection
onto directions whose noise variance becomes small. Proposition
\ref{prop:power_law_bad_regime} shows that the limiting SGD dynamics exploit
these low-noise directions and drive the risk arbitrarily close to zero as $d$ grows; Numerical evidence for this regime is presented in Figure
\ref{fig:power_law_los}(a) for the risk and Figure
\ref{fig:power_law_m}(a) for the alignment with the true mean. Depending on the values of $\a$ and $\b$, the qualitative behavior of SGD on this model may look more similar to the identity model or more similar to the zero-one model.The power-law model, therefore, provides an interpolation between these
two limiting cases.

For comparison, Figure \ref{fig:zero_one_identity_los_m} presents the risk
(display \ref{fig:zero_one_identity_los_m}(a)), alignment
(display \ref{fig:zero_one_identity_los_m}(b)), and overlap
(display \ref{fig:zero_one_identity_los_m}(c)) for the identity model. These
observations are also consistent with Proposition
\ref{prop:power_law_good_regime}. Finally, Figure
\ref{fig:power_law_align_proj} provides the power-law counterpart of Figure
\ref{fig:alignment_proj_ZO} for the zero-one model. Unlike in the zero-one
case, the alignment in the power-law model does not admit a sharp
decomposition into growing and constant components. We therefore illustrate
its spectral distribution using a cutoff at half of the eigenmodes, while
noting that the resulting aggregated projections can depend substantially on
the choice of cutoff.

\begin{proposition}[Mild power-law regime and identity]\label{prop:power_law_good_regime}
Suppose $p_1=\frac{1}{2}$, $K_1 =K_2 = K,$ $X_0 = 0$ 
and Assumption \ref{ass:W_1W_2ab} holds together with Assumption \ref{ass:power_law_mu_lam_con} for $\beta+1>\a$ or for $K_1=K_2 =I_d$ for any $\mu$ with $\|\mu\| = O(1)$ (identity model).
Then $\csm(t),\csB_1(t)$, and $\csB_2(t)$ are all bounded above while the loss is bounded below by a positive constant.  More precisely, 
\begin{enumerate}
    \item $\csm(t) \to c\mu^\top[K]^{-1}\mu$ for all $t\ge 1$ with $c\in [1, C_w]$ 
    \item $\csB(t) = \csB_1(t) = \csB_2(t)< C$ for some $C=C(\gamma, \|\mu\|, \a,\b)>0.$ 
    \item 
    $\log(1+e^{-C_w\mu^\top[K]^{-1}\mu} 
    )\le \csL(t)\le \log(1+e^{-(1-\frac{C_w^2}{2})\mu^\top[K]^{-1}\mu+\frac{\gamma C_w}{4d}\textup{Tr}(K)}).$ 
\end{enumerate}
\end{proposition}

\begin{figure}[t]
\centering
\vspace{-0.2cm} 

\begin{minipage}[t]{0.4\textwidth}
    \centering
    \includegraphics[height=4cm]{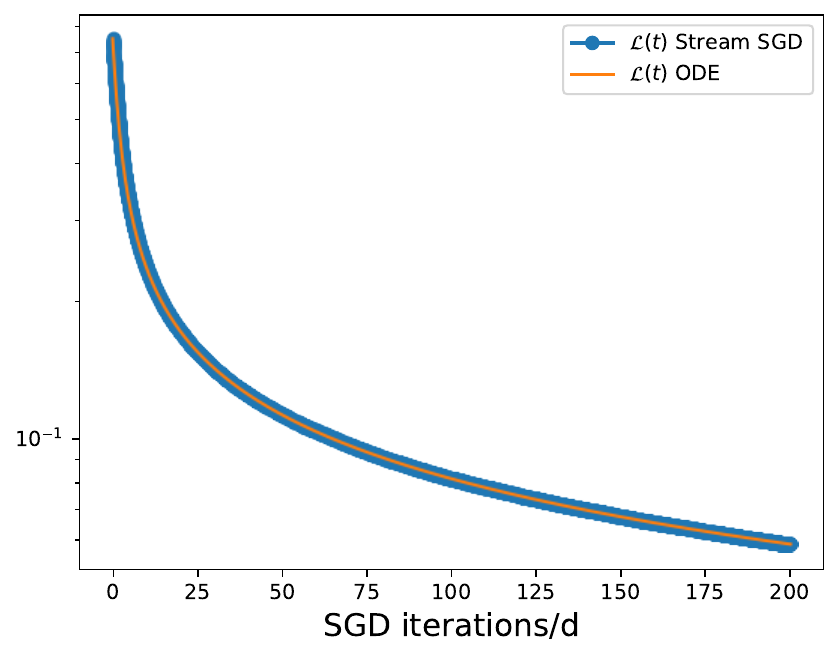}
    \medskip
    
    \textit{(a) Risk curve, $\beta=0$}
\end{minipage}\hspace{0.05\textwidth}%
\begin{minipage}[t]{0.4\textwidth}
    \centering
    \includegraphics[height=4cm]{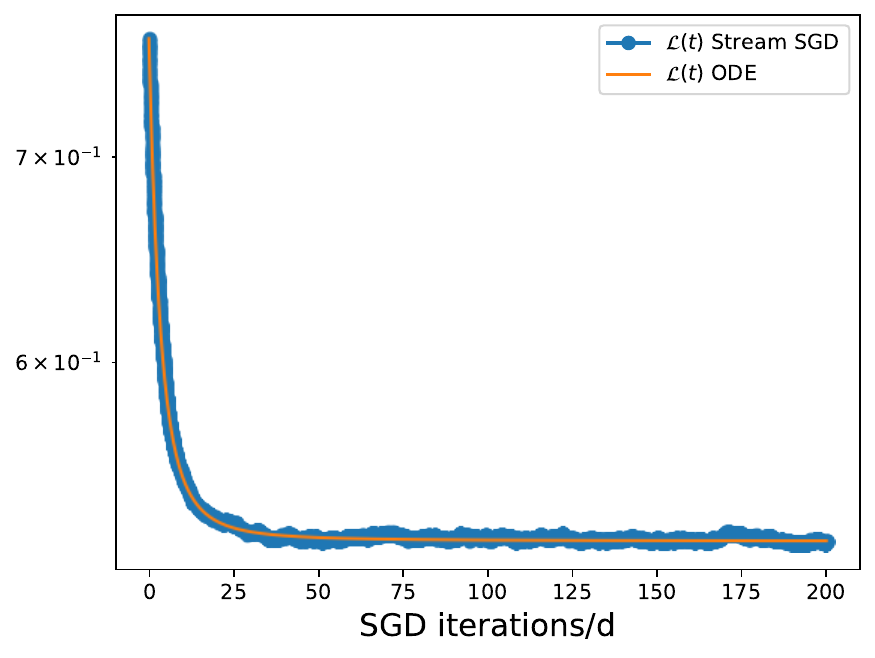}
    \medskip
    
    \textit{(b) Risk curve, $\beta=1.2$}
\end{minipage}
\vspace{0.1cm}
\caption{\textbf{Power-law model - Learning curve} with $d=1000$, $a_1=a_2=1.2$, and different values of $\beta$.
Panel (a) corresponds to the extreme power law regime; the slow decay of the loss illustrates Proposition \ref{prop:power_law_bad_regime}. Panel (b) corresponds to the mild power law regime; the loss decays to a constant, illustrating Proposition \ref{prop:power_law_good_regime}.}
\label{fig:power_law_los}
\vspace{-0.1cm} 
\end{figure}
\begin{proposition}[Extreme power-law]\label{prop:power_law_bad_regime}
Suppose $p_1=\frac{1}{2}$,  $K_1 =K_2 = K,$ $X_0 = 0$, Assumption \ref{ass:power_law_mu_lam_con},  hold and $\b+1\le \a$. Then 
\begin{enumerate}
    \item Suppose $\beta +1< \a $ then $ \csm(t) 
    =\Omega(\log^{1-\frac{1+\b}{\a}}(t))$ and if $\beta +1 =  \a $ then  $\csm(t) 
     =\Omega(\log(\log(t)))$    
    \item Suppose $\b < 2\a$, Assumption \ref{ass:W_1W_2ab} holds and there exist $\gamma,t_0>0$ and $\e=\e(\gamma,\a,\b)\in(0,1)$ such that, for all $t\geq t_0$, we have $C_w\leq2-\e.$  
    Then, 
    \begin{enumerate}
    \item  $ \csB(t)\asymp \csm(t)  
   \asymp \log(t)$
    \item There exist constants $c_1(\gamma,\a,\b), c_2(\gamma,\a,\b)>0 $ such that the risk decreases to zero with 
    $$t^{-c_1}\le \csL(t) \le t^{-c_2}.$$
    \end{enumerate}
\end{enumerate}
\end{proposition}

The proofs of the above propositions under milder assumptions on the mean and covariance matrix (see Assumption \ref{ass:F_K_power_law_assmp}) are deferred to Appendices  \ref{sec:power_law_proofs} and \ref{sec:identity_proof}.

\begin{remark}
    We expect that similar results hold when the two classes have different covariance matrices (i.e. $K_1\neq K_2$). We conjecture that what will determine the transition is the maximal power in $\a$, see Figure \ref{fig:logistic_concentration_GMM}. As the analysis is not trivial, we leave this for future work.
\end{remark}

\begin{figure}[t]
\centering
\vspace{-0.2cm} 
\begin{minipage}[t]{0.4\textwidth}
    \centering
    \includegraphics[height=4cm]{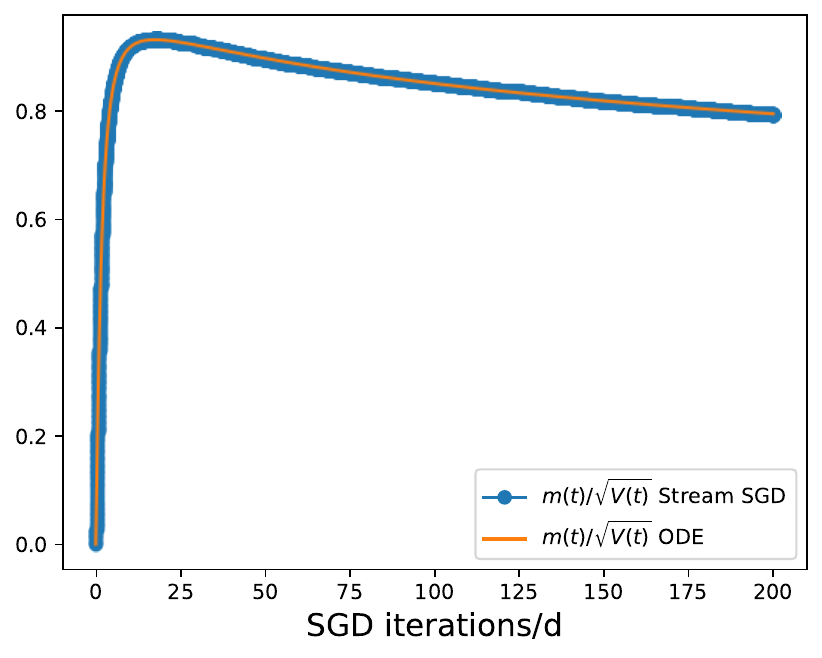}
    \medskip
    \textit{(a) Alignment, $\beta=0$}
\end{minipage}\hspace{0.05\textwidth}%
\begin{minipage}[t]{0.4\textwidth}
    \centering
    \includegraphics[height=4cm]{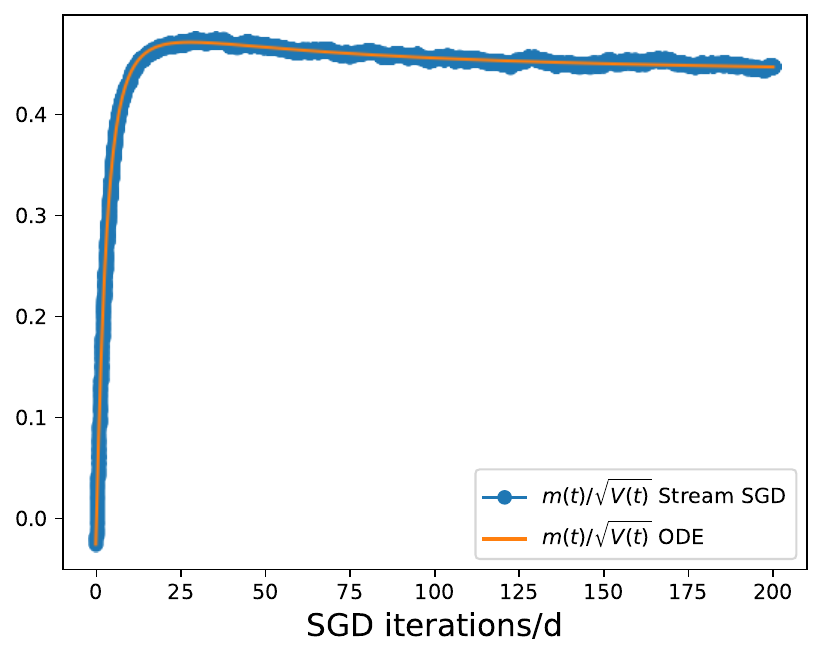}
    \medskip
    \textit{(b) Alignment, $\beta=1.2$}
\end{minipage}
\vspace{0.1cm}
\caption{\textbf{Alignment for the power-law model} with $d=1000$, $a_1=a_2=1.2$, and different values of $\beta$.}
\label{fig:power_law_m}
\vspace{-0.4cm} 
\end{figure}

\begin{figure}[t]
\centering
\begin{minipage}[t]{0.25\textwidth}
    \centering
    \includegraphics[height=4cm,keepaspectratio]{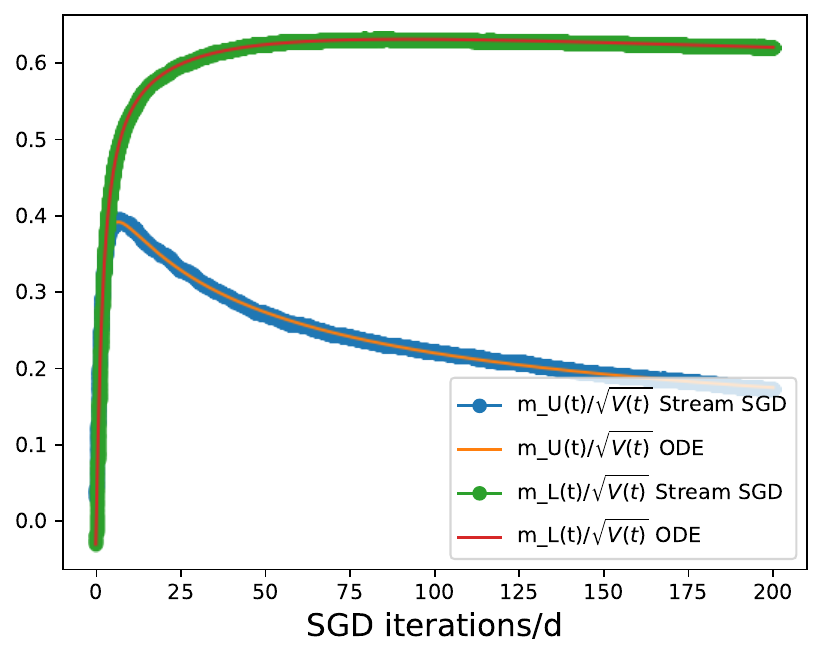}
    \medskip
    \textit{(a) Subspace alignment, $\beta=0$}
\end{minipage}\hspace{0.1\textwidth}%
\begin{minipage}[t]{0.25\textwidth}
    \centering
    \includegraphics[height=4cm,keepaspectratio]{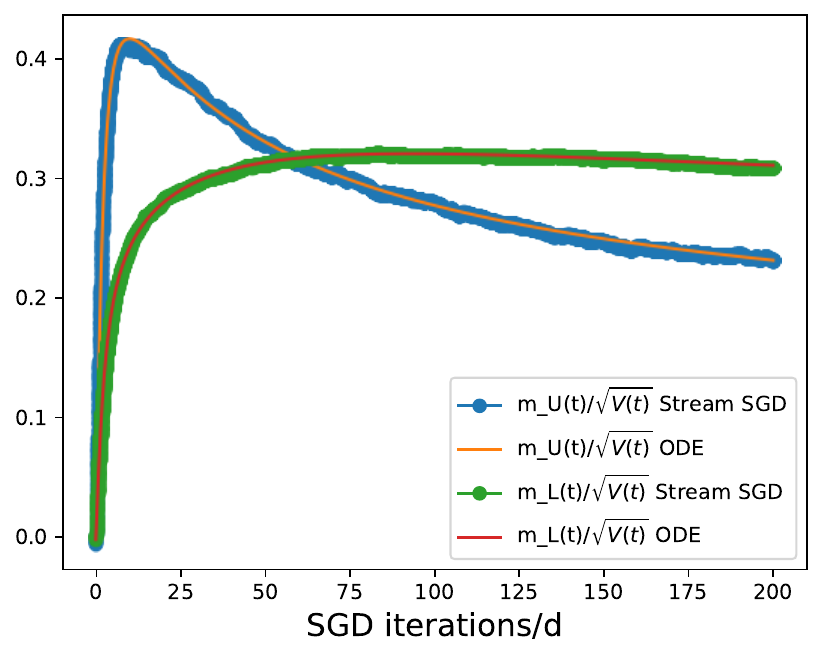}
    \medskip
    \textit{(b) Subspace alignment, $\beta=0.4$}
\end{minipage}\hspace{0.1\textwidth}%
\begin{minipage}[t]{0.25\textwidth}
    \centering
    \includegraphics[height=4cm,keepaspectratio]{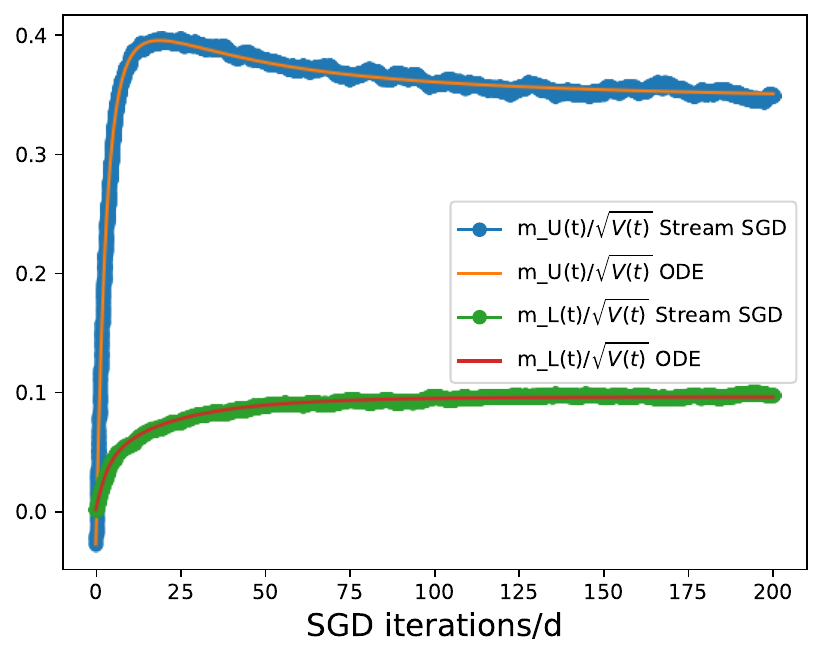}
    \medskip
    \textit{(c) Subspace alignment, $\beta=1.2$}
\end{minipage}
\vspace{0.1cm}
\caption{\textbf{Subspace alignment in the power-law model.} $d=1000$, $a_1=a_2=1.2$. We project $\bar m(t)$ into the lower and upper eigenspaces, obtaining $m_L(t)$ and $m_U(t)$, respectively.}
\label{fig:power_law_align_proj}
\end{figure}

\subsection{High-dimensional ODE description for logistic regression}
Beginning with the more general setting of multi-class logistic regression, let us consider a one-hot encoding label vector $y_I\in  \{0,1\}^{\ell}$ which is a function of the true label $I\in [\ls]\sim p_I.$ The data is then pairs of $(a_I,y_I).$ The risk function considered is the cross-entropy loss: 
\begin{equation}
    f_I(r_I)=-\sum_{i\in [\ell]} y_I^i r_I^{i}+\log\sum_{j=1}^\ell\exp(r_I^{j})\quad\text{where }r_{i}=X^\top a|(a\in\text{class }i)=\begin{bmatrix}r_{i}^1\\ 
    \vdots \\ r_{i}^\ell\\ \end{bmatrix}.
\end{equation}
The expected loss is then, 
\begin{equation}
    \csL(t) = \EE [f_I(r_{t,I})]= \sum_{k = 1}^{\ls} p_k\left(\sum_{i\in [\ell]} -\EE [y_k^i\mid a] m_k^{i}(t)+\EE[\log\sum_{j=1}^\ell\exp(r_{t,k}^{j})]\right).
\end{equation}
In this setting, the gradient $i\in [\ell]$ is given by 
\begin{equation}
    \nabla f_I(r_I)=-y_I+\begin{bmatrix}\w_{I1}\\ \vdots\\ \w_{I\ls}\end{bmatrix}\quad\text{where } \w_{Ij}=\frac{\exp(r_I^j)}{\sum_{k=1}^{\ell}\exp(r_I^k)}\quad\text{for $I\in [\ls]$ and $j\in[\ell]$}.
\end{equation}
We then have that 
\begin{align}
    \partial_k\partial_j f_I(r_I) = \w_{Ij}\delta_{jk} - \w_{Ij}\w_{Ik}
\end{align}
Finally, the Fisher term reads
\begin{align}
    \EE[\partial_k f_I(r_I) \partial_j f_I(r_I)\mid I = i] &= \EE[y_i^j y_i^k]  - \EE[y_i^j]\w_{ik} - \EE[y_i^k]\w_{ij} + \EE[\w_{ij}\w_{Ik}]    
    \\\nonumber &= \delta_{ik}\delta_{ji} - \delta_{ij}\w_{ik} - \delta_{ik}\w_{ij} + \EE[\w_{ij}\w_{ik}]  
\end{align}
where the last transition holds in the case where $\ell=\ls$ and thus $y_I^j = \delta_{Ij}$.

We now specialize to the binary case $\ls=2$.  By the symmetry of the problem, we simplify the analysis by taking $X=\left[\begin{array}{cc}
x &
0
\end{array}\right]$ with $x\in \R ^d. $  We then have that 
\begin{equation}
    f_i(r_i)=-y_ir_i+\log(1+\exp(r_i))\quad\text{where }r_{i}=x^\top a|(a\in\text{class }i),\quad y_i=\mathbf{1}_{\{i=1\}}.
\end{equation}
Furthermore, the $\w$ variables become
\[
\w_{11}=\frac{e^{r_1}}{e^{r_1}+1},\quad
\w_{12}=\frac{1}{e^{r_1}+1},\quad
\w_{21}=\frac{e^{r_2}}{e^{r_2}+1},\quad
\w_{22}=\frac{1}{e^{r_2}+1}.
\]
where we note that $\w_{11}+\w_{12}=\w_{21}+\w_{22}=1.$
Using this, the derivatives of $f_i$ are 
\begin{align}
    f'_i(r_i)=-y_i+\w_{i1} 
=\begin{cases}
    -\w_{12}& i=1,\\
    \w_{21}& i=2,
\end{cases}\qquad
    f''_i(r_i)
    = \w_{i1} \w_{i2}.
\end{align}
Following the above simplification, we obtain that, $\csm_{\rho,1} = - \csm_{\rho,2} = \csm_{\rho}  = d x^\top u_{\rho} u_{\rho}^\top  \mu  \in \R$, similarly we reduce $V_\rho  = \left[\begin{array}{cc}
\csV_\rho & 0 \\
0 & 0
\end{array}\right]$ with  $\csV_\rho  =d x^\top  u_\rho u_\rho^\top  x\in  \R$, and $r_i = x^\top  a_i.$

For the risk, we get 
\begin{align}
\csL(t)&= \EE[f_i(r_i)]=-p_1 \csm(t)+p_1\EE_z[\log(1+\exp(\csm(t)+\sqrt{\csB_1(t)}z))] \\\nonumber &+p_2\EE_z[\log(1+\exp(-\csm(t)+\sqrt{\csB_2(t)}z))]   .
\end{align}
with $\csm(t) = \ip{\mu, X}$, and  $\csB_i(t) = \ip{K_i, X^{\otimes 2}}.$ 
Thus, the system of equations \eqref{eq:V_m_rho_2} reduces to
\begin{align}
    \frac{\dif \csV_\rho}{\dif t}
    &=-2\gamma \csV_\rho\sum_{i=1}^2 p_i \lambda_\rho^{(i)}\EE[w_{i1}w_{i2}]+ 2\gamma \csm_{\rho}(p_1\EE[w_{12}]+p_2\EE[w_{21}])+{\gamma^2}\sum_{i=1}^2 p_i\left(\lambda_\rho^{(i)}+(\ip{\mu,u_\rho})^2\right)\EE[w_{i(\sim i)}^2]\\ \nonumber
    \frac{\dif \csm_{\rho}}{\dif t}&=-\gamma \csm_{\rho}\sum_{i=1}^2 p_i\lambda_\rho^{(i)}\EE[w_{i1}w_{i2}]+ {\ip{\mu, u_\rho}^2}\gamma (p_1 \EE[w_{12}]+p_2 \EE[w_{21}])d
\end{align}
where $w_{ij}$ denotes the deterministic equivalent of $\w_{ij}$ and can be written as $w_{12} = 
\left(1+e^{\csm(t)+\sqrt{\csB_1(t)}z}\right)^{-1}
$ and $
w_{11}
= 1 -w_{12}$ where the expectation is with respect to $z\sim \mathcal{N}(0,1)$. Similarly, $w_{22} = 
\left(1+e^{-\csm(t)+\sqrt{\csB_2(t)}z}\right)^{-1}
$ and, $w_{21}=1-w_{22}$. We note that in the identity covariance case $w_{12}\overset{d}{=}w_{22}, $ and $w_{21}\overset{d}{=}w_{11}.$ Finally, we use $\sim i$ to denote ``not $i$'' (so $w_{i(\sim i)}=w_{12},w_{21}$ for $i=1,2$ respectively).

\section{Gaussian Mixture with MSE Loss: Multi-class Example\label{sec:linear_model}}
To illustrate the broader capability of Theorem~\ref{thm:main_risk_m_v}, we present an example with more than two classes in the context of linear regression. We stress that the results of the theorem also apply to nonlinear multi-class models.

\subsection{Setup and Main Results}

Consider data satisfying Assumptions~\ref{ass:data} and \ref{ass:target} with a soft linear link function. In particular, the label $y\in \mathbb{R}^{\ell^*}$ are generated such that $y\mid i = \langle X^\star, a_i \rangle + \varepsilon$ with $\varepsilon \sim \mathcal{N}(0, \sigma^2/\ell^* I_{\ell^*})$ for all $i\in [\ell^*]$ and the ground truth matrix $X^\star \in \mathbb{R}^{d\times \ell^*}.$ We consider the mean square error loss function $f(r)=\|r-y\|^2/2$ with $r=\ip{X, a}_{\R^d}$.  The population error per class is then
$$\mathcal{L}_i(X) = \frac{1}{2}\text{Tr}((X-X^\star)^\top (K_i+\mu_i \mu_i^\top) (X-X^\star)) +\frac{1}{2}\sigma^2,$$
such that $\mathcal{L}(X) = \sum_{i=1}^{\ell^*}p_i \mathcal{L}_i(X)$.

We note that \cite{mai2019high} considered the superiority of the mean square error estimator in the hard label case for the binary model with a linear regression solution vs a logistic regression solution. Our focus here is on characterizing the behavior of SGD in this setting, in particular when there is a large number of classes. 

\subsection{Deterministic Dynamics and Convergence}

A direct application of Theorem~\ref{thm:main_risk_m_v} and Lemma~\ref{lem:spe_decom_m_S} shows that the following functions, for any $\rho \in [d]$,
$$D_\rho(X) \defas d\text{Tr}((X-X^\star)^\top u_\rho u_\rho^\top (X-X^\star)), \quad \text{and} \quad M_{\rho}(X) \defas d \mu^\top u_\rho u_\rho^\top (X-X^\star) \in \mathbb{R}^{\ell^* \times \ell^*}$$
converge to deterministic limits $\mathscr{D}_\rho(t)$ and $\mathfrak{m}_{\rho,ij}(t)$ that satisfy the following equations:
\begin{align} 
 \frac{d \mathscr{D}_\rho(t)}{dt}  =   &-2\gamma_t \sum_{i=1}^{\ell^*} p_i \lambda_\rho^{(i)}\mathscr{D}_\rho(t)- 2\gamma_t \sum_{j,i=1}^{\ell^*} p_i \mathfrak{m}_{\rho,i,j}(t)m_{i,j}(t) +2\gamma_t^2 \sum_{i=1}^{\ell^*} p_i (\lambda_\rho^{(i)}+\tilde{\mu}_{\rho}^{(i)})\csL_i(t) \label{eq:D_rho}
\end{align}

\begin{align}
 \frac{d \mathfrak{m}_{\rho, j, u}(t)}{dt}  =   &-\gamma_t \sum_{i=1}^{\ell^*} p_i \lambda_\rho^{(i)}\mathfrak{m}_{\rho,j,u}(t) -\gamma_t d\sum_{i=1}^{\ell^*} p_i \mu_j^{\top} u_\rho u_\rho^{\top}\mu_i \mathfrak{m}_{i, u}(t), \label{eq:dm_rho_j_dt}
\end{align}
where $\mathscr{D}(t) = \frac{1}{d} \sum_\rho \mathscr{D}_\rho(t)$, $\mathfrak{m}_{i,u}(t) = \frac{1}{d} \sum_\rho \mathfrak{m}_{\rho,i, u}(t)$, the average overlap for the $u$th direction in $X$ is $\mathfrak{m}_u(t) = \sum_{j=1}^{\ell^*} p_j\mathfrak{m}_{j,u}$, and the loss per class is
$$\csL_i(t) =\frac{1}{2d}\sum_\rho \lambda_\rho^{(i)}\mathscr{D}_\rho(t) +\frac{1}{2}\sum_{j=1}^{\ell^*} \mathfrak{m}_{i,j}^2+\frac{\sigma^2}{2}.$$
The Fisher term for the $i$th class is then $\mathbb{E}[\nabla f_i(\theta_{t,i})^{\otimes 2}] = 2\csL_i(t)$. 

\begin{remark}[Orthogonal means]
Averaging over $\rho$ in Eq.~\eqref{eq:dm_rho_j_dt}, if the means are orthogonal, the second term vanishes. Therefore, if optimization starts from zero, then $\mathfrak{m}_{j,u}(t) = 0$ for all $t$, although each component may vary over time. A better observable is $\frac{1}{d}\sum_{j=1}^{\ell^*}\sum_{\rho=1}^d p_j \mathfrak{m}_{\rho ju}^2$, which measures the variance of the eigenmodes. See Figures~\ref{fig:lin_13_m} and \ref{fig:lin_05_m} for the extreme and mild power-law settings with random means, which in high dimensions are nearly orthogonal.
\end{remark}

\subsection{Stepsize Condition}

The following lemma provides a stability condition on the learning rate:

\begin{lemma}[Linear least square GMM]\label{lem:linear_lr_thresh}
Suppose the data is generated as above with $\sigma = 0$. Then online SGD for the mean square error converges to zero for any learning rate satisfying
$$\gamma_t < \frac{1}{\max_{i\in [\ell^*]}\frac{1}{d}(\textup{Tr}(K_i) + \|\mu_i\|^2)}.$$
\end{lemma}

\begin{proof}
Summing over $\rho$ in Eq.~\eqref{eq:D_rho} and dividing by $d$, for $\sigma =0$, we have
\begin{align}
 \frac{d \mathscr{D}(t)}{dt}  =   -2\gamma_t \sum_{i=1}^{\ell^*} p_i \csL_i(t)\left(1-\gamma_t \frac{1}{d}(\text{Tr}(K_i) + \|\mu_i\|^2)\right).
\end{align}
where $\scrD(t) = \frac{1}{d} \sum_\rho \scrD_\rho(t)$, $\csm_{i,u}(t) = \frac{1}{d} \sum_\rho \csm_{\rho,i, u}(t),$ and using the definition of the loss per class $\csL_i(t) =\frac{1}{2d}\sum_\rho \lambda_\rho^{(i)}\csD_\rho(t) +\frac{1}{2}\sum_{j=1}^{\ls} \csm_{i,j}^2.$ The distance to optimality converges for learning rates satisfying the stated bound.
\end{proof}
\begin{remark}
Unlike standard worst-case convergence analyses in optimization that depend on $\|K_i\|_{\text{op}}$ (the largest eigenvalue), Lemma~\ref{lem:linear_lr_thresh} provides a stepsize condition depending on $\text{Tr}(K_i)/d$, which captures the average spectral content across all directions. This can be substantially less restrictive when eigenvalues are distributed diffusely in high dimensions. This extends the results in \cite{collinswoodfin2023hitting} to the multiclass setting.
\end{remark}

\subsection{Numerical Simulations and Results}

We conclude this section with numerical simulations illustrating SGD iterates and our theoretical predictions derived from ODEs~\eqref{eq:D_rho} and \eqref{eq:dm_rho_j_dt}, for a large number of classes ($\ell^*=5,10,20,100$), with $d=100$ for $\ell^*=100$ and $d=1000$ for $\ell^*=5,10,20$, $\mu_i\sim\mathcal{N}(0,I_d/d)$, $X_{ij}^\star\sim \mathcal{N}(0,1)/\sqrt{d\ell^*}$, and $p_i=1/\ell^*$. We plot loss curves for mild ($\alpha = 0.5$) and extreme ($\alpha = 1.3$) power-law settings (Figures~\ref{fig:lin_05_L} and \ref{fig:lin_13_L}).

In addition to the loss, we examine the following two ratios, which effectively demonstrate the effect of different covariance matrices and class structures:
\begin{align}
\frac{1}{\ell^*}\sum_{ij} p_i\frac{\mathfrak{m}_{ij}^2(t)}{\mathscr{D}(t)} \quad \text{and} \quad \frac{1}{\ell^* d}\sum_{ij\rho} p_i\frac{\mathfrak{m}_{\rho ij}^2(t)}{\mathscr{D}(t)}.
\end{align}

As seen in Figure \ref{fig:lin_05_m} (mild power-law) and Figure \ref{fig:lin_13_m} (extreme power-law), the first measure behaves similar in both models, as expected for random means, whereas the second measure decreases in the extreme power-law setting and first decreases then increases in the mild power-law case, indicating the model's ability to recover some information about the means and the hidden true parameter $X^\star$.

The difference $\frac{1}{\ls d}\sum_{ij\rho} p_i\frac{\csm_{\rho ij}^2(t)}{\csD(t)}-\frac{1}{\ls}\sum_{ij} p_i\frac{\csm_{ij}^2(t)}{\csD(t)}$ can be viewed as measuring the variability of $\csm_{\rho ij}^2(t)$ across the different eigenvector directions.  We see that, with the extreme power-law, this difference is monotonically increasing, suggesting that the iterates focus more in certain eigenvector directions, as we saw with logistic regression.  With the mild power-law, however, we observe that this difference increases and then decreases again, at least in some cases. This suggests that iterates initially focus in certain eigenvector directions but, as SGD learns more about the means and $X^\star$, it focuses less on particular eigenvector directions (provided that the power-law is sufficiently mild and the number of classes is sufficiently large).

We note that, although our theorem allows for a number of classes logarithmic in the dimension, in all the examples plotted, the ODE perfectly predicts the loss curve even beyond the regime of our theorem. We observe some discrepancies when examining the second moment of the mean overlap, mainly for the mild power-law case, where determining the mean is more challenging due to noise in all directions. This discrepancy becomes more apparent when $\ls\asymp d$, which exceeds the number of classes covered by our theorem. In addition, we observe that, in the mild power-law setting, the loss decreases faster toward zero, suggesting that learning is more efficient when the variance is more uniform across different directions. As the number of classes grows, the decrease in the loss is even faster (see Figure \ref{fig:lin_05_L}). This agrees with the binary logistic regression, where the convergence of the loss in the extreme power-law regime is slower than in the mild power-law regime.

\begin{figure}[t]
\centering

\begin{minipage}[t]{0.23\textwidth}
    \centering
    \includegraphics[width=\linewidth,keepaspectratio]{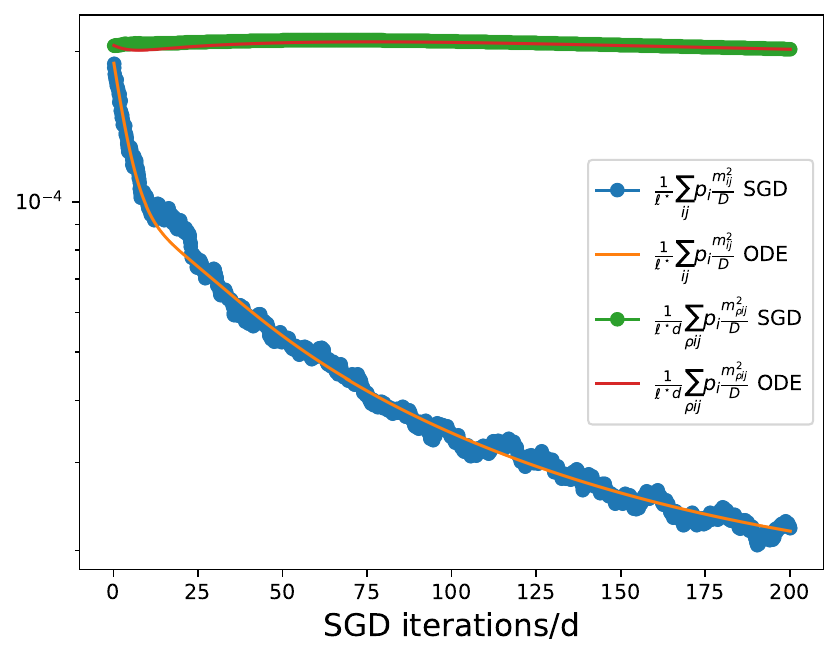}
    \medskip
    \textit{(a) $\ell_s=5$}
\end{minipage}\hspace{0.01\textwidth}%
\begin{minipage}[t]{0.23\textwidth}
    \centering
    \includegraphics[width=\linewidth,keepaspectratio]{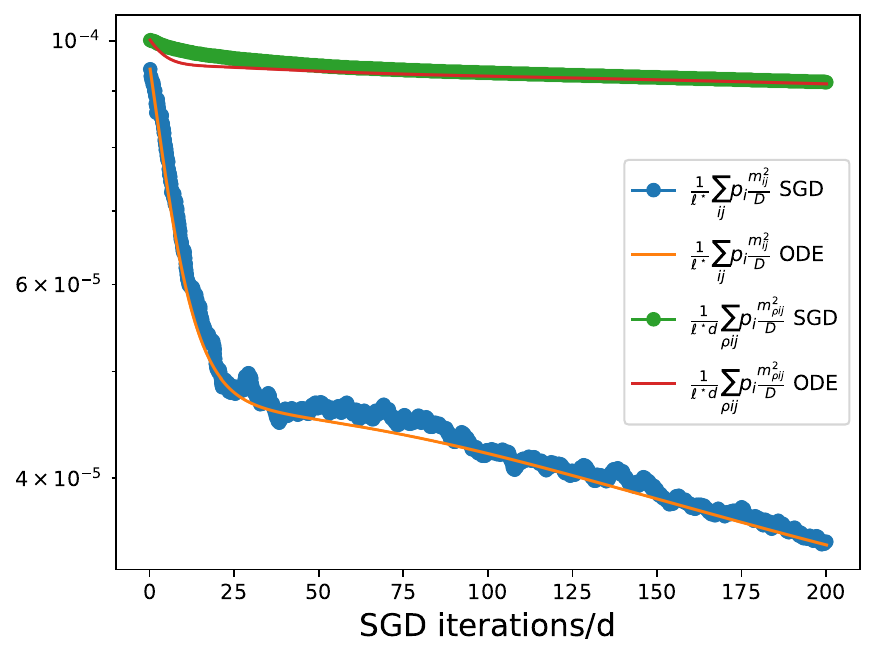}
    \medskip
    \textit{(b) $\ell_s=10$}
\end{minipage}\hspace{0.01\textwidth}%
\begin{minipage}[t]{0.23\textwidth}
    \centering
    \includegraphics[width=\linewidth,keepaspectratio]{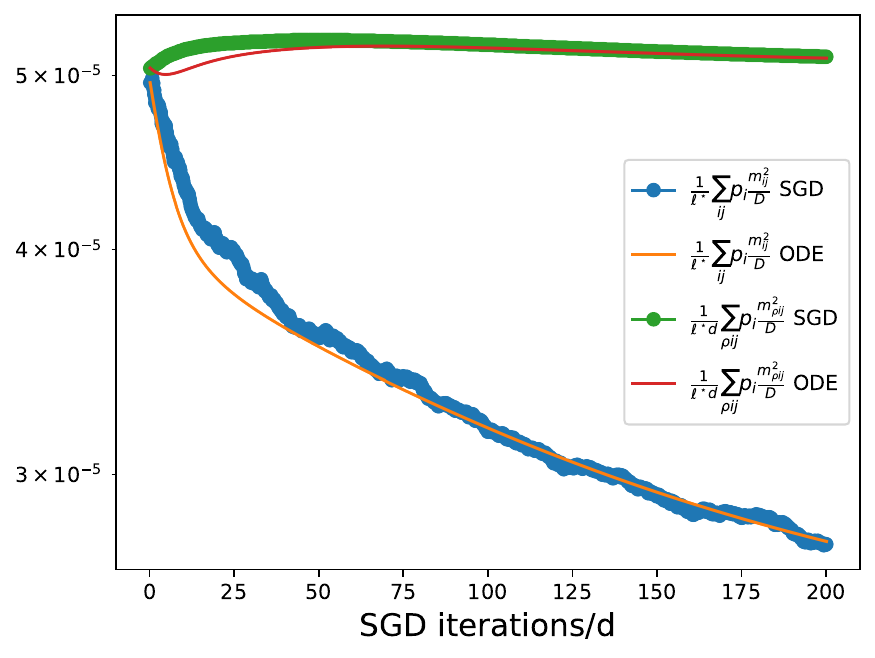}
    \medskip
    \textit{(c) $\ell_s=20$}
\end{minipage}\hspace{0.01\textwidth}%
\begin{minipage}[t]{0.23\textwidth}
    \centering
    \includegraphics[width=\linewidth,keepaspectratio]{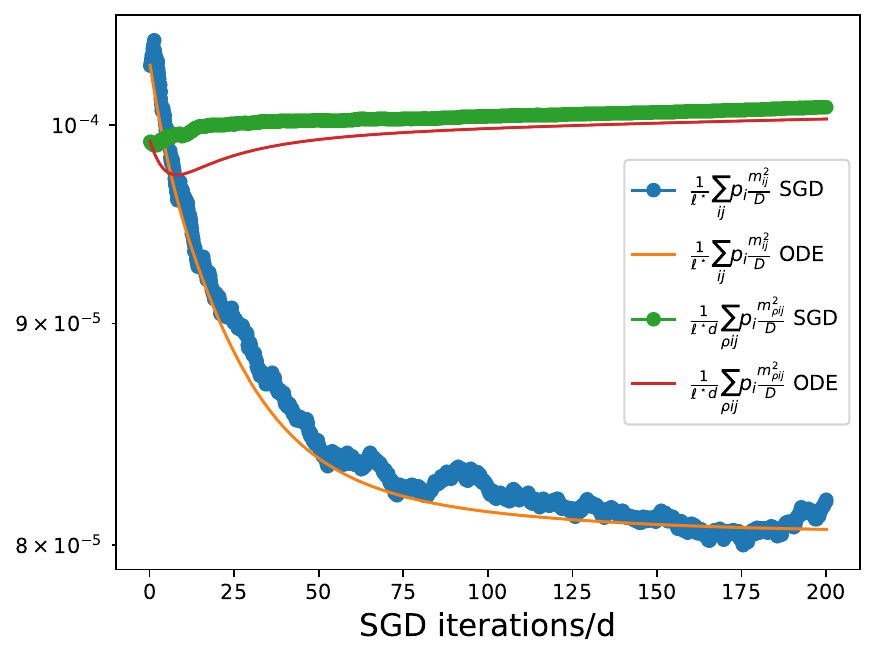}
    \medskip
    \textit{(d) $\ell_s=100$}
\end{minipage}

\caption{\textbf{Alignments for extreme power-law with MSE loss}, $\alpha = 1.3$, and random means.}
\label{fig:lin_13_m}
\end{figure}
\begin{figure}[t]
\centering

\begin{minipage}[t]{0.24\textwidth}
    \centering
    \includegraphics[width=\linewidth,keepaspectratio]{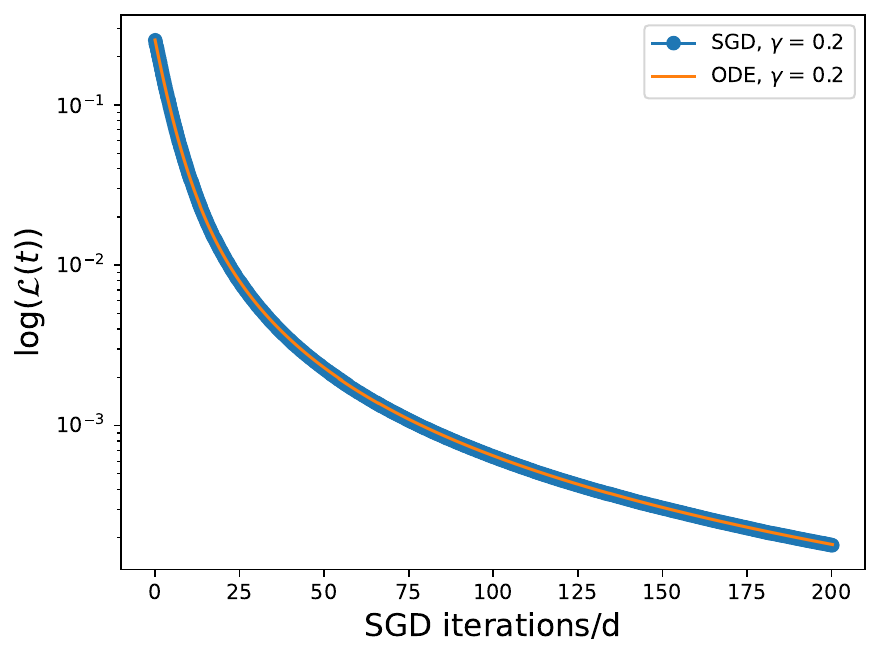}
    \medskip
    \textit{(a) $\ell_s=5$}
\end{minipage}\hspace{0.01\textwidth}%
\begin{minipage}[t]{0.24\textwidth}
    \centering
    \includegraphics[width=\linewidth,keepaspectratio]{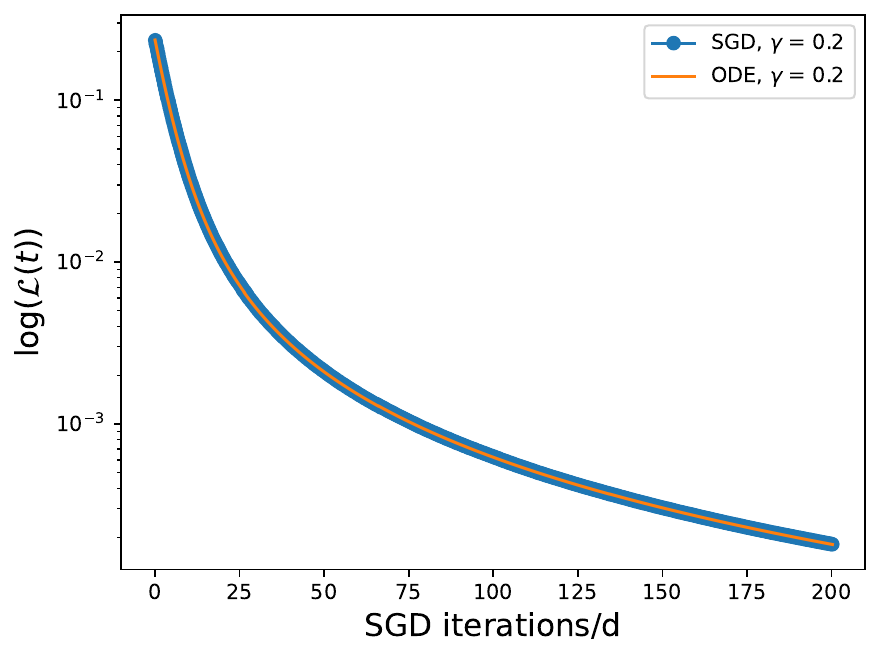}
    \medskip
    \textit{(b) $\ell_s=10$}
\end{minipage}\hspace{0.01\textwidth}%
\begin{minipage}[t]{0.24\textwidth}
    \centering
    \includegraphics[width=\linewidth,keepaspectratio]{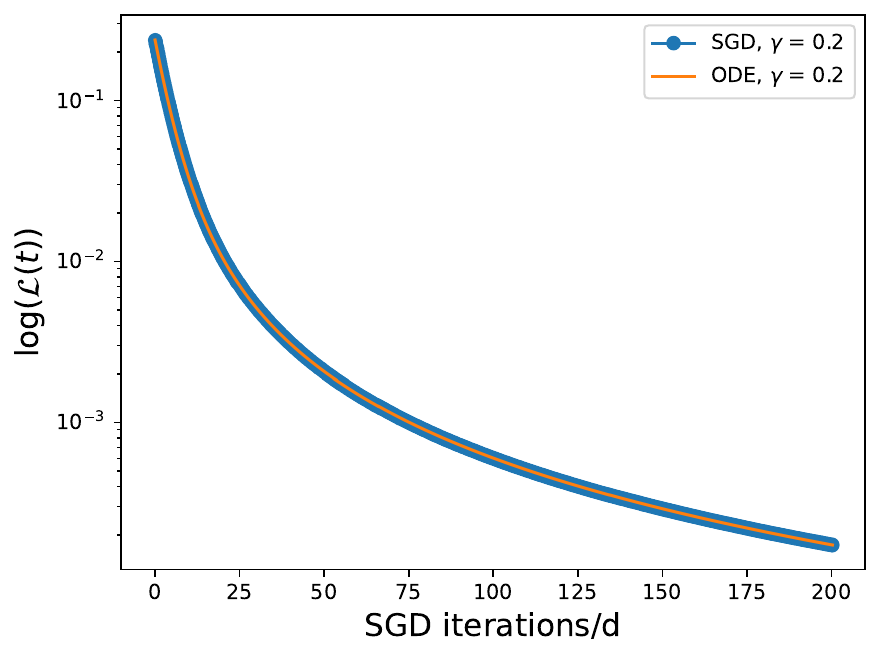}
    \medskip
    \textit{(c) $\ell_s=20$}
\end{minipage}\hspace{0.01\textwidth}%
\begin{minipage}[t]{0.24\textwidth}
    \centering
    \includegraphics[width=\linewidth,keepaspectratio]{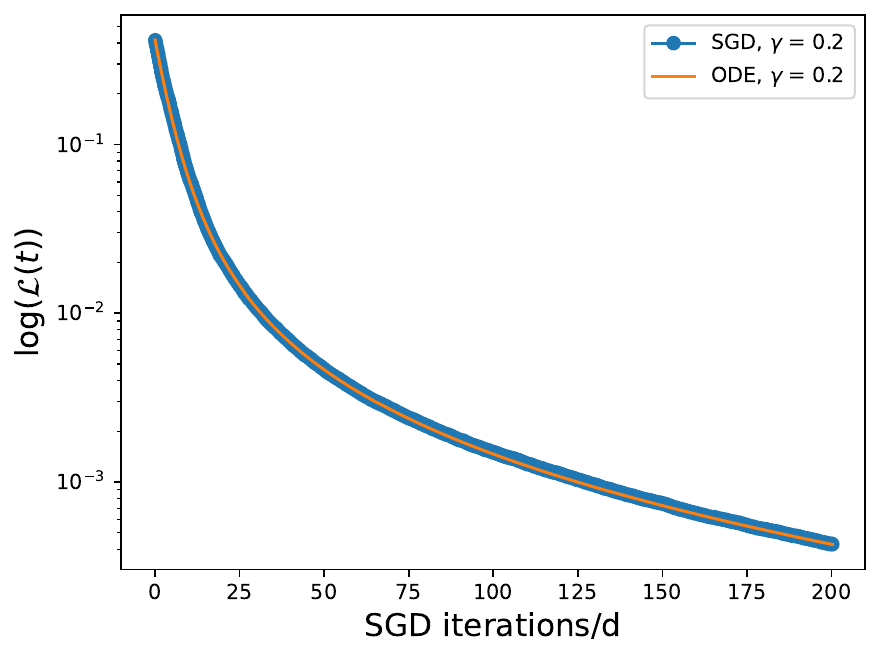}
    \medskip
    \textit{(d) $\ell_s=100$}
\end{minipage}

\caption{\textbf{Learning curves for extreme power-law with MSE loss.} $\alpha = 1.3$ and random means.}
\label{fig:lin_13_L}
\end{figure}
\begin{figure}[t]
\centering

\begin{minipage}[t]{0.23\textwidth}
    \centering
    \includegraphics[width=\linewidth,keepaspectratio]{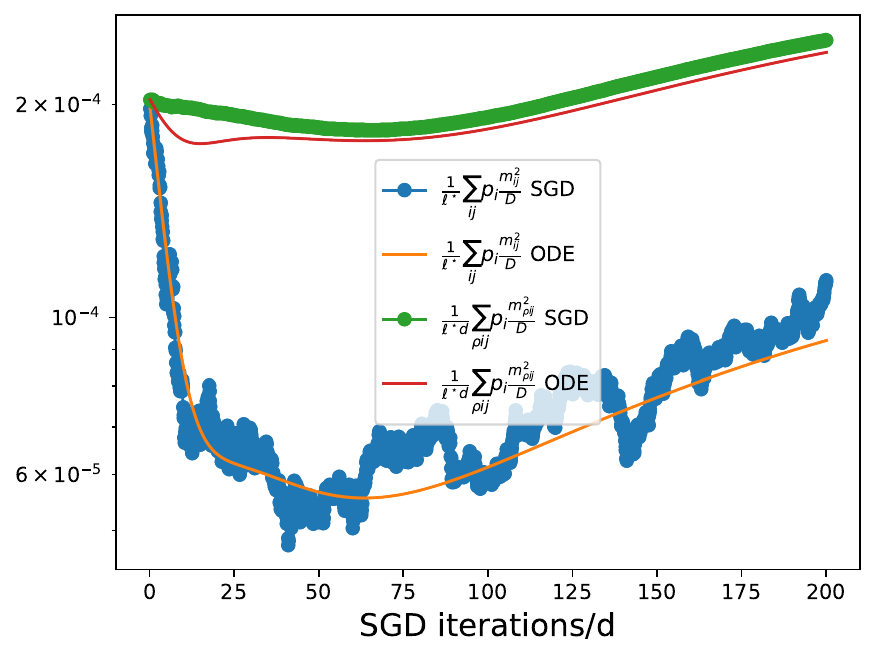}
    \medskip
    \textit{(a) $\ell_s=5$}
\end{minipage}\hspace{0.01\textwidth}%
\begin{minipage}[t]{0.23\textwidth}
    \centering
    \includegraphics[width=\linewidth,keepaspectratio]{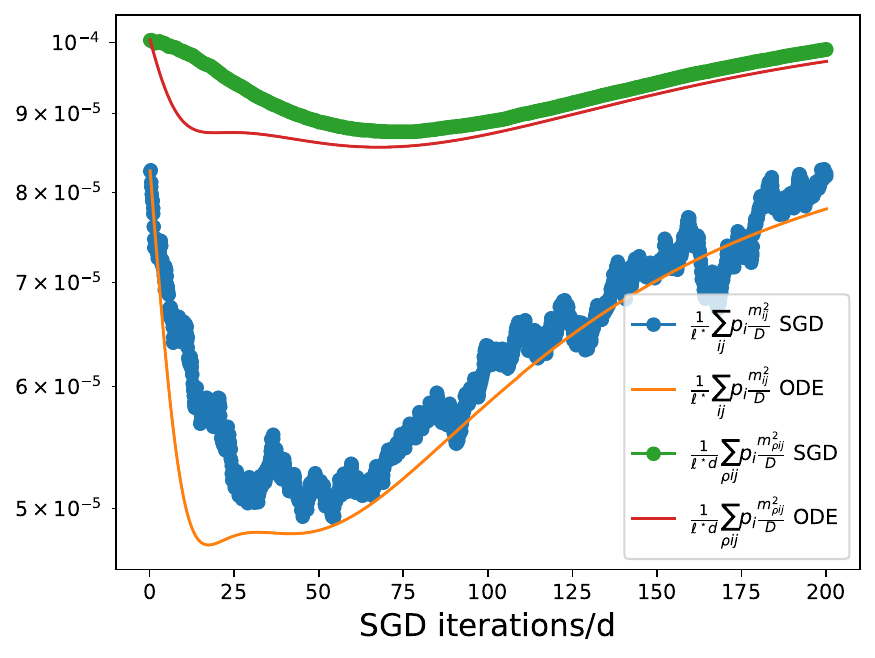}
    \medskip
    \textit{(b) $\ell_s=10$}
\end{minipage}\hspace{0.01\textwidth}%
\begin{minipage}[t]{0.23\textwidth}
    \centering
    \includegraphics[width=\linewidth,keepaspectratio]{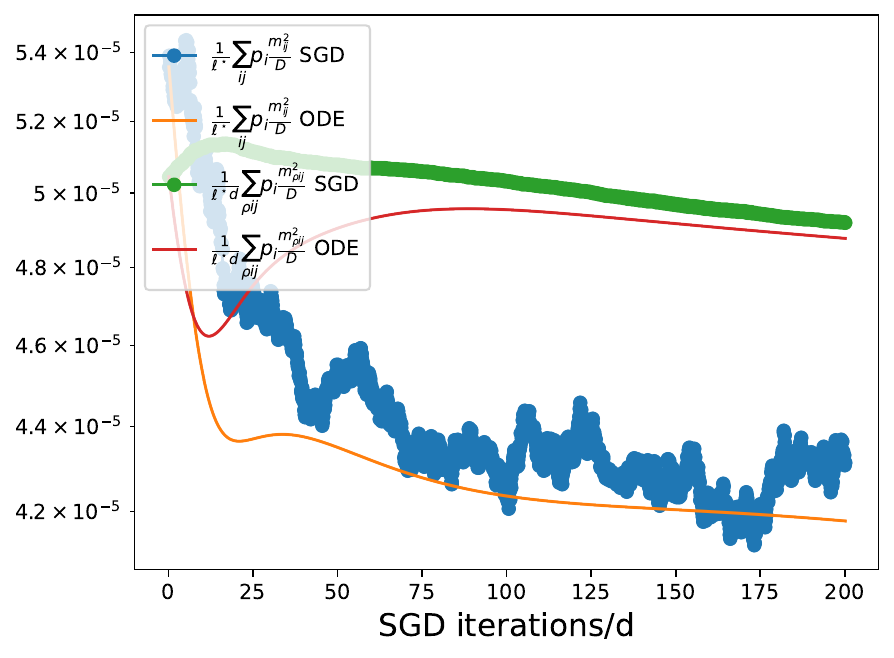}
    \medskip
    \textit{(c) $\ell_s=20$}
\end{minipage}\hspace{0.01\textwidth}%
\begin{minipage}[t]{0.23\textwidth}
    \centering
    \includegraphics[width=\linewidth,keepaspectratio]{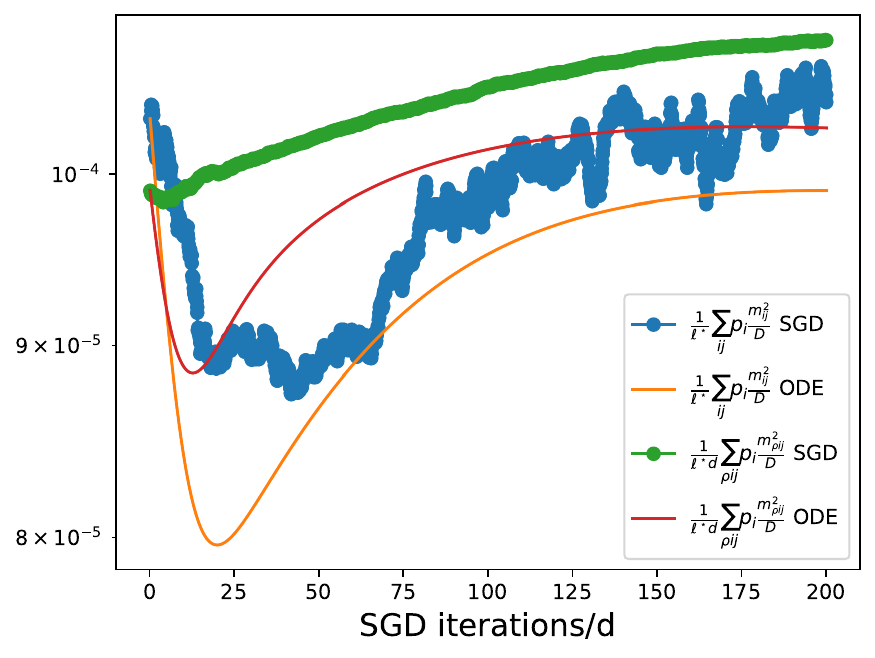}
    \medskip
    \textit{(d) $\ell_s=100$}
\end{minipage}

\caption{\textbf{Alignments for mild power-law with MSE loss.} $\alpha = 0.5$ and random means.}
\label{fig:lin_05_m}
\end{figure}
\begin{figure}[t]
\centering

\begin{minipage}[t]{0.24\textwidth}
    \centering
    \includegraphics[width=\linewidth,keepaspectratio]{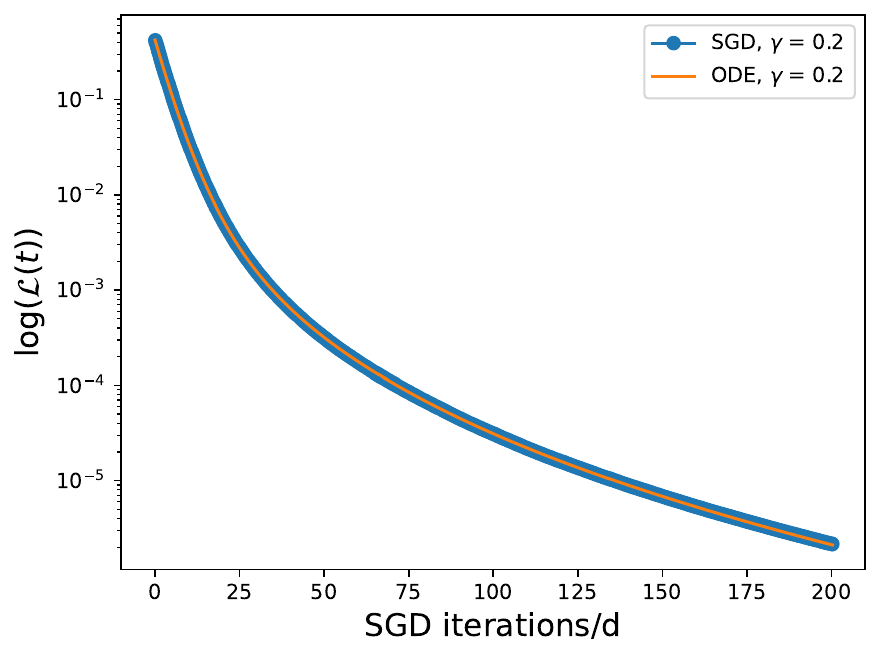}
    \medskip
    \textit{(a) $\ell_s=5$}
\end{minipage}\hspace{0.01\textwidth}%
\begin{minipage}[t]{0.24\textwidth}
    \centering
    \includegraphics[width=\linewidth,keepaspectratio]{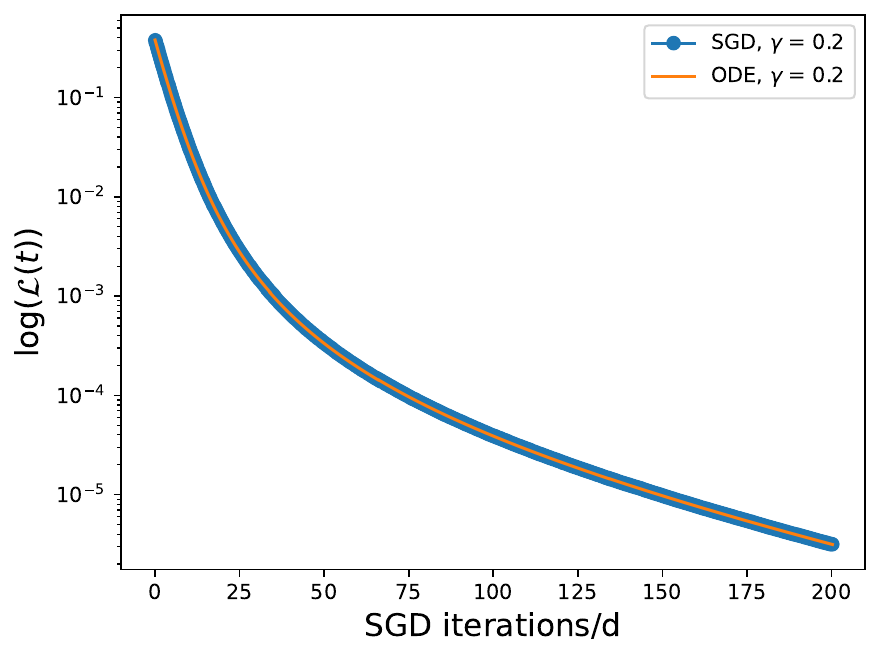}
    \medskip
    \textit{(b) $\ell_s=10$}
\end{minipage}\hspace{0.01\textwidth}%
\begin{minipage}[t]{0.24\textwidth}
    \centering
    \includegraphics[width=\linewidth,keepaspectratio]{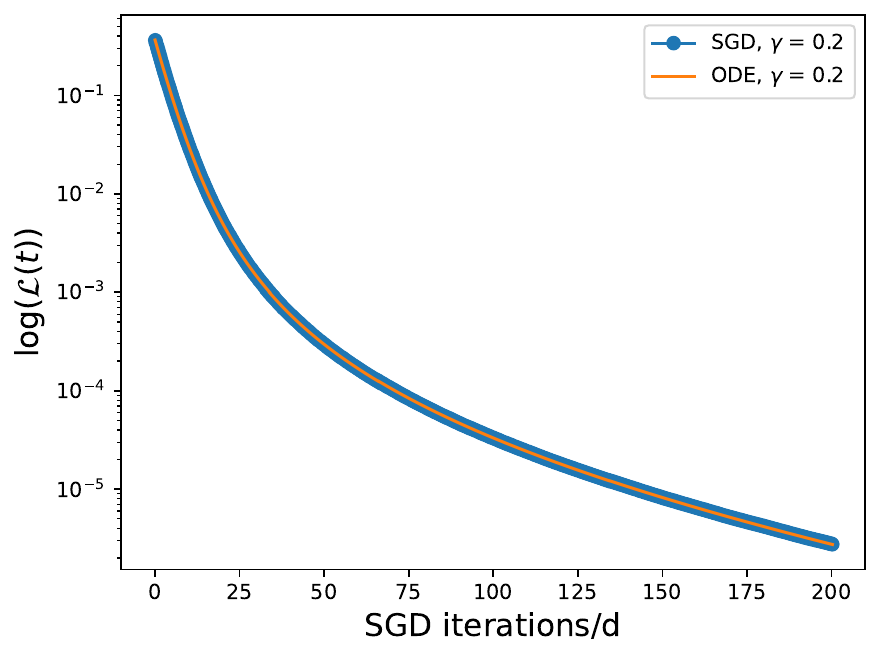}
    \medskip
    \textit{(c) $\ell_s=20$}
\end{minipage}\hspace{0.01\textwidth}%
\begin{minipage}[t]{0.24\textwidth}
    \centering
    \includegraphics[width=\linewidth,keepaspectratio]{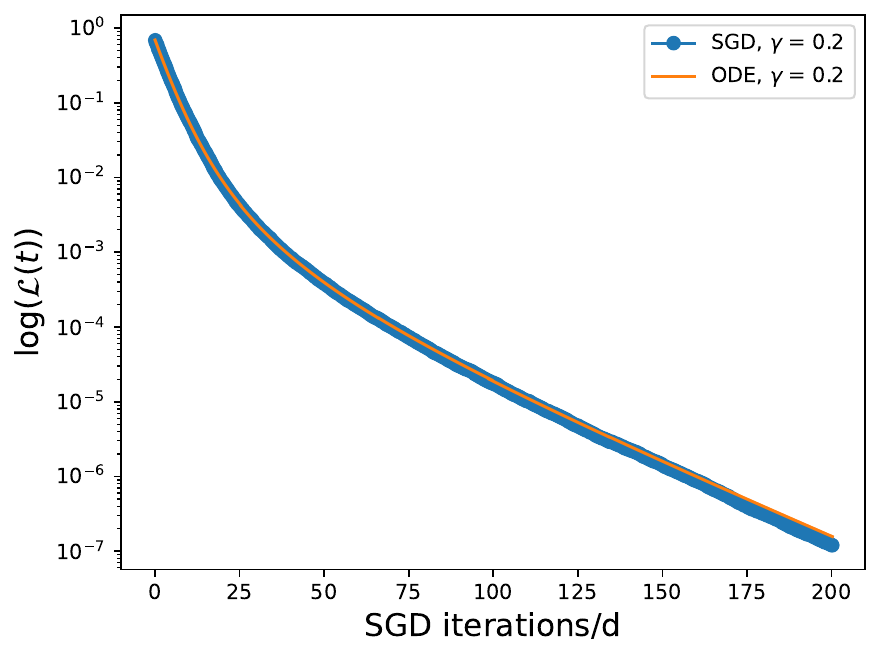}
    \medskip
    \textit{(d) $\ell_s=100$}
\end{minipage}

\caption{\textbf{Learning curves for mild power-law with MSE loss.} $\alpha = 0.5$ and random means.}
\label{fig:lin_05_L}
\end{figure}

\section{Proof of Theorem \ref{thm:main_risk_m_v}\label{sec:proof_main_thm}}
In this section, we prove the main theorem (Theorem \ref{thm:main_risk_m_v}) on the concentration of the risk curves and other statistics. Subsection \ref{sec:doob_decomp} provides the first steps for setting up the proof.  Subsection \ref{sec:integrodiff} provides the resolvent formulation of our solution, which is equivalent to system of ODEs given in Theorem \ref{thm:main_risk_m_v}, but more amenable to our proof techniques.  The heart of the argument comes in Subsection \ref{sec:approx_sol}, where we prove that SGD under the statistics gives an approximate solution to the integro-differential equation (see Proposition \ref{prop:approx_sol}).  Theorem \ref{thm:main_risk_m_v} follows from Proposition \ref{prop:approx_sol}, modulo some supporting lemmas proved in Subsections \ref{sec:supporting_lemmas} and \ref{sec:error_bds}.

\subsection{Doob decomposition}\label{sec:doob_decomp}

In this section, we aim to provide some intuition for where our solution comes from and also set up the key quantities that we will need to bound in the proof.

Our starting point is a Taylor expansion of $\varphi(W_{k+1})$ at the point $W_k$ (recall that $W_k=\hat{X}_k\oplus\mu$ where $\hat{X}_k=X_k$ in the hard label case and $\hat{X}_k=X_k\oplus X^\star$ in the soft label case).  
We begin with the simple case where $\varphi$ is quadratic, and it will turn out that this is all we need. Using the SGD update rule \eqref{eq:SGD_update_rule}, 
and setting $\Delta_{k,i} \defas  a_{k+1,i}\otimes \nabla_x f_{i}(r_{k,i},\e_{k+1})$ with $r_{k,i}\defas \ip{a_{k+1,i},\hat{X}_k}$ we obtain
\begin{align}\label{eq:Taylor_expansion}
    &\varphi(W_{k+1}) =  \varphi(W_{k})-\frac{\gamma_k}{d}\ip { \nabla\varphi(W_{k}), \Delta_{k,I_{k+1}}} 
   +  \frac{\gamma_k^2}{2d^2}\ip{\nabla^2\varphi(W_{k}),\Delta_{k,I_{k+1}}^{\otimes 2} } 
\end{align}

We define the $\sigma$-algebra $\cF_k \defas \sigma(\{W_i\}_{i=0}^k)$ generated by the iterates of SGD. We now relate this equation to its expectation by performing a Doob decomposition, involving the following martingale increments and error terms:
\begin{align}
\Delta \mathcal{M}_k^{\text{Grad}}(\varphi) &\defas  \frac{\gamma_k}{d} \ipa{\nabla \varphi(W_k),  \Delta_{k,I_{k+1}}   
      - 
     \sum_{i = 1}^{\ls} p_i \EE \big [\Delta_{k,i}  \, \mid \, \mathcal{F}_k \big ]}, \\
\Delta \mathcal{M}_k^{\text{Hess}}(\varphi) &\defas \frac{\gamma_k^2}{2d^2} \left( \ipa{\nabla^2\varphi(W_{k}),\Delta_{k,I_{k+1}}^{\otimes 2} } - \ipa{\nabla^2\varphi(W_{k}),\sum_{i=1}^{\ls}p_i\EE \big [ \Delta_{k,i}^{\otimes 2} \,  \mid  \, \mathcal{F}_k \big ]}\right) ,
\label{eq:Hessian_Martingale}\\
\EE [\mathcal{E}_{k}^{\text{Hess}}(\varphi) \, | \, \cF_{k} ]    &\defas \frac{\gamma_k^2}{2d^2} \ipa{\nabla^2\varphi(W_{k}),\sum_{i=1}^{\ell}p_i\left(\EE \big [ \Delta_{k,i}^{\otimes 2} -   (K_i+\mu_i\mu_i^\top) \otimes \nabla_x f_{i}(r_{k,i},\e_{k+1})^{\otimes2}\,  \mid  \, \mathcal{F}_k \big ]\right)}.
\end{align}
Define $\cL_i(X_k)\defas \EE[f_i(r_{k,i})|I_{k+1} = i, \cF_k]$ for any $i \in [\ls]$ the gradient term then simplifies, 
\begin{equation}
\EE\left[\ip { \nabla\varphi(W_{k}), a_{k+1,I_{k+1}}\otimes \nabla_x f_{I_{k+1}}(r_{k,I_{k+1},\e_{k+1}})} \mid \cF_k\right]=\sum_{i=1}^{\ls}p_i \ip{\nabla\varphi(W_{k}), \nabla{\cL}_i(X_k)}.
\end{equation}
We can then write
\begin{align*}
    \varphi(W_{k+1}) &= \varphi(W_k) - \frac {\gamma_k}{d}\sum_{i=1}^{\ls} p_i\ip{\nabla \varphi(W_{k}),\nabla \cL_i(X_{k})}  
    \\\nonumber &+ \frac{\gamma_k^2}{2d^2} \sum_{i=1}^{\ls}p_i\ip{\nabla^2\varphi(W_{k}), (K_i+\mu_i\mu_i^\top) \otimes \EE [\nabla_x f_{i}(r_{k,i},\e_{k+1})^{\otimes2}\,  \mid  \, \cF_k \big ]} \\ 
    &+ \Delta \mathcal{M}_k^{\text{Grad}}(\varphi) 
    + \Delta \mathcal{M}_k^{\text{Hess}}(\varphi) 
    + \EE [\mathcal{E}_{k}^{\text{Hess}}(\varphi) \, | \, \cF_{k} ]    
\end{align*}
Extending $W_k$ into continuous time by defining $W_t = W_{\lfloor t \rfloor}$ and similarly for $r_{k,i}$, we sum up (integrate).  Recall also $\gamma(s)=\gamma_{\lfloor sd\rfloor}.$ For this, we introduce the forward difference
\[
(\Delta \varphi)(W_{j}) \defas \varphi(W_{j+1}) - \varphi(W_{j}),\]
giving us 
\[\varphi(W_{td}) = \varphi(W_0) + \sum_{j=0}^{\lfloor td \rfloor-1} (\Delta \varphi)(W_j)  \defas \varphi(W_0) + \int_0^{t} d \cdot (\Delta \varphi)(W_{sd}) \, \dif s + \xi_{td}(\varphi),
\]
where $ \displaystyle |\xi_{td}(\varphi)|= \bigg | \int_{(\lfloor td\rfloor-1)/d}^t d \cdot \Delta \varphi(W_{sd})  \,\dif s \bigg | \le \max_{0 \le j \le \lfloor td \rfloor} \{ | \Delta \varphi(W_j) | \}.$ With this, we obtain the Doob decomposition for SGD:
\begin{align}\label{eq:Doob_varphi}
\varphi(W_{td}) 
&= \varphi(W_0) - \sum_{i=1}^{\ls}p_i\int_0^{t} \gamma(s) \ip{\nabla \varphi(W_{sd}), \nabla \cL_i(X_{sd})} \, \dif s  \\
&
        +\frac{1}{2}\sum_{i=1}^{\ls}p_i\int_0^{t}\gamma({s})^2 \ipa{\nabla^2\varphi(W_{sd}), \frac{1}{d}(K_i+\mu_i\mu_i^\top) \otimes \EE [\nabla f_{i}(r_{sd,i},\e)^{\otimes2}\,  \mid  \, \cF_{sd} \big ]} 
 \, \dif s  \nonumber
        \\
        & +  \sum_{j=0}^{\lfloor td \rfloor-1} \mathcal E_j^{\text{all}}(\varphi)  + \xi_{td}(\varphi),
        \nonumber 
\\
\text{with} \quad \mathcal E^{\text{all}}_j(\varphi) &= 
\Delta \mathcal{M}_{j}^{\text{Grad}}(\varphi) 
+ \Delta \mathcal{M}_{j}^{\text{Hess}}(\varphi) \label{eq:error_terms_integrated}  + \EE [\mathcal{E}_{j}^{\text{Hess}}(\varphi) \, | \, \mathcal{F}_{j} ]\nonumber. 
\end{align}

Having obtained this Doob decomposition, the remaining tasks are to show that the first two lines of \eqref{eq:Doob_varphi} are well-approximated by our deterministic representation, while the error terms $\mathcal E_j^{\text{all}}(\varphi)$ and $\xi_{td}(\varphi)$ are negligible.  The form of \eqref{eq:Doob_varphi} naturally suggests that statistics $\varphi(\cdot)$ should be expressible in terms of a differential equation involving derivatives of $\varphi,f_i,$ and their inner products with $K_i$ and $\mu_i$. It is not hard to see that, if one tries to analyze a class of all statistics arising in this way, one needs to also analyze statistics involving higher powers of $K_i$.  A concise and powerful way to control all such statistics is through resolvents.  We build on the approach used in \cite{collinswoodfin2023hitting}.  That paper required only a single resolvent, but since we have multiple covariance matrices and the dynamics generates higher-order polynomials of combinations of them, we work with a product of resolvents.  

In the next sub-section, we describe the resolvent formulation of our statistics and the resulting integro-differential equations.

\subsection{Integro-differential equation set-up}\label{sec:integrodiff}
In this section, we will introduce the integro-differential equation \eqref{eq:ODE_resolvent_2} (see below for specifics) and prove its equivalence to the differential equations for $\csV_\rho(t),\csm_{\rho,j}(t)$ that were introduced in Section \ref{sec:main_result_rho_ODEs_thm} in the main text.  While those differential equations can be easier to implement in the theorem, the representation described in this section will be an important tool for our proof.

The key idea behind our integro-differential equation \eqref{eq:ODE_resolvent_2} is that the resolvent of $K_i$ can be used to encode a variety of statistics by integrating functions of the resolvent on a contour encircling the spectrum of $K_i$.  We shall always work on a fixed contour $\Omega$ surrounding the spectra of $\{K_i\}_{i=1}^\ls$, given by $\Omega \stackrel{\text { def }}{=}\left\{w\in\mathbb{C}:|w|=\max \left\{2,2\|\cK\|_{\ls}\right\}\right\}$ with $\|\cK\|_{\ls} = \max_{i\le \ls}\|K_i\|_{\text{op}}$. We note that this contour is always a distance of at least 1 from the spectrum of any $K_i$.
When referring to complex integration over all $z_i$ we write for any function $f: \mathbb{C}^{\ls}\to \mathbb{C}.$ $$\oint f(z) \Dif z  \defas \oint_{\Omega^{\ls}} f(z)\prod_{j=1}^{\ls} \frac{-1}{2\pi \rm i} \dif z_j  $$
Define $\mathcal{K}=\{K_i\}_{i=1}^\ls$ and $z\in\mathbb{C}^\ls.$ We introduce the resolvent product
\begin{equation}\label{eq:R_ls_def}
    \cR_{\ls}(\mathcal{K},z)\defas\prod_{i=1}^\ls (K_i-z_iI_d)^{-1},\qquad z_i\in\Omega,
\end{equation}
and the mean matrix $\mu\defas\mu_1\oplus\cdots\oplus\mu_\ls$. We define two important statistics that will capture the dynamics of SGD:
\begin{equation}
    S(W,z)\defas \hat{X}^\top\cR_\ls \hat{X},\qquad M(W,z)\defas \hat{X}^\top\cR_\ls\mu.
\end{equation}
where (recalling $\lpp\defas\ell+\ls$) we have $S(W,z)\in\R^{\bar{\ell}\times\bar{\ell}}$, and $M(W,z)\in\R^{\bar{\ell}\times\ls}$, and 
\begin{equation}\label{def:Z(X,z)}
Z(W,z) = 
W^\top \cR_\ls W  = 
\begin{bmatrix}
S(W,z) 
& M(W,z)
\\
M(W,z)^\top  
 & \mu^\top \cR_\ls\mu
\end{bmatrix}\in \R^{ \lp \times \lp }
\end{equation}
 such that $W = \hat{X}\oplus \mu$, with $\hat{X} = X \oplus X^\star$ where we denote by $\lp \defas \ell+2\ls.$  
 
 We will show that $S,M,Z$ are approximate solutions to a system for $\mathscr{S},\mathscr{M},\mathscr{Z}$ described below.
\paragraph{Integro-Differential Equation for $(\mathscr{S}(t, z), \mathscr{M}(t, z))$.}
For any contour $\Omega \subset \mathbb{C}$ enclosing the eigenvalues of $\mathcal{K}$, we have an expression for the derivative of $\mathscr{Z}$ with respect to $t$:
\begin{equation}\label{eq:ODE_resolvent_2}
    \dif \mathscr{Z}(t,\cdot) 
    = \mathscr{F}(z, \mathscr{Z}(t, \cdot)) \, \dif t,\qquad 
\text{with initial condition }\mathscr{Z}(0,z)=Z(W_0,z)
\end{equation}
where $\mathscr{Z}(t, z) = \begin{bmatrix}
\mathscr{S}(t,z) &  \mathscr{M}(t,z)^\top \\
 \mathscr{M}(t,z) & \mu^\top \cR_\ls\mu
\end{bmatrix}$
and $$\mathscr{F}(z, \mathscr{Z}(t, \cdot)) \defas \begin{bmatrix}
\mathscr{F}_\mathscr{S}(z, \mathscr{Z}(t, \cdot))&  \mathscr{F}_\mathscr{M}(z, \mathscr{Z}(t, \cdot))^\top \\
 \mathscr{F}_\mathscr{M}(z, \mathscr{Z}(t, \cdot)) & 0_{\ls\times \ls}
\end{bmatrix}.$$ The maps $\mathscr{F}_\mathscr{S}, \mathscr{F}_\mathscr{M}$ are defined as follows:
\paragraph{Hard labels}
In this case, $X^\star = 0$ and $\hat{X} = X,$ therefore $\csS\in \mathbb{C}^{\ell \times \ell}$ and  $\csM\in \mathbb{C}^{\ell \times \ls}.$  We denote by $$\r_{t,i}\defas 
\sqrt{\mathscr{B}_i(t)}v +\mathcalligra{m}_i(t) \text{ with } v \sim \mathcal{N}(0, I_\ell)\text{ and }\e\sim\mathcal{N}(0,\sigma^2I_\ls)\text{ independent},$$ 
and $\mathcalligra{m}_i(t)\defas \oint \mathscr{M}_i(t,z) \Dif z$ and $\mathscr{B}_i(t)\defas \oint z_i\mathscr{S}(t,z) \Dif z$ with $\mathscr{M}_i$ is the $i$th column of $\mathscr{M}$, associated with the $\mu_i$ mean.  
In addition, we denote by $\nabla_x$ the derivative with respect to only $\theta$ but not $\e.$ The maps in the hard labels setting are:
\begin{equation}\label{eq:integro_SM_hard}\begin{split}
    \mathscr{F}_\mathscr{S}(z, \mathscr{Z}(t, \cdot))
    &\defas 
    -2\gamma(t)\sum_{i=1}^\ls p_i\left(\left(z_i \mathscr{S}(t,z)-\frac{1}{2\pi\mathrm{i}}\oint \mathscr{S}(t,z)\dif z_i\right)\EE[\nabla_x^2 f_i(\r_{t,i},\e)]+\mathscr{M}_i(t,z)\otimes\EE[\nabla_x f_i(\r_{t,i},\e)]\right)
    \\&
    \qquad\quad+\frac{\gamma(t)^2}{d}\left(\sum_{i=1}^\ls p_i\ip{K_i+\mu_i\mu_i^\top,\cR_\ls}\EE[\nabla_x f_i(\r_{t,i},\e)^{\otimes2}]\right)
    \\
    \mathscr{F}_\mathscr{M}(z, \mathscr{Z}(t, \cdot))
    & \defas -\gamma(t)\sum_{i=1}^\ls p_i\left(\left(z_i\mathscr{M}(t,z)-\frac{1}{2\pi\mathrm{i}}\oint \mathscr{M}(t,z)\dif z_i\right)\EE[\nabla_x^2f_i(\r_{t,i},\e)]+\ip{\mu_i,\cR_\ell\mu}_{\R^d}\otimes \EE[\nabla_x f_i(\r_{t,i},\e)]\right).
\end{split}\end{equation}
Note that $\EE[\nabla_x f_i(\r_{t,i},\e)]$ and $\EE[\nabla_x^2 f_i(\r_{t,i},\e)]$ are also expressible in terms of $\mathscr{S},\mathscr{M}$. 

\paragraph{Soft labels}
The vector $\r_t$ is now of size $\lpp=\ell+\ls$. We now use $\nabla_x$ to denote the derivative with respect to the first $\ell$ components of $\r$ and we use $\nabla_\star$ for the derivative with respect to the second $\ls$ components. We write  $x_{t,i}\in \R^\ell$ and $x^\star\in \R^\ls,$ such that $$\r_{t,i} = x_{t,i}\oplus x^\star \defas 
\sqrt{\mathscr{B}_i(t)}v +\mathcalligra{m}_i(t) \text{ with } v \sim \mathcal{N}(0, I_{\ell+\ls}).$$ 
 The integro-differential equation now reads as follows: 
\begin{equation}\label{eq:integro_SM_soft}\begin{split}
    \mathscr{F}_\mathscr{S}(z, \mathscr{Z}(t, \cdot))
    &\defas 
    -2\gamma(t)\sum_{i=1}^\ls p_i\left(\left(z_i \mathscr{S}(t,z)-\frac{1}{2\pi\mathrm{i}}\oint \mathscr{S}(t,z)\dif z_i\right)H_{1,i}(t;\csZ) +\mathscr{M}_i(t,z)\otimes H_{2,i}(t;\csZ)\right)
    \\&
    \qquad+\frac{\gamma(t)^2}{d}\sum_{i=1}^\ls p_iI_i(t;\csZ)
    \\
    \mathscr{F}_\mathscr{M}(z, \mathscr{Z}(t, \cdot))
    & \defas -\gamma(t)\sum_{i=1}^\ls p_i\left(H_{1,i}(t;\csZ)^\top\left(z_i\mathscr{M}(t,z)-\frac{1}{2\pi\mathrm{i}}\oint \mathscr{M}(t,z)\dif z_i\right)+\ip{\mu_i,\cR_\ell\mu}\otimes H_{2,i}(t;\csZ)\right)
\end{split}\end{equation}
with
\begin{align}
 H_{1,i}(t;\csZ) \defas \left[\begin{array}{cc}
\EE[\nabla_x^2 f(\r_{t,i})]  & 0_{\ell\times \ls}\\
\EE[\nabla_\star \nabla_x f(\r_{t,i})]  & 0_{\ls\times \ls}
\end{array}\right], \quad H_{2,i}(t;\csZ) \defas \left[\begin{array}{cc}
\EE[\nabla_x f(\r_{t,i})] \\
0_{\ls\times \ell} 
\end{array}\right] 
\end{align}
and 
\begin{align}I_i(t;\csZ) \defas \left[\begin{array}{cc}
\ip{K_i+\mu_i\mu_i^\top,\cR_\ls}\EE[\nabla_x f(\r_{t,i})^{\otimes2}] & 0_{\ell\times \ls}\\
0_{\ls\times \ell} & 0_{\ls\times \ls}
\end{array}\right]. \end{align}

\paragraph{Spectrally-decomposed solution}

The functions $\mathscr{S}(t,z)$ and $\mathscr{M}(t,z)$ can be expressed as solutions to the integro-differential equations presented above, but we can also write them in terms of the previously introduced quantities $\csV_\rho(t),\csm_{\rho,j}(t)$, which are defined on the eigenspaces of the covariance matrices $K_i$.  For hard labels, this solution was given in \eqref{eq:V_m_rho_2} and the soft label version is 
\begin{equation}\label{eq:V_m_rho_2_soft}\begin{split}
    \frac{\dif \csV_\rho}{\dif t}
    &=\sum_{i=1}^\ls p_i\left(-2\gamma(t)\left( \csV_\rho\lambda_\rho^{(i)}H_{1,i}(t;\csZ)+\csm_{\rho,i}\otimes H_{2,i}(t;\csZ)\right)+\gamma(t)^2I_i(t;\csZ)\right)\\ 
    \frac{\dif \csm_{\rho,j}}{\dif t}&=-\gamma(t)\sum_{i=1}^\ls p_i\left(\lambda_\rho^{(i)}\csm_{\rho,j}H_{1,i}(t;\csZ)+\ip{\mu_i,u_\rho}\ip{u_\rho,\mu_j}H_{2,i}(t;\csZ)d\right),
\end{split}\end{equation}

\begin{lemma}[Equivalence of spectrally-decomposed solution]\label{lem:spe_decom_m_S}
    The system \eqref{eq:ODE_resolvent_2} has solution
    \begin{equation}
    \csS(t,z)=\frac1d\sum_{\rho=1}^d\frac{\csV_\rho(t)}{\prod_{i=1}^{\ls}(\lambda_{\rho}^{(i)}-z_i)},\qquad
    \csM_j(t,z)=\frac1d\sum_{\rho=1}^d\frac{\csm_{\rho,j}(t)}{\prod_{i=1}^{\ls}(\lambda_{\rho}^{(i)}-z_i)}
    \end{equation}
    where $\csV_{\rho},\csm_{\rho,j}$ solve the system \eqref{eq:V_m_rho_2} for the hard label case or \eqref{eq:V_m_rho_2_soft} for the soft label case.
\end{lemma}
\begin{proof}   We begin with the identity
\begin{equation}\label{eq:spectraldecomp_equividentity_S}\begin{split}
    \frac1d\sum_{\rho=1}^d\frac{\lambda_\rho^{(i)}\csV_\rho(t)}{\prod_{j=1}^{\ls}(\lambda_{\rho}^{(j)}-z_j)}
    &=\frac1d\sum_{\rho=1}^d\frac{\csV_\rho(t)}{\prod_{j\neq i}^{\ls}(\lambda_{\rho}^{(j)}-z_j)}
    +\frac{1}{d}\sum_{\rho=1}^d\frac{z_i\csV_\rho(t)}{\prod_{j=1}^{\ls}(\lambda_{\rho}^{(j)}-z_j)}\\
    &=-\frac{1}{2\pi\mathrm{i}}\oint\csS(t,z)\dif z_i+z_i\csS(t,z)
\end{split}\end{equation}
and by similar reasoning
\begin{equation}\label{eq:spectraldecomp_equividentity_M}
    \frac1d\sum_{\rho=1}^d\frac{\lambda_\rho^{(i)}\csm_{\rho,j}(t)}{\prod_{i=1}^{\ls}(\lambda_{\rho}^{(i)}-z_i)}
    =-\frac{1}{2\pi\mathrm{i}}\oint\csM_j(t,z)\dif z_i+z_i\csM_j(t,z).
\end{equation}
Using these identities and \eqref{eq:V_m_rho_2}, we see that $\csS(t,z),\csM(t,z)$ as defined in the lemma will solve \eqref{eq:ODE_resolvent_2}. 

\end{proof}
Note that the system \eqref{eq:V_m_rho_2} has a unique solution under Assumption \ref{ass:risk_fisher}, as the right side of \eqref{eq:V_m_rho_2} is Lipschitz. 
The lemma above tells us how to obtain a solution to \eqref{eq:ODE_resolvent_2} from the solution of \eqref{eq:V_m_rho_2}.  It is also possible to go in the reverse direction.  Given any solution to 
\eqref{eq:ODE_resolvent_2}, we see that, by computing the right hand side of \eqref{eq:spectraldecomp_equividentity_S} and \eqref{eq:spectraldecomp_equividentity_M} and performing partial fraction decomposition, one obtains a solution to the system \eqref{eq:V_m_rho_2}.  We do not attempt to verify if the system \eqref{eq:ODE_resolvent_2} has a unique solution, since this is not needed for our purposes.

\subsection{SGD is an approximate solution to the integro-differential equation}\label{sec:approx_sol}

The main task in proving Theorem \ref{thm:main_risk_m_v} is to show that SGD approximately solves \eqref{eq:ODE_resolvent_2}. 
We prove this under the more general setting of Assumption \ref{ass:risk_fisher_U}, a slightly weaker version of Assumption \ref{ass:risk_fisher} that still allows for locally unique solutions to \eqref{eq:V_m_rho_2}. In this version of the assumption, we do not require the functions $f_i$ to be differentiable everywhere, but only on some open set $\csU$ and we do not require the derivative functions to be Lipschitz, but merely $\eta$-PL. The trade-off is that, in this context, our theorem gives locally unique solutions on $\csU$ but not necessarily a unique global solution. If the derivative functions are Lipschitz ($\eta = 0$), then the set $\csU$ is the full space. 

\begin{assumption}[Risk and its derivatives]\label{ass:risk_fisher_U} 

Let $B_i\in \R^{\lpp \times \lpp}$ and $m_i\in \R^{\lpp }$, and $z\sim \mathcal{N}(0, I_{\lpp}),$ for $i\in [\ls].$ 
There exists an open set $\csU \subseteq \R^{\lpp\times \lpp}\times \R^{\lpp}, $ such that $B_i, m_i \in \csU.$
The functions   
$\E_z[\nabla f_i(\sqrt{B_i} z + m_i; \e)], $ $\E_z[\nabla f_i(\sqrt{B_i} z + m_i, \e)^{\otimes 2}]$ and $\E_z[\nabla^2 f_i(\sqrt{B_i} z + m_i;\e)], $ are $\eta$-PL with respect to $B_i, m_i$ for all $i\in[ \ls],$ 
with Lipschitz constants, $L_1, L_2, L_{22}>0$ respectively. 
\end{assumption}

To make the notion of ``approximate solution'' precise, we introduce the following definitions:

\begin{definition}[Stopping time $\tau_Q$] \label{def:stoping_time}For any $Q>0$
    \[
\tau_Q(\mathscr{Z})\defas\inf\left\{t\geq0:\mathscr{N}(t)>\lp Q \quad \textup{or} \quad \oint z\mathscr{S}(t,z)\Dif z\notin \csU \quad \textup{or} \quad \oint z \mathscr{M}(t,z)\Dif z\notin \csU \right\}
\]
where $\mathscr{N}(t)\defas\oint\Tr(\mathscr{Z}(t,z))\Dif z$. 
\end{definition}
  We remark that, in the case where $\mathscr{Z}=Z(W_{\ifl{td}},z)$, we have $\mathscr{N}(t)=\|W_{\ifl{td}}\|^2.$
  We will sometimes need an additional stopping time.  
\begin{definition}[Stopping time $\tau_{Q,\star}$] \label{def:stoping_time_star}For any $Q>0$
    \[
\tau_{Q,\star}(\mathscr{Z})\defas\inf\left\{t\geq0:\sup_{z\in \Omega^{\ls}}\|\mathscr{Z}(t,z)\|_{\textup{op}}> Q \quad \textup{or} \quad \oint z \mathscr{S}(t,z)\Dif z\notin \csU \quad \textup{or} \quad \oint z \mathscr{M}(t,z)\Dif z\notin \csU \right\}.
\]
\end{definition}
\begin{definition}[Omega-norm]
For any continuous function $A: \Omega^{\ls} \to (\mathbb C^{\lp}) ^{\otimes 2}$ 
\begin{equation}
\|A\|_{\Omega}\defas
\sup _{z \in \Omega^{\ls}}\|A(z)\|.
\end{equation}
\end{definition}
Note that this norm is slightly different from the supremum that appears in the definition of $\tau_{Q,\star},$ since that one involves operator norms.

\begin{definition}[$(\e, Q, T)$-approximate solution to the integro-differential equation]\label{def:approx_solution}
 For $Q, T, \varepsilon > 0$, we say that a continuous function $\mathscr{Z} \, : \, [0, \infty)  \times \mathbb{C}^\ls \to \mathbb R^{\lp\times \lp }$
 is an \textit{$(\e, Q, T)$-approximate solution} of \eqref{eq:ODE_resolvent_2} if 
\[
\sup_{0 \le t \le (\tau_Q \wedge T)} \big \| \mathscr{Z}(t, \cdot) -  \mathscr{Z}(0, \cdot) - \int_0^t \mathscr{F}(\cdot, \mathscr{Z}(s, \cdot) ) \, \dif s \big \|_{{\Omega}} \le \varepsilon
\]
with the initial condition  $\mathscr{Z}(0,\cdot) = W_0^\top \cR_\ls(\mathcal{K},\cdot) W_0$, and $W_0 = X_0 \oplus X^\star\oplus \mu$.  
\end{definition}

We are now ready to introduce the main result of this section:

\begin{proposition}[SGD is an approximate solution]\label{prop:approx_sol}
Fix $T, Q>0$ and $\delta \in (0,\frac{1}{2}).$ Then, $Z(W_{td},z)$ is an approximate solution w.o.p. for any $\lp = O(\log(d))$, that is 
     \begin{align}
       \sup_{0\le t\le (\htQ\wedge T)}\big \|{Z}(W_{td}, \cdot) -  {Z}(W_0, \cdot) - \int_0^t \mathscr{F}(\cdot, {Z}(W_{sd}, \cdot) ) \, \dif s \big \|_{{\Omega}} \le d^{-\delta} 
\quad \text{w.o.p.}  
   \end{align}
\end{proposition}

Proving this proposition is the main technical element of this paper. The proposition holds under a stopping time.  Once this proposition is proved, Theorem \ref{thm:main_risk_m_v} follows after showing that the stopping time can be removed under our assumptions (see Proposition \ref{prop:nonexplosiveness} in sub-section \ref{sec:supporting_lemmas}).

\subsection{Proof of Proposition \ref{prop:approx_sol} and supporting lemmas}\label{sec:proof_prop_approx_sol}

Before proving Proposition \ref{prop:approx_sol}, we give three pre-requisite lemmas.
We start with a bound on the norm $\|\cdot\|_{\Omega}$ in terms of the parameters. 
This is analogous to Lemma 5 in \cite{collinswoodfin2023hitting}. 
\begin{lemma}\label{lem:normequivalence}
Recall $\mathscr{N}(t) \defas   \oint\Tr (\mathscr{Z}(t, z) )\Dif z$.  Then the following hold, 
\[
C(\ls,\lp) \le
\frac{\| Z(W_{td}, \cdot ) \|_\Omega}{ \|W_{td}\|^2},
\frac{ \|\mathscr{Z}(t, \cdot )\|_{\Omega}}{\mathscr{N}(t)} \leq 1
\]
with $C(\ls,\lp) = \frac{1}{\sqrt{\lp}} \left(\frac{1}{\max\{2, 2\|\cK\|_{\ls}\}} \right)^\ls.$ 
Furthermore, we have
\[
\|\nabla Z(W_{td},z)\|_{\Omega}\leq\sqrt{\lp}\|W_{td}\|,\qquad \|\nabla^2 Z(W_{td},z)\|_{\Omega}\leq\lp.
\]
\end{lemma}
\begin{proof} We first note that \begin{align}\label{eq:sup_R_ls}
    \sup_{z\in \Omega} \|\cR_\ell(\cK, z)\|_{\op} \le \sup_{z\in \Omega} \prod_{i=1}^\ls\|(K_i - z_iI_d)^{-1}\|_{\op}\le 1.
\end{align}
The norm of SGD iterate (with $|\Omega|$ denoting the length of each contour) is
    \[
        \|W_{td}\|^2 = \oint_{\Omega^{\ls}} \Tr (Z(W_{td}, z)) \prod_{j=1}^{\ls} \frac{-1}{2\pi i} \dif z_j \leq \left|\frac{\Omega}{2\pi}\right|^{\ls}\sqrt{\lp} \|Z \|_\Omega.
    \]
    On the other hand, 
    \[
    \| Z(W_{td}, \cdot ) \|_\Omega
    = 
    \sup_{z \in \Omega} \|\ip{ W_{td}^{\otimes 2}, \cR_\ls}_{(\R^d)^{\otimes2}}\|
    \leq \|W_{td}\|^2 \sup_{z \in \Omega} \|\cR_\ell(\cK, z)\|_\op
    \leq \|W_{td}\|^2.
    \]   
    The proof for the deterministic quantities $\mathscr{N}, \mathscr{Z}$ is almost identical by definition and, therefore, omitted. 
    For the gradient and Hessian bounds, we have
\begin{align*}
    \nabla Z(W,z)&\asymp I_{\lp}\otimes\ip{W,\cR_{\ls}(z;K)}_{\R^d}+\ip{W,\cR_{\ls}(z;K)}_{\R^d}\otimes I_{\lp},\\
    \nabla^2 Z(W,z)&\asymp I_{\lp}^{\otimes2}\otimes\cR_{\ls}(z;K)+\cR_{\ls}(z;K)\otimes I_{\lp}^{\otimes2}.
\end{align*}
Taking norms and using the bound \eqref{eq:sup_R_ls}, the lemma is proved.
\end{proof}

\begin{lemma}\label{lem:R_dif} For $z, \bar{z} \in \mathbb{C}^\ls$ the following formula holds
    \begin{align}
\|
     \cR_\ls(\cK,z)-\cR_\ls(\cK,\bar{z})    \|_{\textup{op}}&\le \sum_{i=1}^\ls \|\cR_{i}(\cK,z)\|_{\textup{op}}|\bar{z}_i-z_i|\left \|\prod_{m=i}^{\ls}\cR(K_m,\bar{z}_m) \right \|_{\textup{op}}
    \end{align} 
    with $\cR_{j}(\cK, z)$ as defined in Eq. \eqref{eq:R_ls_def} for any $j\in [\ls].$
\end{lemma}
\begin{proof}
    First, we recall the following two identities 
    \begin{align}\label{eq:tel_R}
     \cR_\ls(\cK,z)-\cR_\ls(\cK,\bar{z})    &= \cR_{\ls-1}(\cK,z)(\cR(K_\ls,z_\ls)-\cR(K_\ls,\bar{z}_\ls))
     \\\nonumber &+ ( \cR_{\ls-1}(\cK,z)- \cR_{\ls-1}(\cK,\bar{z}))\cR(K_\ls,\bar{z}_\ls)
    \end{align}
    and by the first resolvent identity for all $i\in [\ls]$: 
    \begin{align}\label{eq:R_i_dif}
    \cR(K_i,{z}_i)  - \cR(K_i,\bar{z}_i)  =    \cR(K_i,{z}_i)(z_i - \bar{z}_i)\cR(K_i,\bar{z}_i).
    \end{align}
    Applying Eq. \eqref{eq:R_i_dif} recursively in  Eq. \eqref{eq:tel_R}, we get
   \begin{align}
     \cR_\ls(\cK,z)-\cR_\ls(\cK,\bar{z})   =  \sum_{i=1}^\ls \cR_{i}(\cK,z)(z_i-\bar{z}_i)\prod_{m=i}^{\ls}\cR(K_m,\bar{z}_m)  
    \end{align} and the lemma follows.
    \end{proof}

In addition to the bounds above, the proof of Proposition \ref{prop:approx_sol} will require us to extract bounds on the norm $\|\cdot\|_\Omega$ based on bounds that hold for each fixed $z\in\Omega$ with overwhelming probability.  Since we cannot take a union probability bound over an uncountable set, we introduce a net argument, following a similar method as in \cite{collinswoodfin2023hitting}. The main difference is that the input dimension of the functions we work with is not scalar, but in $\mathbb{C}^\ls.$ In particular, the resolvent involves the multiplication of $\ls$ resolvents. We start by building a $d^{-\delta}$ mesh of $\Omega$, which we denote as $\Omega_\delta$ such that $\Omega_\delta \subset \Omega.$ In particular, for all $i\in[\ls]$ and $z_i\in \Omega$, there exists a $\bar{z}_i\in \Omega_\delta $  such that  $|z_i-\bar{z}_i|< d^{-\delta },$ with $\Omega_\delta$ having a cardinality, $|\Omega_\delta| = {C(|\Omega|)}d^{\delta }.$ (We use the notation $|\cdot|$ in two ways here: For the continuous set $\Omega$ it denotes contour length, and for the discrete set $\Omega_\delta$ it denotes cardinality).  The proof of the following lemma is deferred to the appendix.
\begin{lemma}[Net argument]\label{lem:net}
   Fix $T, Q,\delta >0.$ Suppose $\Omega_\delta$ is a $d^{-\delta}$ mesh of $\Omega$ with $|\Omega_\delta| = Cd^{\delta}$ for some $C>0$, and that the function $Z(t,z) = Z(W_{td},z)$ satisfies
   \begin{align}
       \sup_{0\le t\le ({\tau}_Q\wedge T)}\big \|{Z}(t, \cdot) -  {Z}(0, \cdot) - \int_0^t \mathscr{F}(\cdot, {Z}(s, \cdot) ) \, \dif s \big \|_{{\Omega_\delta}} \le \varepsilon
   \end{align}
   with ${\tau}_{Q}({Z})$ as in Definition \ref{def:stoping_time}. 
   Then $Z$ is an
   approximate solution to the integro-differential equation, that is  
     \begin{align}
       \sup_{0\le t\le ({\tau}_Q\wedge T)}\big \|{Z}(t, \cdot) -  {Z}(0, \cdot) - \int_0^t \mathscr{F}(\cdot, {Z}(s, \cdot) ) \, \dif s \big \|_{{\Omega}} \le 
       d^{-\delta} ({\lp})^{\alpha+3} +\e
   \end{align}
   where, recalling the Lipschitz constants $L_1,L_2,L_{22}$ from Assumption \ref{ass:risk_fisher}, the constant above can be expressed as $C =  C(T, L_1,L_2, L_{22}, \bar{\gamma}, \|\mu\|_\ls, \|\cK\|_\ls, Q)>0,$ with $\|\mu\|_\ls \defas \max_{i\in[\ls]}\|\mu_i\|.$ 
\end{lemma}

\begin{remark}In the isotropic covariance case, as the net argument is not needed, the deterministic equivalence can be shown to hold for larger values of $\ls>d^c$ for some $c>0$ where the value of $c$ comes from the martingale bounds in Section \ref{sec:error_bds}.
\end{remark} 

Finally, we are ready to prove the main proposition.

\subsubsection{Proof of Proposition \ref{prop:approx_sol}}

\begin{proof}[Proof of Proposition \ref{prop:approx_sol}]
Denote for short $Z_{td} = Z(W_{td},z).$ Applying Eq. \eqref{eq:Doob_varphi} element-wise in $Z$, we obtain the following Doob decomposition for $Z$:
\begin{align}
Z_{td} = Z_{0} +\int_0^t\mathscr{F}(z, Z_{sd})\dif s+  \sum_{j=0}^{\lfloor td \rfloor-1} \mathcal E_j^{\text{all}}(Z)  + \xi_{td}(Z),     
\end{align}
with $\mathscr{F}$ as defined either in Eq. \eqref{eq:integro_SM_hard} or Eq. \eqref{eq:integro_SM_soft}. To show that $Z(W_{td},z)$ is an approximate solution, it amounts to bounding the above error terms. We thus have that, 
\begin{align}
&\sup_{0\le t\le (\htQ\wedge T)}\big \|Z_{td} -  Z_{0} - \int_0^t \mathscr{F}(\cdot, Z_{sd} ) \, \dif s \big \|_{{\Omega}}   
\\\nonumber & \le \sup_{0\le t\le (\htQ\wedge T)} \|\cM_{td}^{\text{Grad}}(Z)\|+\sup_{0\le t\le (\htQ\wedge T)} \|\cM_{td}^{\text{Hess}}(Z)\|
\\\nonumber & +\sup_{0\le t\le (\htQ\wedge T)} \|\sum_{j=0}^{\ifl{td} -1}\EE [\mathcal{E}_{j}^{\text{Hess}}(Z) \, | \, \mathcal{F}_{j} ]\|+ \sup_{0\le t\le (\htQ\wedge T)} \|\xi_{td}(Z)\|. 
\end{align}
The bounds on $\cM_{td}^{\text{Grad}},\cM_{td}^{\text{Hess}},\mathcal{E}_{j}^{\text{Hess}}$ are deferred to Section \ref{sec:error_bds}, where they can be found in Lemmas \ref{lem:martingale_error}, \ref{lem:Hess_error} respectively.  Next, fix a constant $\delta>0.$ Let $\Omega_\delta\subset \Omega$ such that for every $i\in [\ls]$ there exists $\bar{z}_i\in \Omega_\delta$ such that $|z_i-\bar{z}_i|\le d^{-\delta}$ with total cardinality  $|\Omega_\delta|^\ls =Cd^{\delta \ls}$ where $C>0$ depends on $\|\cK\|_\ls.$ 
By Lemmas \ref{lem:martingale_error}, \ref{lem:Hess_error} and an application of a union bound on the net, we have that, for any $\zeta>0$ for any $\ls = O(\log d)$, 
\begin{align}
\sup_{z\in \Omega_\delta^\ls}\sup_{0\le t\le (\htQ\wedge T)} \left(\|\cM_{td}^{\text{Grad}}(Z)\| +\|\cM_{td}^{\text{Hess}}(Z) \|+ \sum_{j=0}^{\ifl{td} -1}\|\EE [\mathcal{E}_{j}^{\text{Hess}}(Z) \, | \, \mathcal{F}_{j} ]\|\right)
\le d^{-\frac{1}{2}+\zeta} \quad \text{w.o.p.}
\end{align}
Next, we bound the deterministic error, 
\begin{align}\label{eq:det_error}
\sup_{0\le t\le (\htQ\wedge T)} \|\xi_{td}(Z)\|_{\Omega} \le  \frac{1}{d}\sup_{0\le t\le (\htQ\wedge T)} \|\mathscr{F}(z,Z(W_{td},z))\|_{\Omega}
\end{align}
where 
\begin{align}\label{eq:F_bounded}
\|\mathscr{F}(z,Z_{td})\|_{\Omega}\leq 2\|\mathscr{F}_{\csM}(z,Z_{td})\|_{\Omega}+\|\mathscr{F}_{\csS}(z,Z_{td})\|_{\Omega}.
\end{align}
We show the proof here for the hard labels. The proof for the soft labels is the same and therefore omitted. 
Let us start with $\mathscr{F}_{\csM}
$: 
\begin{align}
\|\mathscr{F}_{\csM}(z,Z_{td})\|_\Omega\le    \bar{\gamma}2|\Omega| \|M_{td}\|_\Omega \sup_{i\in[\ls]}\|\EE[\nabla^2f_i(r_{td})]\|+\bar{\gamma}\ls\|\mu\|_\ls^2\sup_{i\in[\ls]}\|\EE[\nabla f_i(r_{td})]\|
\end{align}
where we use \eqref{eq:sup_R_ls} and recall  $\|\mu\|_\ls = \max_{i\in [\ls]} \|\mu_i\|.$ Next, using Assumption \ref{ass:risk_fisher} , \eqref{eq:bound_E_nablaf} and Lemma \ref{lem:normequivalence} we obtain, 
\begin{align}
\|\mathscr{F}_{\csM}(z,Z_{td})\|_\Omega
&\le   C_{\csM}(\lp)^{\alpha+2}
\end{align}
with $C_{\csM} = C_{\csM}(\|\mu\|_\ls,\|\cK\|_\ls, \bar{\gamma}, Q, L_1, L_2).$
Similarly for $\mathscr{F}_{\csS},$  we obtain 
\begin{align}
 \|\mathscr{F}_{\csS}(z, Z_{td})\|_\Omega &\le    
 4\bar{\gamma}|\Omega|\|S_{td}\|_{\Omega} \sup_{i\in[\ls]} \|\EE [\nabla^2 f_i(r_{td})]\|
+ 2\bar{\gamma}\|M_{td}\|_{\Omega}\sup_{i\in[\ls]} \|\EE[\nabla f_i(r_{td})]\| \\\nonumber&+\bar{\gamma}^2(\|K\|_{\ls}+\|\mu\|^2_{\ls})\sup_{i\in[\ls]} \|\EE [\nabla f_i(r_{td})^{\otimes 2}]\|\le C_{\csS} (\lp)^{\alpha+2}
\end{align}
with $C_{\csS} = C_{\csS}(\|\mu\|_\ls,\|K\|_\ls, |\Omega|, \bar{\gamma}, Q, L_1, L_2, L_{22}).$ Therefore, plugging everything back into  Eq. \eqref{eq:det_error} and \eqref{eq:F_bounded}, there exist a constant $C>0$ such that 
\begin{align}
\sup_{0\le t\le (\htQ\wedge T)} \|\xi_{td}(Z)\|_{\Omega} \le  \frac{1}{d} C (\lp)^{\alpha+2}.
\end{align}

Combining all the errors, we deduce that, for some $C>0$ that doesn't depend on $d$,$n$ or $\lp$ and for any $\zeta>0$,
\begin{align}
       \sup_{0\le t\le (\htQ\wedge T)}\big \|{Z}_{td} -  {Z}_0 - \int_0^t \mathscr{F}(\cdot, {Z}_{sd}) \, \dif s \big \|_{{\Omega}_\delta} \le C\max(d^{-\frac{1}{2}+\zeta}, d^{-1}(\lp)^{\alpha+2})
\quad \text{w.o.p.}  
   \end{align}
Application of the net argument, Lemma \ref{lem:net}, finishes the proof after setting $\zeta  = 1-2\delta$ for any $\delta\in (0,\frac{1}{2}).$
\end{proof}

To complete the proof of Theorem \ref{thm:main_risk_m_v}, it remains only to verify the error bounds (Section \ref{sec:error_bds}) and to remove the stopping time that was in Proposition \ref{prop:approx_sol}.  In Section \ref{sec:supporting_lemmas} we prove non-explosiveness (see Proposition \ref{prop:nonexplosiveness}), which allows removal of the stopping time, and we also verify stability of solutions to our integro-differential equation.

\section{GMM Homogenized SGD (GMM-HSGD)\label{sec:GMM_homogenized}}
Our result can also be presented as a stochastic differential equation (SDE) that is close to SGD over the set of functions in Definition \ref{ass:statistic} with overwhelming probability. To be precise, one can show that $Z(\mathscr{W}_t, z)$ is an approximate solution with $\mathscr{W}_t = \hat{\mathscr{X}}_t\oplus \mu$, and $\hat{\mathscr{X}}_t = {\mathscr{X}}_t\oplus X^\star$ (soft labels) and $\hat{\mathscr{X}}_t = \mathscr{X}_t$ (hard labels). The process $\mathscr{X}$, is defined by the following SDE with $\mathscr{X}_0 = X_0$
We present here the SDE in the context of GMM. We omit the proof here, as the analysis is very similar to Section \ref{prop:approx_sol}. Having the SDE introduces an additional martingale error term, which, over the statistics, can be shown to be small (see also \cite{collinswoodfin2023hitting}, where it is done for Gaussian data from a single class).

We note that our analysis is in the setting where the learning rate is $\gamma/d$.  In this setting, one obtains an SDE in which both the drift and diffusion terms contribute and the path of the SDE concentrates in the high-dimensional limit.  One could also consider other scalings of the learning rate with $d$ which, in some cases, yield an SDE that does not concentrate (see \cite{arous2022high} for an example).  However, we do not analyze these other scalings.

\paragraph{GMM-HSGD - Hard label}

\begin{equation}\label{eq:HSGD_hard}
    \dif\csX_t=-\gamma(t)\sum_{i=1}^\ls p_i\nabla\cL_i(\csX_t)\dif t+\gamma(t)\Big\langle\sqrt{
    \frac{1}{d}\textstyle\sum_{i=1}^\ls p_i(K_i+\mu_i\mu_i^\top)\otimes\EE[\nabla f_{i}(\rho_{t,i}; \e_t)^{\otimes2}]},\dif B_t\Big\rangle_{\R^{\bar{\ell}\times d}}
\end{equation}
where the class-indexed loss functions $\{\cL_i\}_{i\leq\ell}$ are $\cL_i(X)\defas \EE[f_i(a^\top X)|a\in\text{class }i]$ and $\rho_{t,I}\defas  \hat{\csX}^\top a_I$.  The process $B_t$ is a standard Brownian motion in $\R^{\lpp\times d }$. We can further compute the gradient term (using Stein's lemma) as
\begin{equation}
    \nabla\cL_i(\csX_t)
    =K_i\csX_t\EE[\nabla^2 f_i(\rho_{t,i}; \e_t)]+\mu_i\EE[\nabla f_i(\rho_{t,i}; \e_t)].
\end{equation}

\paragraph{GMM-HSGD - Soft label}
\begin{equation}\label{eq:HSGD_soft}
    \dif\csX_t=-\gamma(t)\sum_{i=1}^\ls p_i\nabla\cL_i(\csX_t)\dif t+\gamma(t)\Big\langle\sqrt{
    \frac{1}{d}\textstyle\sum_{i=1}^\ls p_i(K_i+\mu_i\mu_i^\top)\otimes I_{i,t}},\dif B_t\Big\rangle_{\R^{\bar{\ell}\times d}}
\end{equation}
with $\cL_i(\csX)=\EE [f(a^\top \hat{\csX})\mid a\in \text{class $i$} ]$ and
\begin{equation}
    \nabla\cL_i(\csX_t)
    = K_i\csX_t H_{1,i,t}+\mu_i H_{2,i,t}.
\end{equation} such that 
\begin{align}
 H_{1,i,t} \defas \left[\begin{array}{cc}
\EE[\nabla_x^2 f(\rho_{t,i}; \e_t)]  & 0_{\ell\times \ls}\\
\EE[\nabla_\star \nabla_x f(\rho_{t,i};\e_t)]  & 0_{\ls\times \ls}
\end{array}\right], \quad H_{2,i,t} \defas \left[\begin{array}{cc}
\EE[\nabla_x f(\rho_{t,i}; \e_t)] \\
0_{\ls\times \ell} 
\end{array}\right] 
\end{align}
and 
\begin{align}I_{i,t} \defas \left[\begin{array}{cc}
\ip{K_i+\mu_i\mu_i^\top,\cR_\ls}_{\R^{d\times d}}\EE[\nabla_x f(\rho_{t,i};\e_t)^{\otimes2}] & 0_{\ell\times \ls}\\
0_{\ls\times \ell} & 0_{\ls\times \ls}
\end{array}\right]. \end{align}


\appendix

\section*{Acknowledgments}
We thank Courtney Paquette and Elliot Paquette for fruitful discussions.  

\section*{Funding}
The work of Elizabeth Collins-Woodfin was supported by Fonds de recherche du Qu\'ebec – Nature et technologies
(FRQNT) postdoctoral training scholarship (DOI https://doi.org/10.69777/344253), Centre de recherches math\'ematiques (CRM) Applied
math postdoctoral fellowship, and Institut Mittag-Leffler (IML) Junior Fellowship (Swedish Research Council, grant no. 2021-06594). The work
of Inbar Seroussi is supported by the Israel Science Foundation grant no. 777/25, the NSF-BSF grant no. 0603624011, and by the Alon fellowship.
  
\bibliographystyle{abbrv}
\bibliography{refs}

@article{balasubramanian2023high,
  title={High-dimensional scaling limits and fluctuations of online least-squares {SGD} with smooth covariance},
  author={Balasubramanian, Krishnakumar and Ghosal, Promit and He, Ye},
  journal={arXiv preprint arXiv:2304.00707},
  year={2023}
}

@article {chandrasekher2021sharp,
    AUTHOR = {Chandrasekher, Kabir Aladin and Pananjady, Ashwin and
              Thrampoulidis, Christos},
     TITLE = {Sharp global convergence guarantees for iterative nonconvex
              optimization with random data},
   JOURNAL = {Ann. Statist.},
  FJOURNAL = {The Annals of Statistics},
    VOLUME = {51},
      YEAR = {2023},
    NUMBER = {1},
     PAGES = {179--210},
      ISSN = {0090-5364,2168-8966},
   MRCLASS = {62J02 (90C06 90C26)},
  MRNUMBER = {4564853},
       DOI = {10.1214/22-aos2246},
       URL = {https://doi.org/10.1214/22-aos2246},
}

@article{loureiro2021learning,
  title={Learning curves of generic features maps for realistic datasets with a teacher-student model},
  author={Loureiro, Bruno and Gerbelot, Cedric and Cui, Hugo and Goldt, Sebastian and Krzakala, Florent and Mezard, Marc and Zdeborov{\'a}, Lenka},
  journal={Advances in Neural Information Processing Systems},
  volume={34},
  pages={18137--18151},
  year={2021}
}

@article{celentano2021highdimensional,
  title={The high-dimensional asymptotics of first order methods with random data},
  author={Celentano, Michael and Cheng, Chen and Montanari, Andrea},
  journal={arXiv preprint arXiv:2112.07572},
  year={2021}
}

@inproceedings{goldt2022gaussian,
  title={The gaussian equivalence of generative models for learning with shallow neural networks},
  author={Goldt, Sebastian and Loureiro, Bruno and Reeves, Galen and Krzakala, Florent and M{\'e}zard, Marc and Zdeborov{\'a}, Lenka},
  booktitle={Mathematical and Scientific Machine Learning},
  pages={426--471},
  year={2022},
  organization={PMLR}
}

@article{wang2019solvable,
  title={A solvable high-dimensional model of {GAN}},
  author={Wang, Chuang and Hu, Hong and Lu, Yue},
  journal={Advances in Neural Information Processing Systems},
  volume={32},
  year={2019}
}

@article{saad1995exact,
  title={Exact solution for on-line learning in multilayer neural networks},
  author={Saad, David and Solla, Sara A},
  journal={Physical Review Letters},
  volume={74},
  number={21},
  pages={4337},
  year={1995},
  publisher={APS}
}

@article{biehl1994line,
  title={On-line learning with a perceptron},
  author={Biehl, Michael and Riegler, Peter},
  journal={Europhysics Letters},
  volume={28},
  number={7},
  pages={525},
  year={1994},
  publisher={IOP Publishing}
}

@article{biehl1995learning,
  title={Learning by on-line gradient descent},
  author={Biehl, Michael and Schwarze, Holm},
  journal={Journal of Physics A: Mathematical and general},
  volume={28},
  number={3},
  pages={643},
  year={1995},
  publisher={IOP Publishing}
}

@article{goldt2019dynamics,
  title={Dynamics of stochastic gradient descent for two-layer neural networks in the teacher-student setup},
  author={Goldt, Sebastian and Advani, Madhu and Saxe, Andrew M and Krzakala, Florent and Zdeborov{\'a}, Lenka},
  journal={Advances in neural information processing systems},
  volume={32},
  year={2019}
}

@article{PPAP01,
  title={Homogenization of {SGD} in high-dimensions: Exact dynamics and generalization properties},
  author={Paquette, Courtney and Paquette, Elliot and Adlam, Ben and Pennington, Jeffrey},
  journal={Mathematical Programming},
  pages={1--90},
  year={2024},
  publisher={Springer}
}

@article{LeeChengPaquettePaquette,
  title={Trajectory of mini-batch momentum: batch size saturation and convergence in high dimensions},
  author={Lee, Kiwon and Cheng, Andrew and Paquette, Elliot and Paquette, Courtney},
  journal={Advances in Neural Information Processing Systems},
  volume={35},
  pages={36944--36957},
  year={2022}
}

@ARTICLE{CollinsWoodfinPaquette01,
       author = {{Collins-Woodfin}, Elizabeth and {Paquette}, Elliot},
        title = "{High-dimensional limit of one-pass SGD on least squares}",
      journal = {Electronic Communications in Probability},
         year = 2024,
        volume = {29},
        pages = {1-15},
          doi = {10.1214/23-ECP571},
archivePrefix = {arXiv},
       eprint = {2304.06847},
 primaryClass = {math.PR},
       adsurl = {https://ui.adsabs.harvard.edu/abs/2023arXiv230406847C}
}

@article{arous2021online,
  title={Online stochastic gradient descent on non-convex losses from high-dimensional inference},
  author={Ben Arous, Gerard and Gheissari, Reza and Jagannath, Aukosh},
  journal={The Journal of Machine Learning Research},
  volume={22},
  number={1},
  pages={4788--4838},
  year={2021},
  publisher={JMLRORG}
}

@article{arous2022high,
  title={High-dimensional limit theorems for {SGD}: Effective dynamics and critical scaling},
  author={Ben Arous, Gerard and Gheissari, Reza and Jagannath, Aukosh},
  journal={Advances in Neural Information Processing Systems},
  volume={35},
  pages={25349--25362},
  year={2022}
}

@article{saad1995dynamics,
  title={Dynamics of on-line gradient descent learning for multilayer neural networks},
  author={Saad, David and Solla, Sara},
  journal={Advances in neural information processing systems},
  volume={8},
  year={1995}
}

@article{Goldt,
  title={Modeling the influence of data structure on learning in neural networks: The hidden manifold model},
  author={Goldt, Sebastian and M{\'e}zard, Marc and Krzakala, Florent and Zdeborov{\'a}, Lenka},
  journal={Physical Review X},
  volume={10},
  number={4},
  pages={041044},
  year={2020},
  publisher={APS}
}

@article{Yoshida,
  title={Data-dependence of plateau phenomenon in learning with neural network---Statistical mechanical analysis},
  author={Yoshida, Yuki and Okada, Masato},
  journal={Advances in Neural Information Processing Systems},
  volume={32},
  year={2019}
}

@article{mignacco2020dynamical,
  title={Dynamical mean-field theory for stochastic gradient descent in gaussian mixture classification},
  author={Mignacco, Francesca and Krzakala, Florent and Urbani, Pierfrancesco and Zdeborov{\'a}, Lenka},
  journal={Advances in Neural Information Processing Systems},
  volume={33},
  pages={9540--9550},
  year={2020}
}

@article{dobriban2018high,
  title={High-dimensional asymptotics of prediction: Ridge regression and classification},
  author={Dobriban, Edgar and Wager, Stefan},
  journal={The Annals of Statistics},
  volume={46},
  number={1},
  pages={247--279},
  year={2018},
  publisher={JSTOR}
}

@article{azizyan2013minimax,
  title={Minimax theory for high-dimensional gaussian mixtures with sparse mean separation},
  author={Azizyan, Martin and Singh, Aarti and Wasserman, Larry},
  journal={Advances in Neural Information Processing Systems},
  volume={26},
  year={2013}
}

@inproceedings{lesieur2016phase,
  title={Phase transitions and optimal algorithms in high-dimensional gaussian mixture clustering},
  author={Lesieur, Thibault and De Bacco, Caterina and Banks, Jess and Krzakala, Florent and Moore, Cris and Zdeborov{\'a}, Lenka},
  booktitle={2016 54th Annual Allerton Conference on Communication, Control, and Computing (Allerton)},
  pages={601--608},
  year={2016},
  organization={IEEE}
}

@article{arous2025local,
  title={Local geometry of high-dimensional mixture models: Effective spectral theory and dynamical transitions},
  author={Ben Arous, Gerard and Gheissari, Reza and Huang, Jiaoyang and Jagannath, Aukosh},
  journal={arXiv preprint arXiv:2502.15655},
  year={2025}
}

@inproceedings{refinetti2021classifying,
  title={Classifying high-dimensional gaussian mixtures: Where kernel methods fail and neural networks succeed},
  author={Refinetti, Maria and Goldt, Sebastian and Krzakala, Florent and Zdeborov{\'a}, Lenka},
  booktitle={International Conference on Machine Learning},
  pages={8936--8947},
  year={2021},
  organization={PMLR}
}

@article{dandi2024universality,
  title={Universality laws for gaussian mixtures in generalized linear models},
  author={Dandi, Yatin and Stephan, Ludovic and Krzakala, Florent and Loureiro, Bruno and Zdeborov{\'a}, Lenka},
  journal={Advances in Neural Information Processing Systems},
  volume={36},
  year={2024}
}

@inproceedings{seddik2020random,
  title={Random matrix theory proves that deep learning representations of GAN-data behave as gaussian mixtures},
  author={Seddik, Mohamed El Amine and Louart, Cosme and Tamaazousti, Mohamed and Couillet, Romain},
  booktitle={International Conference on Machine Learning},
  pages={8573--8582},
  year={2020},
  organization={PMLR}
}

@inproceedings{chen2024achieving,
  title={Achieving optimal clustering in Gaussian mixture models with anisotropic covariance structures},
  author={Chen, Xin and Zhang, Anderson Ye},
  booktitle={The Thirty-eighth Annual Conference on Neural Information Processing Systems},
  year={2024}
}

@article{huang2025minimax,
  title={Minimax-Optimal Covariance Projected Spectral Clustering for High-Dimensional Nonspherical Mixtures},
  author={Huang, Chengzhu and Gu, Yuqi},
  journal={arXiv preprint arXiv:2502.02580},
  year={2025}
}

@article{mai2019high,
  title={High Dimensional Classification via Regularized and Unregularized Empirical Risk Minimization: Precise Error and Optimal Loss},
  author={Mai, Xiaoyi and Liao, Zhenyu},
  journal={stat},
  volume={1050},
  pages={25},
  year={2020}
}

@article{loffler2021optimality,
  title={Optimality of spectral clustering in the Gaussian mixture model},
  author={L{\"o}ffler, Matthias and Zhang, Anderson Y and Zhou, Harrison H},
  journal={The Annals of Statistics},
  volume={49},
  number={5},
  pages={2506--2530},
  year={2021},
  publisher={Institute of Mathematical Statistics}
}

@article{thrampoulidis2020theoretical,
  title={Theoretical insights into multiclass classification: A high-dimensional asymptotic view},
  author={Thrampoulidis, Christos and Oymak, Samet and Soltanolkotabi, Mahdi},
  journal={Advances in Neural Information Processing Systems},
  volume={33},
  pages={8907--8920},
  year={2020}
}

@article{loureiro2021learningGMM,
  title={Learning gaussian mixtures with generalized linear models: Precise asymptotics in high-dimensions},
  author={Loureiro, Bruno and Sicuro, Gabriele and Gerbelot, C{\'e}dric and Pacco, Alessandro and Krzakala, Florent and Zdeborov{\'a}, Lenka},
  journal={Advances in Neural Information Processing Systems},
  volume={34},
  pages={10144--10157},
  year={2021}
}

@inproceedings{ji2021fast,
  title={Fast margin maximization via dual acceleration},
  author={Ji, Ziwei and Srebro, Nathan and Telgarsky, Matus},
  booktitle={International Conference on Machine Learning},
  pages={4860--4869},
  year={2021},
  organization={PMLR}
}

@article{soudry2018implicit,
  title={The implicit bias of gradient descent on separable data},
  author={Soudry, Daniel and Hoffer, Elad and Nacson, Mor Shpigel and Gunasekar, Suriya and Srebro, Nathan},
  journal={Journal of Machine Learning Research},
  volume={19},
  number={70},
  pages={1--57},
  year={2018}
}

@article{bartlett2017spectrally,
  title={Spectrally-normalized margin bounds for neural networks},
  author={Bartlett, Peter L and Foster, Dylan J and Telgarsky, Matus J},
  journal={Advances in neural information processing systems},
  volume={30},
  year={2017}
}

@inproceedings{ji2021characterizing,
  title={Characterizing the implicit bias via a primal-dual analysis},
  author={Ji, Ziwei and Telgarsky, Matus},
  booktitle={Algorithmic Learning Theory},
  pages={772--804},
  year={2021},
  organization={PMLR}
}

@inproceedings{nacson2019convergence,
  title={Convergence of gradient descent on separable data},
  author={Nacson, Mor Shpigel and Lee, Jason and Gunasekar, Suriya and Savarese, Pedro Henrique Pamplona and Srebro, Nathan and Soudry, Daniel},
  booktitle={The 22nd International Conference on Artificial Intelligence and Statistics},
  pages={3420--3428},
  year={2019},
  organization={PMLR}
}

@article{Bruno,
  title={Escaping mediocrity: how two-layer networks learn hard single-index models with {SGD}},
  author={Arnaboldi, Luca and Krzakala, Florent and Loureiro, Bruno and Stephan, Ludovic},
  journal={CoRR},
  year={2023}
}

@inproceedings{arnaboldi2023highdimensional,
  title={From high-dimensional \& mean-field dynamics to dimensionless {ODE}s: A unifying approach to {SGD} in two-layers networks},
  author={Arnaboldi, Luca and Stephan, Ludovic and Krzakala, Florent and Loureiro, Bruno},
  booktitle={The Thirty Sixth Annual Conference on Learning Theory},
  pages={1199--1227},
  year={2023},
  organization={PMLR}
}

@inproceedings{bordelon2022learning,
  title={Learning Curves for {SGD} on Structured Features},
  author={Bordelon, Blake and Pehlevan, Cengiz},
  booktitle={International Conference on Learning Representations},
year={2022}
}

@article{collinswoodfin2023hitting,
  title={Hitting the high-dimensional notes: An {ODE} for {SGD} learning dynamics on {GLM}s and multi-index models},
  author={Collins-Woodfin, Elizabeth and Paquette, Courtney and Paquette, Elliot and Seroussi, Inbar},
  journal={Information and Inference: A Journal of the IMA},
  volume={13},
  number={4},
  pages={iaae028},
  year={2024},
  publisher={Oxford University Press}
}

@article{damian2023smoothing,
  title={Smoothing the landscape boosts the signal for sgd: Optimal sample complexity for learning single index models},
  author={Damian, Alex and Nichani, Eshaan and Ge, Rong and Lee, Jason D},
  journal={Advances in Neural Information Processing Systems},
  volume={36},
  pages={752--784},
  year={2023}
}

@inproceedings{dandi2024benefits,
  title={The Benefits of Reusing Batches for Gradient Descent in Two-Layer Networks: Breaking the Curse of Information and Leap Exponents},
  author={Dandi, Yatin and Troiani, Emanuele and Arnaboldi, Luca and Pesce, Luca and Zdeborova, Lenka and Krzakala, Florent},
  booktitle={International Conference on Machine Learning},
  pages={9991--10016},
  year={2024},
  organization={PMLR}
}

@article{gerbelot2022rigorous,
  title={Rigorous dynamical mean-field theory for stochastic gradient descent methods},
  author={Gerbelot, Cedric and Troiani, Emanuele and Mignacco, Francesca and Krzakala, Florent and Zdeborova, Lenka},
  journal={SIAM Journal on Mathematics of Data Science},
  volume={6},
  number={2},
  pages={400--427},
  year={2024},
  publisher={SIAM}
}

@article{papyan2020prevalence,
  title={Prevalence of neural collapse during the terminal phase of deep learning training},
  author={Papyan, Vardan and Han, XY and Donoho, David L},
  journal={Proceedings of the National Academy of Sciences},
  volume={117},
  number={40},
  pages={24652--24663},
  year={2020},
  publisher={National Academy of Sciences}
}

@article{cai2019high,
  author  = {Cai, T. Tony and Zhang, Linjun},
  title   = {High Dimensional Linear Discriminant Analysis:
             Optimality, Adaptive Algorithm and Missing Data},
  journal = {Journal of the Royal Statistical Society:
             Series B (Statistical Methodology)},
  volume  = {81},
  number  = {4},
  pages   = {675--705},
  year    = {2019},
  doi     = {10.1111/rssb.12326}
}

@inproceedings{li2017minimax,
  title     = {Minimax Gaussian Classification \& Clustering},
  author    = {Li, Tianyang and Yi, Xinyang and Caramanis, Constantine and Ravikumar, Pradeep},
  booktitle = {Proceedings of the 20th International Conference on Artificial Intelligence and Statistics},
  series    = {Proceedings of Machine Learning Research},
  volume    = {54},
  pages     = {1--9},
  year      = {2017},
  editor    = {Singh, Aarti and Zhu, Jerry},
  publisher = {PMLR}
}

@article{minsker2025classification,
  title   = {Classification in the High Dimensional Anisotropic Mixture Framework:
             A New Take on Robust Interpolation},
  author  = {Minsker, Stanislav and Ndaoud, Mohamed and Shen, Yiqiu},
  journal = {Journal of Machine Learning Research},
  volume  = {26},
  number  = {153},
  pages   = {1--39},
  year    = {2025}
}


\appendix
\section{Details of main theorem proof: SGD as an approximate solution}
In this appendix, we provide some of the technical details of our theorem proof.

\subsection{Proof of Lemma 
\ref{lem:net}}
\begin{proof}
We provide the proof for the hard label setup; the proof for the soft label is very similar and therefore omitted.  For a fixed $t\le {\tau}_Q,$ $z\in \Omega$ and $\bar{z} \in \Omega_\delta$ such that $|z_i-\bar{z}_i| < d^{-\delta}$ for all $i\in [\ls]$  
 \begin{align}
     &\big \|{Z}(t, z) -  {Z}(0, z) - \int_0^t \mathscr{F}(z, {Z}(s, \cdot) ) \, \dif s \big \|  \\\nonumber&\le \|Z(t,z) - Z(t,\bar{z})\| + \|Z(0,z) - Z(0,\bar{z})\|
     \\\nonumber & 
     +\int_0^t \|\mathscr{F}_{\mathscr{S}}(z, {Z}(s, \cdot) )   - \mathscr{F}_{\mathscr{S}}(\bar{z}, {Z}(s, \cdot) )\|\, \dif s
        +2\int_0^t \|\mathscr{F}_{\mathscr{M}}(z, {Z}(s, \cdot) )   - \mathscr{F}_{\mathscr{M}}(\bar{z}, {Z}(s, \cdot) )\|\, \dif s
     \\\nonumber & +\|Z(t, \bar{z}) - Z(0, \bar{z}) - \int_0^t\mathscr{F}(\bar{z}, {Z}(s, \cdot) ) \, \dif s \|.
   \end{align}
Using Lemma \ref{lem:normequivalence}, Lemma \ref{lem:R_dif}, and the fact that $\sup_{z\in \Omega} \|\cR_\ls(\cK, z)\|_{\op}\le 1$, we see that, for any $t<\tau_Q$,  
\begin{align}\label{eq:dif_z_omega_net}
  \|Z(t,z) - Z(t,\bar{z})\|&\le \|\cR_\ls(\cK,z)-\cR_\ls(\cK,\bar{z})\|_\op\|W_{td}\|^2  
  \\\nonumber 
  & \le \sum_{i=1}^\ls |z_i-\bar{z}_i|
  \|W_{td}\|^2   \le d^{-\delta}(\lp)^2 Q
\end{align}
and
\begin{align}\label{eq:dif_zz_omega_net}
  \|z_iZ(t,z) - \bar{z}_iZ(t,\bar{z})\|& \le 
  |z_i|  \|Z(t,z) - Z(t,\bar{z})\| + |z_i-\bar{z_i}|\|Z(t,\bar{z})\|  \\\nonumber 
  & \le  |\Omega|d^{-\delta}(\lp)^2 Q +  d^{-\delta}{{\lp}} Q.
\end{align}

We will also need bounds on the 
the moments of gradients of the loss $\EE[\nabla f_i(r_{td})], \EE[\nabla^2 f_i(r_{td})], \EE[\nabla f_i(r_{td})^{\otimes 2}]$. These functions can be written as functions of $B_i(t) = \oint z_i S(W_{td}, z) \Dif z = \ip{X_{td}^{\otimes 2}, K_i}_{(\R^d)^{\otimes 2}}$ and $m_i(t) = \oint M_i(W_{td}, z) \Dif z = \ip{X_{td}, \mu_i}_{\R^d}$ which are independent of a particular $z.$ Furthermore, by Assumption \ref{ass:risk_fisher},  these functions are $\alpha-$PL with respect to $B_i$ and $m_i.$ Therefore, we have that 
\begin{align}\label{eq:bound_E_nablaf}
\|\EE[\nabla f_i(r_{td})]\| \le L_{1} (\|\cm_i\| +\|B_i\|) (\|\cm_i\|^{\a} +\|B_i\|^{\a}+1)\le 6L_1 Q^{\a +1} ({\lp})^{\a +1} \max\{1,\|\mathcal{K}\|_{\ls}\}
\end{align}
where we use that $\|B_i(t)\|
\le \|\cK\|_\ls \|W_{td}\|^2\le \|\cK\|_\ls \lp Q
,$ and $\|\cm_i(t)\|\le \|W_{td}\|^2\le \lp Q$
.
A similar bound is obtained for $\EE[\nabla^2 f_i(r_{td})], \EE[\nabla f_i(r_{td})^{\otimes 2}]$ with the corresponding Lipschitz constants. 

We can now bound the difference in the map for different $z$'s, 
\begin{align}
\|\mathscr{F}_{\mathscr{S}}(z, {Z}(s, \cdot) )   - \mathscr{F}_{\mathscr{S}}(\bar{z}, {Z}(s, \cdot) )\|\,&\le 2\gamma(s)     \|z_iS(t,z) - \bar{z}_iS(t,\bar{z})\|\sum_{i=1}^{\ls} p_i \|\EE[\nabla^2 f_i(r_{td,i})]\|
\\\nonumber&
+\frac{2\gamma(s)}{2\pi}     \|\oint \mathscr{S}(t,z)\dif z_i -\oint \mathscr{S}(t,\bar{z})\dif \bar{z}_i \|\sum_{i=1}^{\ls} p_i \|\EE[\nabla^2 f_i(r_{td,i})]\|
\\\nonumber&
+2\gamma(s) 
  \|M(t,z) - M(t,\bar{z})\| \sum_{i=1}^{\ls} p_i \|\EE[\nabla f_i(r_{td,i})]\|
   \\\nonumber& 
    +{\gamma(s)^2}\lp d^{-\delta  } (\|\cK\|_{\ls}+\|\mu\|_\ls^2)\sum_{i=1}^\ls p_i\|\EE[\nabla f_i(r_{td,i})^{\otimes2}]\|
\end{align}
where the last term is using Lemma \ref{lem:R_dif}. We bound $\|M(t,z) - M(t,\bar{z})\|$ and $\|\oint \mathscr{S}(t,z)\dif z_i -\oint \mathscr{S}(t,\bar{z})\dif \bar{z}_i\|$ using \eqref{eq:dif_z_omega_net}.  Likewise, we bound $\|z_iS(t,z) - \bar{z}_iS(t,\bar{z})\|$ using \eqref{eq:dif_zz_omega_net}, and $\|\EE[\nabla f_i(r_{td})]\|$ using \eqref{eq:bound_E_nablaf} (similarly for $\|\EE[\nabla^2 f_i(r_{td})]\|,\|\EE[\nabla f_i(r_{td})^{\otimes2}]\|$).
Combining all these facts and that $\gamma_k$ are uniformly bounded by some $\bar{\gamma}<\infty$.
\begin{align}
\|\mathscr{F}_{\mathscr{S}}(z, {Z}(s, \cdot) )   - \mathscr{F}_{\mathscr{S}}(\bar{z}, {Z}(s, \cdot) )\|\,&\le 2\bar{\gamma}    \left(
|\Omega|{(\lp)}^2d^{-\delta } Q +  (1+|\Omega|)d^{-\delta}{{\lp}} Q
\right)\sum_{i=1}^{\ls} p_i \|\EE[\nabla^2 f_i(r_{td,i})]\|
\\\nonumber&
+2\bar{\gamma} 
{(\lp)}^2 d^{-\delta } Q \sum_{i=1}^{\ls} p_i \|\EE[\nabla f_i(r_{td,i})]\|
   \\\nonumber& 
    +{\bar{\gamma}^2}\lp d^{-\delta  } (\|\cK\|_{\ls}+\|\mu\|_\ls^2)\sum_{i=1}^\ls p_i\|\EE[\nabla f_i(r_{td,i})^{\otimes2}]\|
     \\\nonumber & 
\leq  C_{\mathscr{S}}  d^{-\delta} ({\lp})^{\alpha+3}
\end{align}
with $C_{\mathscr{S}} = C_{\mathscr{S}}(Q, \alpha, \bar{\gamma}, \|\cK\|_\ls,\|\mu\|_\ls, L_1, L_{22}, L_2)>0.$ 
Similarly, for $\mathscr{F}_{\mathscr{M}}$ 
\begin{align}
\|\mathscr{F}_{\mathscr{M}}(z, {Z}(s, \cdot) )   - \mathscr{F}_{\mathscr{M}}(\bar{z}, {Z}(s, \cdot) )\|&\le \gamma(s) \left(
|\Omega|{(\lp)}^2d^{-\delta } Q +  (1+|\Omega|)d^{-\delta}{{\lp}} Q
\right)\sum_{i=1}^{\ls} p_i \|\EE[\nabla^2 f_i(r_{td,i})]\|
\\\nonumber & + \gamma(s) \sum_{i=1}^{\ls} p_i\|\EE[\nabla f_i(r_{td,i})]\|\lp d^{-\delta } \|\mu_i\|\|\mu\|_{\ls}
\\\nonumber & 
\leq  C_{\mathscr{M}}  d^{-\delta} ({\lp})^{\alpha+3} 
\end{align}
with $C_{\mathscr{M}} = C_{\mathscr{M}}(Q, \alpha, \bar{\gamma}, \|\cK\|_\ls,\|\mu\|_\ls, L_1, L_{22}, L_2)>0.$
Combining the above, we obtain the result of the Lemma, 
 \begin{align}
     &\sup_{0\le t\le (T\wedge {\tau}_Q )} \big \|{Z}(t, z) -  {Z}(0, z) - \int_0^t \mathscr{F}(z, {Z}(s, \cdot) ) \, \dif s \big \|_{\Omega}  \\\nonumber&\le d^{-\delta \ls}
     \lp Q
         +\int_0^t \|\mathscr{F}_{\mathscr{S}}(z, {Z}(s, \cdot) )   - \mathscr{F}_{\mathscr{S}}(\bar{z}, {Z}(s, \cdot) )\|\, \dif s
      \\\nonumber &    +2\int_0^t \|\mathscr{F}_{\mathscr{M}}(z, {Z}(s, \cdot) )   - \mathscr{F}_{\mathscr{M}}(\bar{z}, {Z}(s, \cdot) )\|\, \dif s +\e
     \le 
  C  d^{-\delta} ({\lp})^{\alpha+3}
+\e
   \end{align}
   with $C = C(T, L_1,L_2, L_{22}, \bar{\gamma}, \|\mu\|_\ls, \|\cK\|_\ls)>0$. 
\end{proof}

\subsection{Stability and non-explosiveness of solutions (removing the stopping time)}\label{sec:supporting_lemmas}
In this section, we verify the stability of approximate solutions to our integro-differential equation \eqref{eq:ODE_resolvent_2} and show that solutions are non-explosive under suitable assumptions (see Proposition \ref{prop:nonexplosiveness}).  This allows for the removal of the stopping time that appears in Proposition \ref{prop:approx_sol}.

We begin with a general stability result that holds even for the case of a growing number of classes, but requires the pseudo-Lipschitz parameter $\alpha=0$.  We then prove that, for a fixed number of classes, stability holds for any $\alpha$.

\begin{proposition}[Stability growing $\ls$]\label{prop:stability_Q}
Suppose $\alpha 
 = 0$ and the assumptions in our set-up. Denote by $\cA_q$ the function class of all polynomial function $q:\mathbb{C}^\ls: \mathbb{C}$ such that $\oint |q| \Dif z  < D$ for some constant $D>0$ independent of $d$. For all $(\e,Q, T)-$ approximate solutions $\csZ_1$ and  $\csZ_2$, define the function $\mathscr{Q}_i(t,q) \defas \oint q(z)\csZ_i(t,z)\Dif z$ 
for $i = 1,2.$  Then, there exists a constant $C = C(T, \bar{\gamma}, \|\cK\|_{\ls}, \|\mu\|_{\ls},  L_1, L_2, L_{22}, D)$, such that 
\begin{equation}\label{eq:Q_diff_stability}
  \sup _{0\le t\le (T\wedge \hat{\tau}_Q)} \sup_{q\in \cA_q} \,\left \|\mathscr{Q}_1(t,q) - \mathscr{Q}_2(t,q)\right\|\le C\e  
\end{equation} 
with $\hat{\tau}_Q = \min\{{\tau}_{Q,\star}(\csZ_1),{\tau}_{Q,\star}(\csZ_2)\}$. 
\end{proposition}
\begin{proof}

By the assumptions of the proposition, $\csZ_1$ and $\csZ_2$ are both $(\e, Q, T)$-approximate solutions, i.e. they satisfy Definition \ref{def:approx_solution}. We therefore have that, for $i = \{1,2\}$,
\[
\mathscr{Q}_i(t,q) = \mathscr{Q}_i(0,q) + \int_0^t \cF_i(s,q) \, \dif s +\e_i(\mathscr{Q},t)
\]
with $\sup_{0\le t\le T\wedge \tau_Q} \|\e_i(\mathscr{Q},t)\| \le D\e,$ and we define $\cF_i(s,q) \defas \oint q(z)\mathscr{F}(z, \mathscr{Z}_i(s, \cdot) )\Dif z.$ 
Since $\hat{\tau}_Q\le  {\tau}_{Q,\star}(\csZ_1)$ and $\hat{\tau}_Q\le  {\tau}_{Q,\star}(\csZ_2)$, we can work on the smallest time $\hat{\tau}_{Q}.$ To prove Eq. \eqref{eq:Q_diff_stability} we 
need to show first that for all $s\le \tau_Q$, there exists a constant, $L(\mathcal{F})>0$ such that
\begin{align}\label{eq:F_Lip}
\sup_{q\in \cA_q}\| \cF_1(s,q) -  \cF_2(s,q)\| \le L(\mathcal{F}) \sup_{q\in \cA_q}\|\csQ_1(s,q) - \csQ_2(s,q)\| 
\end{align}
i.e. that the map $\cF$ is $L$-Lipschitz. 

We denote in the following for $a \in \{1,2\}$ 
$$\r_{t,i,a}\defas 
\sqrt{\mathscr{B}_{i,a}(t)}v +\mathcalligra{m}_{i,a}(t) \text{ with } v \sim \mathcal{N}(0, I_\ell)$$ with $\mathcalligra{m}_{i,a}(t)\defas \oint \mathscr{M}_{i,a}(t,z) \Dif z$ and $\mathscr{B}_{i,a}(t)\defas \oint z_i\mathscr{S}_{i,a}(t,z) \Dif z$ .  We first notice that, 
\begin{align}\label{eq:cF_to_MS}
\|\cF_1(s,q) - \cF_1(s,q)\|
\le  \|\cF_{\csS,1}(s,q)- \cF_{\csS,2}(s,q)\|+2\|\cF_{\csM,1}(s,q) - \cF_{\csM,2}(s,q)\|
\end{align}
with $\cF_{\csS,a}(s,q) \defas \oint q(z)\csF_{\csS}(z, \mathscr{Z}_a(s, \cdot) )\Dif z$ and $\cF_{\csM,a}(s,q) \defas \oint q(z)\csF_{\csM}(z, \mathscr{Z}_a(s, \cdot) )\Dif z$ with $a = {1,2}$   following the definition of $\cF_\csS,$ and $\cF_\csM,$ in Eq. \eqref{eq:integro_SM_hard}, hard label. The proof for the soft label is similar and, therefore, omitted. 
\begin{align}\label{eq:FM1-FM2}
    \|\cF_{\csM,1}(s,q)- \cF_{\csM,2}(s,q)\|  &\le 2|\Omega| \bar{\gamma}\|\csQ_1(s,q) - \csQ_2(s,q)\|_{\Omega} \sum_{i=1}^\ls p_i \|\EE[\nabla^2f_i(\r_{s,i,1})]\|_{\op}
    \\\nonumber & + 2|\Omega|\bar{\gamma} \sup_{z\in \Omega^\ls}\|\csM_2(s,z)\|_{\op}\sum_{i=1}^\ls p_i\|\EE[\nabla^2f_i(\r_{s,i,1})] - \EE[\nabla^2f_i(\r_{s,i,2})]\|
     \\\nonumber & + \bar{\gamma}
    \| \sum_{i=1}^\ls p_i\ip{\mu_i,\mu}_{\R^d}\|
     \sup_{i\in [\ls]}\|\EE[\nabla f_i(\r_{s,i,1})] - \EE[\nabla f_i(\r_{s,i,2})]\|.
\end{align}
where $\bar{\gamma}$ is a uniform bound on $\gamma_k.$ Similarly for $\cF_\csS$ 
\begin{align}\label{eq:FS1-FS2}
    \|\cF_{\csS,1}(s,q)- \cF_{\csS,2}(s,q)\|  &\le 4 |\Omega|\bar{\gamma}\|\csQ_1(s,q) - \csQ_2(s,q)\|_{\Omega}
     \sum_{i=1}^\ls p_i 
    \|\EE[\nabla^2f_i(\r_{s,i,1})]\|_{\op}
    \\\nonumber & + 4 |\Omega|\bar{\gamma}\sup_{z\in \Omega^\ls}\|\csS_2(s,z)\|_{\op}
    \sum_{i=1}^\ls p_i 
    \|\EE[\nabla^2f_i(\r_{s,i,1})] - \EE[\nabla^2f_i(\r_{s,i,2})]\|
\\\nonumber &+ 2 \bar{\gamma}\|\csQ_1(s,q) - \csQ_2(s,q)\|_{\Omega}
     \sum_{i=1}^\ls p_i 
    \|\EE[\nabla f_i(\r_{s,i,1})]\|_{\op}   \\\nonumber & + 2 \bar{\gamma}\sup_{z\in \Omega^\ls}\|\csM_2(s,z)\|_{\op}
    \sum_{i=1}^\ls p_i 
    \|\EE[\nabla f_i(\r_{s,i,1})] - \EE[\nabla f_i(\r_{s,i,2})]\|
     \\\nonumber & + \bar{\gamma}^2 (\|\mu\|_{\ls}^2+\|\cK\|_\ls) 
      \sum_{i=1}^\ls p_i 
     \|\EE[\nabla f_i(\r_{s,i,1})^{\otimes 2}] - \EE[\nabla f_i(\r_{s,i,2})^{\otimes 2}]\|.
\end{align}
Next, using the stopping time, for any $a \in \{1,2\}$ $$\sup_{z\in \Omega^\ls} \left\{\|\csM_a(s,z)\|_{\op},\|\csS_a(s,z)\|_{\op}\right\}
\le 
\sup_{z\in \Omega^\ls} \|\csZ_a(s,z)\|_{\op}
\le Q.
$$
Due to the stopping time criteria, Assumption \ref{ass:risk_fisher_U} can be reduced to Assumption \ref{ass:risk_fisher} with a Lipschitz constant multiplied by the factor $1+4(Q\lp)^\alpha$. Hence, 
by Lemma \ref{lem:Hess_growth} (proved in the next subsection), we also have that $\|\EE[\nabla^2f_i(\r_{s,i,1})]\|_{\op}\le L_1(1+4(Q\lp)^\alpha)$ for all $i\in [\ls].$ In addition, by Assumption \ref{ass:risk_fisher}, for all $i \in [\ls]$
\begin{align}
&\|\EE[\nabla^2f_i(\r_{s,i,1})] - \EE[\nabla^2f_i(\r_{s,i,2})]\| \\\nonumber 
&\le  L_{2} (4(Q\lp)^{\alpha}+1)(\|\mathcalligra{m}_{i,1}(s) - \mathcalligra{m}_{i,2}(s)\| +\|\mathscr{B}_{i,1}(s) - \mathscr{B}_{i,2}(s)\| )
\\\nonumber &\le  L_{2} (4(Q\lp)^{\alpha}+1)
(\|\csQ_1(s,q_1) - \csQ_2(s,q_1)\|+ \|\csQ_1(s,q_0) - \csQ_2(s,q_0)\|)
\end{align}
with $q_1(x)=x$, $q_0(x) = 1$. 
Similar expressions hold for  $\EE[\nabla f_i(\r_{s,i,a})]$ and $\EE[\nabla f_i(\r_{s,i,a})^{\otimes 2}]$ replacing $L_2$ by $L_1$ and $L_{22}$ respectively. 
Plugging the above in Eq. \eqref{eq:cF_to_MS}  and taking the supremum over $q\in \cF_q$, and using Assumption \ref{ass:data} on the means,  
we obtain Eq. \eqref{eq:F_Lip} with some Lipschitz constant $L(\mathcal{F}) =
L(\bar{\gamma}, \|\cK\|_{\ls}, \|\mu\|_{\ls}, L_1, L_2, L_{22})$. 

Next, by the definition of $(\e, Q, T)-$approximate solutions for $\csZ_1$, and $\csZ_2$, for any $t\in [0, T\wedge \hat{\tau}_Q]$
\begin{align}
\sup_{q\in \cA_q}\|\csQ_1(t,q) - \csQ_2(t,q)\|&\le 2D\e   +  \sup_{q\in \cA_q}\int_0^t \| \cF_1(s,q) -  \cF_2(s,q)\|\dif s
\\\nonumber & \le 2D\e+ 
L(\mathcal{F})\int_0^t \sup_{q\in \cA_q}\|\csQ_1(s,q) - \csQ_2(s,q)\|\dif s
\end{align}
Then, taking the supremum over $t\in [0, T\wedge \hat{\tau}_Q],$ and 
applying Gronwall's inequality, we obtain that, 
\begin{align}
\sup _{0\le t\le (T\wedge \hat{\tau}_Q)} \sup_{q\in \cA_q} \,\left \|\mathscr{Q}_1(t,q) - \mathscr{Q}_2(t,q)\right\|&\le  2D\e e^{L(\mathcal{F})T}.  
\end{align}
The result of the proposition follows. 

\end{proof}
\begin{remark}
When the number of classes and indices is of finite order, i.e., $\ell, \ls = O(1).$ The polynomial complexity is bounded by a constant independent of $d$, and therefore, a stronger stability statement can be made analogous to the one in Proposition 10 in \cite{collinswoodfin2023hitting} for any $\alpha\ge 0$ (Note that, having \eqref{eq:z_1_Z_2_stabel}, \eqref{eq:Q_diff_stability} immediately follows; see \eqref{eq:Q_Lip}). 
However, for most practical purposes, having \eqref{eq:Q_diff_stability} is enough.

\end{remark} 
\begin{corollary}
[Stability finite $\ls$]\label{cor:stability_z_l_finite}
Suppose $\ell, \ls  = O(1).$ For all $(\e,Q, T)-$approximate solutions $\csZ_1$ and  $\csZ_2,$ there exists a positive constant, $C = C(\e, T, \bar{\gamma}, \|\cK\|_{\ls}, \|\mu\|_{\ls}, \alpha, L_1, L_2, L_{22},Q, \ell, \ls)$ such that
\begin{align}\label{eq:z_1_Z_2_stabel}
 \sup_{0\le t\le (T\wedge\hat{\tau}_Q)} \|\csZ_1(t,z) - \csZ_2(t,z)\|_{\Omega}\le C \e   
\end{align}
with $\hat{\tau}_Q = \min\{{\tau}_Q(\csZ_1),{\tau}_Q(\csZ_2)\}.$
\end{corollary}
\begin{proof}
The proof is similar to the one in Proposition \ref{prop:stability_Q}, replacing $\csQ_i$ by $\csZ_i$ for $i = \{1,2\}$ 
 with the only difference that the map $\cF$ is different. We therefore need to show that there exists $L>0$ such that, 
\begin{align}
    \|\csF(z,\csZ_1(s,z)) - \csF(z,\csZ_2(s,z))\|_{\Omega}  &\le L \|\csZ_1(s,z) -\csZ_2(s,z)\|_\Omega.
\end{align}
As in Proposition \ref{prop:stability_Q} replacing $\csQ_i$ by $\csZ_i$ for $i = \{1,2\}$, one can show that the map $\csF$ is Lipschitz in $\csZ$ with $L = L(\lp, \bar{\gamma}, \|\cK\|_{\ls}, \|\mu\|^2_{\ls}, L_1, L_2, L_{22}).$
This is straightforward, and we omit the proof here.
\end{proof}

The extension for a general function $\varphi$ satisfying Definition \ref{ass:statistic} is given by the following proposition. The proof is similar to the proof of Proposition 11 in \cite{collinswoodfin2023hitting}. We repeat here the proof to highlight the dependency in $\ls$, which is important in our setup. 
\begin{proposition}[Stability general $\varphi$] Suppose $\varphi: \R^{\lp}\times \R^{\lp} \to \R$ is a statistic satisfying Definition \ref{ass:statistic} such that $\varphi(X) = g(\ip{W^{\otimes 2}, q(\{K_i\}_{i=1}^{\ls})})$ with $q$ a polynomial function of finite order and the sum of its coefficients in absolute value is bounded independent of $d$. Suppose $\csZ_1, \csZ_2$ are $(\e,Q,T)-$approximate solutions. Then there exists a positive constant $C$
such that,
\begin{align}
 \sup_{0\le t\le T\wedge\hat{\tau}_Q} \left\|g\left(\oint q(z)\csZ_1(t,z)\Dif z \right) - g\left(\oint q(z)\csZ_2(t,z)\Dif z \right)\right\|\le C (\lp)^\alpha \e   
\end{align}
where $\hat{\tau}_Q = \min\{{\tau}_Q(\csZ_1),{\tau}_Q(\csZ_2)\}.$
\end{proposition}

\begin{proof} 
We define the function $\mathscr{Q}_i(t) \defas \oint q(z)\csZ_i(t,z)\Dif z$ and stopped processes $\mathscr{Q}_i^{\hat{\tau}_Q}(t) \defas \mathscr{Q}_i(t\wedge \hat{\tau}_Q)$ and $\mathscr{Z}_i^{\hat{\tau}_Q}(t,z) \defas \mathscr{Z}_i(t\wedge \hat{\tau}_Q,z)$ for $i = 1,2.$ First, we observe that 
\begin{align}\label{eq:norm_Q}
\|\mathscr{Q}_i^{\hat{\tau}_Q}(t)\|\le \oint |q(z)|\|\mathscr{Z}_i^{\hat{\tau}_Q}(t,z)\|\Dif z  \le  \oint |q(z)|\Dif z \|\mathscr{Z}_i^{\hat{\tau}_Q}(t,\cdot)\|_{\Omega }\le  \oint |q(z)|\Dif z \lp Q.     
\end{align} 
If $\lp$ is bounded independent of $d$ in which $\oint |q(z)|\Dif z$ is uniformly bounded, then using the fact that the function $\mathscr{Q}$ is Lipschitz and Corollary \ref{cor:stability_z_l_finite},
\begin{align}\label{eq:Q_Lip}
\|\mathscr{Q}_1^{\hat{\tau}_Q}(t) - \mathscr{Q}_2^{\hat{\tau}_Q}(t)\|   \le \left(\oint |q(z)|\Dif z \right)\|\mathscr{Z}_1^{\hat{\tau}_Q}(t,z) - \mathscr{Z}_2^{\hat{\tau}_Q}(t,z)\|_{\Omega}\le C\e. 
\end{align}
In the setting where 
$\lp$ grows with $d$, we use Proposition \ref{prop:stability_Q} to bound the difference in $\mathscr{Q}$, which states  a similar bound with a different constant.  
Since $g$ is $\alpha$-pseudo-Lipschitz with $L(g)>0$ being the pseudo-Lipschitz constant, and using Eq. \eqref{eq:norm_Q} and \eqref{eq:Q_Lip}, 
\begin{align}
 \|g(\mathscr{Q}_1^{\hat{\tau}_Q}(t)) - g(\mathscr{Q}_2^{\hat{\tau}_Q}(t))\|&\le L(g) \|\mathscr{Q}_1^{\hat{\tau}_Q}(t) - \mathscr{Q}_2^{\hat{\tau}_Q}(t)\|(1+\|\mathscr{Q}_1^{\hat{\tau}_Q}(t)\|^\alpha+\|\mathscr{Q}_2^{\hat{\tau}_Q}(t)\|^\alpha) 
 \le C
 (\lp)^{\alpha} \varepsilon
\end{align}
 with $C = C(Q, L(g), \oint |q(z)|\Dif z)>0. $ Taking the supremum over all $0\le t\le T$ and applying Proposition \ref{prop:stability_Q} when $\lp$ grows with $d$ or Eq. \eqref{eq:Q_Lip} complete the proof. 
\end{proof}

Finally, we prove non-explosiveness, the condition needed to remove the stopping time from Proposition \ref{prop:approx_sol} and thus prove Theorem \ref{thm:main_risk_m_v}.  

\begin{proposition}[Non-explosiveness] \label{prop:nonexplosiveness}
Suppose that the assumptions in our set-up
hold. 
Then, $\mathscr{N}(t)$ solves the following ordinary differential equation:
\begin{equation}\label{eq:N(t)}
\frac{\dif \mathscr{N} (t)}{\dif t}  
=-2\gamma(s)\sum_{i=1}^\ls p_i \E[\ip{x_{t,i}, \nabla_x f_i(\r_{t,i})}]
    +\frac{\gamma(s)^2}{d}\left(\sum_{i=1}^\ls p_i\Tr(K_i+\mu_i\mu_i^\top)\EE[\|\nabla_x f_i(\r_{t,i})\|^{2}]\right)  
\end{equation}
with $x_{t,i}\in \R^\ell$ and $x^\star\in \R^\ls,$ such that $$\r_{t,i} = x_{t,i}\oplus x^\star \defas 
\sqrt{\mathscr{B}_i(t)}v +\mathcalligra{m}_i(t) \text{ with } v \sim \mathcal{N}(0, I_{\ell+\ls})$$ and $\nabla_x$ being the derivative with respect to the first $\ell$ components. For hard labels $\r_{t,i} = x_{t,i},$ and $\nabla_x = \nabla.$ 

Suppose further that the objective function $f$ is $\alpha$-pseudo-Lipschitz with $\alpha\le 1$. Then, there is a constant $C$ depending on $\|\cK\|_{\ls}$, $\bar{\gamma}$, $L(f), \EE[\|\e\|^2]$ so that
\[
\mathscr{N}(t) \leq (1+ \mathscr{N}(0))e^{C t}.
\]
for all time $t\in [0,T]$ such that $(\mathscr{B}_i(t),\cm_i(t))$ is in $\mathcal{U}$ for all $i$. 

\end{proposition}
\begin{proof} To derive Eq. \ref{eq:N(t)}, we take contour integration and trace in either Eq. \ref{eq:integro_SM_hard} or \ref{eq:integro_SM_soft}, following the definition of $\mathscr{N}(t) = \oint\Tr(\mathscr{Z}(t,z))\Dif z$. The gradient term is obtained by reverse application of Stein's Lemma. From Assumption \ref{ass:pseudo_lipschitz} that $f$ is $\alpha$-pseudo Lipschitz with $\alpha\le 1$, we conclude that a.s. 
$$
    \|\nabla_x f_i(\r_{t,i})\| \le L(f)(1+\|\r_{t,i}\| +\|\e_t\|)^\alpha \le L(f)(1+\|\r_{t,i}\| +\|\e_t\|). 
$$
It then follows that the Hessian term, 
\begin{align*}
   &\sum_{i=1}^\ls p_i\Tr(K_i+\mu_i\mu_i^\top)\EE[\|\nabla_x f_i(\r_{t,i})\|^{2}]\\ &\le  (\|\cK\|_{\ls}+\|\mu\|_{\ls}^2/d)L^2(f)\sum_{i=1}^\ls p_i \EE[(1+\|\r_{t,i}\|+\|\e_t\|)^2]
  \\ & \le \sum_{i=1}^{\ls} p_iC(1+ \tr(\csB_i(t)) +\|\csm_i(t)\|^2) 
  \le \sum_i p_iC(1+ \tr(\csB_i(t)) +\|\mu_i\|^2 \csN(t) ) 
   \le C(1+\csN(t) ) 
\end{align*}
for some $C = (\|\cK\|_{\ls}, \mu_i\mu_i^\top,\EE[\|\e\|^2], L(f))>0$.  Similarly, by application of Cauchy-Schwarz, the gradient term can be bounded, 
\begin{align*} |\sum_{i=1}^\ls p_i\E[\ip{ x_{t,i},\nabla_x f_i(\r_{t,i})}]|^2
\le \sum_{i=1}^\ls p_i\E[\|x_{t,i}\|^2]\sum_{i=1}^\ls p_i\EE \|\nabla_x f_i(\r_{t,i})\|^2
\le C \csN(t)
(1+\mathscr{N})\le C(1+\csN(t))^2.
\end{align*}
Hence, for some other constant which depends also on $\bar{\gamma},$ 
we obtain that 
$$\frac{\dif \mathscr{N} (t)}{\dif t} \le C(1+ \mathscr{N} (t)).$$
The proof is completed by application of Gronwall's inequality $(1+\mathscr{N} (t))\le (1+\mathscr{N} (0))e^{Ct}.$ 
\end{proof}

\subsection{Error bounds}\label{sec:error_bds}

\subsubsection{Bounding the SGD martingales}

In Lemma \ref{lem:martingale_error}, we will bound the martingales $\cM^{\grad}$ and $\cM^{\hess}$ whose increments are given by
\begin{align}
\Delta \mathcal{M}_k^{\text{Grad}}(\varphi) &\defas  \frac{\gamma_k}{d} \ipa{\nabla \varphi(W_k),  \Delta_{k,I_{k+1}}   
      - 
     \sum_{i = 1}^{\ls} p_i \EE \big [\Delta_{k,i}  \, \mid \, \mathcal{F}_k \big ]}, \\
\Delta \mathcal{M}_k^{\text{Hess}}(\varphi) &\defas \frac{\gamma_k^2}{2d^2} \left( \ipa{\nabla^2\varphi(W_{k}),\Delta_{k,I_{k+1}}^{\otimes 2} } - \ipa{\nabla^2\varphi(W_{k}),\sum_{i=1}^{\ls}p_i\EE \big [ \Delta_{k,i}^{\otimes 2} \,  \mid  \, \mathcal{F}_k \big ]}\right),\\
&\text{where }\Delta_{k,i} \defas  a_{k+1,i}\otimes \nabla_x f_{i}(r_{k,i}).
\end{align}

These quantities are similar to those in \cite{collinswoodfin2023hitting}, but we implement a more streamlined proof method that works for $\ell$ and $\ls$ with up to poly-logarithmic growth.  Such a method could be implemented in the setting of \cite{collinswoodfin2023hitting} as well.
For convenience, we rewrite the martingale increments in terms of centered Gaussian vectors.  Recall that \begin{equation}\label{eq:g_Xk_def_setup}
a_{k+1,i}=y_{k+1,i}+\mu_i,\quad r_{k,i}=\hat{X}_k^\top (y_{k+1,i}+\mu_i),\qquad\text{ where }y_{k+1,i}\sim\cN(0,K_i).
\end{equation}
We now define, for each $f_i$, a new function
\begin{equation}\label{eq:g_Xk_def}
    g_{X_k,i}(r,\e)\defas f_i(\hat{X}_k^\top \mu_i+r,\e),
\end{equation}
noting that it is reasonable to allow a function depending on $\hat{X}_k=X_k\oplus X^{\star}$ since the martingale increments are conditioned on this.  We observe that $\nabla_x g_{X_k,i}(r,\e)= \nabla_x f_i(\hat{X}_k^\top \mu_i+r,\e)$.  

For the martingale bounds, we will need a few supporting lemmas, including the following growth bound on $\nabla_x g$, which is proved in \cite{collinswoodfin2023hitting}.

\begin{lemma}[Growth of $\nabla g$, Lemma 4 of \cite{collinswoodfin2023hitting}]\label{lem:gradient_growth_bound} Given an $\a$-pseudo-Lipschitz function $g:\R^{\lp}\to\R$ with Lipschitz constant $L(g)$, and noise $\e\sim \cN(0,I_d)$ independent of $a$, then for $p>0$ and any $r\in\R^{\ell+\ls}$,
\[
\|\nabla_x g(r)\|^p\leq C(\a)p(L(g))^p(1+\|r\|+\|\e\|)^{\max\{1,\a p\}}.
\]
Furthermore, if $r=\ip{X,n}$ where $n\sim \cN(0,K)$, then we can express the growth rate of $\nabla_x g(r)$ 
as
\[
    \E[\|\nabla_x g(r)\|^p]\leq C(\a)p(L(g))^p(1+\|K\|^{1/2}_{\textup{op}}\|X\|)^{\max\{1,\a p\}}
\]
\end{lemma}
In our context, this lemma will be applied for the gradient $\nabla_x g_{X_k,i}(\hat{r})$ where $\hat{r}_{k,i}= r_{k,i}-\hat{X}_k^\top \mu_i$ and $\E$ will be the expectation conditioned on $X_k$ and on $I_{k+1}=i$.

\begin{lemma}[Hessian Growth bound] \label{lem:Hess_growth} Suppose Assumption \ref{ass:risk_fisher} holds. Then $\| \E[\nabla^2 f_i(r_i)]\|_{\textup{op}} \le L_1.$    
\end{lemma} 
\begin{proof}
By Assumption \ref{ass:risk_fisher}, we have  
$$\|\E[\nabla f_i(\sqrt{B_i} z + m_i)] - \E[\nabla f_i(\sqrt{B_i'} z + m_i')]\|\le L_1 (\|B_i - B_i'\|+ \|m_i - m_i'\|), $$ and denote by $r_i = \sqrt{B_i} z + m_i.$ 
Take $B_i = B_i'$ and $m_i' = m_i+ c x$ for some $c>0$ and $x\in \R^\ell$ some determinstic vector. Define the function $g(t) = \E[\nabla f_i(r_i+t cx)]$ note that $g'(t) = \ip{\E[\nabla^2 f_i(r_i +t cx)], c x}_{\R^\ell}.$ By the mean value theorem, there exists some $t_c\in [0,1]$ such that
\begin{align}
\E[\nabla f_i(r_i+ c x)] - \E[\nabla f_i(r_i)] 
=  g(1) - g(0) =g'(t_c) 
= \ip{\E[\nabla^2 f_i(r_i+t_c cx)], c x}_{\R^\ell}.
\end{align}
Taking the norm of both sides and applying the Lipschitz condition, we have
\begin{align}
\| \ip{\E[\nabla^2 f_i(r_i+t_c cx)],  x}_{\R^\ell}\|\le L_1 \|x\|.
\end{align}
By taking $c\to 0$
 and using the fact that $f_i$ is twice differentiable, we have $\| \ip{\E[\nabla^2 f_i(r_i)],  x}_{\R^\ell}\|\le L_1 \|x\|$
Thus the absolute value of any eigenvalue of $\E[\nabla^2 f_i(r_i)]$ is bounded by $L_1.$  
\end{proof}

Next, we introduce a corollary of the Burkholder--Davis--Gundy inequality, which will streamline the subsequent proof.

\begin{lemma}\label{lem:BDG_ourcase} Let $M$ be a martingale with increments denoted $\{\Delta M_k\}$ and let $\tau$ be a stopping time.  Suppose that, for some $r > 1$ and some $C, c > 0$ and some $\theta\in(0,1)$ we have that, for $1 \le p \le (\log n)^\theta$,
\[
\sup_{k \le n} \mathbb{E}[(\Delta M_k)^{2p} \mathbf{1}_{k \le \tau}] \le (Cp)^{cp} n^{-rp}.
\]
Then, for any $\zeta > 0$, there is a $c'(\zeta, \theta) > 0$ such that
\[
\mathbb{P}\left(\max_{0 \le k \le n} |M_{k \wedge \tau} - M_0| \ge \frac{1}{n^{(r-1)/2 - \zeta}}\right) \le \exp(-c'(\log n)^{1+\theta}).
\]
\end{lemma}

\begin{proof} Let $[M^\tau]_n = \sum_{k \le \tau \wedge n} (\Delta M_k)^2$ denote the stopped quadratic variation. Since $(\sum_{i \le n} x_i)^p \le n^{p-1} \sum_i x^p_i$ for $x_i \ge 0$, we have for each $1 \le p \le (\log n)^\theta$,
\[
\mathbb{E}[M^\tau]^p_n \le n^p \sup_{k \le n} \mathbb{E}[(\Delta M_k)^{2p} \mathbf{1}_{k \le \tau}] \le \frac{(Cp)^{cp}}{n^{(r-1)p}}.
\]
Thus by the Burkholder--Davis--Gundy inequality, we have
\[
\mathbb{E}\left[\max_{0 \le k \le n} |M_{k \wedge \tau} - M_0|\right]^p \le \frac{(Cp)^{cp}}{n^{(r-1)p/2}},
\]
after changing $C, c > 0$ (the BDG inequality includes a multiplicative, $p$-dependent constant that is less than $p^p$ and can thus be absorbed).  The result then follows by Markov's inequality. 
\end{proof}

\begin{lemma}[Martingale bounds]\label{lem:martingale_error}
For any $\zeta>0$ and 
$T>0$, with overwhelming probability
\[
\sup_{0\leq t\leq T\wedge\tau_Q}\left(\|\cM_{\lfloor dt\rfloor}^{\grad}(Z(\cdot,z))\|+\|\cM_{\lfloor dt\rfloor}^{\hess}(Z(\cdot,z))\|\right)<d^{-\frac12+\zeta}.
\]
\end{lemma}

\begin{proof}
    Let $\varphi(W)\defas Z_{jm}(W)$ be a coordinate of $Z$.  Then it suffices to prove the bound for $\cM^{\grad}_{\lfloor dt\rfloor}(\varphi)$ and $\cM^{\hess}_{\lfloor dt\rfloor}(\varphi)$ and apply a union bound to obtain the lemma for $Z$.  
We first express the martingale increments in terms of centered Gaussian vectors.  Recall the function $g_{X_k,i}(r,\e)$
defined in \eqref{eq:g_Xk_def}.  Taking $\hat{r}_{k,i}= r_{k,i}-\hat{X}_k^\top \mu_i$, we can express $\Delta\cM_k^{\grad}$ and $\Delta\cM_k^{\hess}$ as
\begin{equation}
\Delta\cM_k^{\grad} (\varphi)= \frac{\gamma_k}{d}\left(A_{k,I_{k+1}}-\sum_{i=1}^{\ls}p_i\EE[A_{k,i}|\cF_k]\right),\quad
\Delta\cM_k^{\hess} (\varphi)= \frac{\gamma_k^2}{2d^2}\left(H_{k,I_{k+1}}-\sum_{i=1}^{\ls}p_i\EE[H_{k,i}|\cF_k]\right)
\end{equation}
where

\begin{align*}
    A_{k,i}&\defas\ipa{\ip{\nabla\varphi(W_k),n_{k+1}+\mu_i}_{\R^d},\nabla_x g_{X_k,i}(\hat{r}_{k,i},\e_{k+1})},\\
    H_{k,i}&\defas\ipa{\ip{\nabla^2\varphi(W_k),(n_{k+1}+\mu_i)^{\otimes2}}_{\R^{d\times d}},(\nabla_x g_{X_k,i}(\hat{r}_{k,i},\e_{k+1}))^{\otimes2}}.
\end{align*}
We will bound the stopped versions of $\cM_{k}^{\grad},\cM_{k}^{\hess}$ using Lemma \ref{lem:BDG_ourcase}.  To prove the condition of that lemma, it suffices to bound $\sup_{k\leq Td,\;i\leq\ls}\E[(A_{k,i})^{2p}\mathbf1_{k\leq\tau_Q}]$ and similarly for $H_{k,i}$.  We begin by writing
\begin{equation}\begin{split}
    |A_{k,i}|&\leq\|\nabla\varphi(W_k)\|_{\e}\|n_{k+1}+\mu_i\| \;\|\nabla_xg_{X_k,i}(\hat{r}_{k,i},\e_{k+1})\|\\
    |H_{k,i}|&\leq\|\nabla^2\varphi(W_k)\|_{\e}\|n_{k+1}+\mu_i\|^2 \|\nabla_xg_{X_k,i}(\hat{r}_{k,i},\e_{k+1})\|^2
\end{split}\end{equation}
where, $\|\cdot\|_{\e}$ denotes the injective tensor norm. For $b\in\{1,2\}$, we have $\|\nabla^b\varphi(W_k)\|_{\op}=\sup_{z_i\in\Omega}\|\nabla^b Z_{jm}(W_k,z_i)\|_{\e}\leq\|\nabla^b Z(W_k,z)\|_{\Omega}$.  Using Lemma \ref{lem:normequivalence} and the stopping time, $\|\nabla Z(W_k,z)\|_{\Omega}\leq\sqrt{\lp}\|W_k\|\leq \lp Q$ and $\|\nabla^2 Z(W_k,z)\|_{\Omega}\leq \lp $.  Then, using Cauchy-Schwarz,
\begin{equation}\begin{split}
    \E[(A_{k,i})^{2p}\mathbf1_{k\leq\tau_Q}]&\leq (\lp Q)^{2p}\EE[\|n_{k+1}+\mu_i\|^{4p}]^\frac12\EE[\|\nabla_xg_{X_k,i}\|^{4p}]^\frac12,\\
    \E[(H_{k,i})^{2p}\mathbf1_{k\leq\tau_Q}]&\leq (\lp)^{2p}\EE[\|n_{k+1}+\mu_i\|^{8p}]^\frac12\EE[\|\nabla_xg_{X_k,i}\|^{8p}]^\frac12.
\end{split}\end{equation}
Next, we use the bound $\|n_{k+1}+\mu_i\|^{4bp}\leq 2^{4bp}(\|n_{k+1}\|^{4bp}+\|\mu_i\|^{4bp})$ and note that $\|\mu_i\|$ is uniformly bounded in $i$. Recalling that $\|n_{k+1}\|$ is $\|K_i\|_\op^\frac12$-sub-Gaussian, we get $\EE[\|n_{k+1}\|^{4bp}]^\frac12\leq(C\|K_i\|_\op\sqrt{bp})^{2bp}$ for some $C$.  From Lemma \ref{lem:gradient_growth_bound}, we have $ \E[\|\nabla_x g_{X_k,i}\|^{4bp}]^\frac12\leq C(\a)^{2bp}(bp)^{2bp}(L(g_{X_k,i}))^{2bp}(1+\|K_i\|^{1/2}_{\op}\|X_k\|)^{\max\{\frac12,2\a bp\}}$ where $\|X_k\|\leq\|W_k\|\leq\sqrt{\lp}Q$.  Putting all this together, and noting that $\|K_i\|$ and $L(g_{X_k,i})$ are bounded uniformly in $i$, we conclude that, for some $C,c$ depending on $K,Q,L(g),\a$, we have
\begin{equation}
    \sup_{k\leq Td,\;i\leq\ls}\E[(A_{k,i})^{2p}\mathbf1_{k\leq\tau_Q}]+\E[(H_{k,i})^{2p}\mathbf1_{k\leq\tau_Q}]\leq (Cp\lp)^{cp}.
\end{equation}
 The same bound holds for the terms involving the conditional expectations, since all of our arguments involved norm bounds and, by Jensen's inequality, $\big\|\E[\;\cdot\;|\cF_k]\big\|^{R}\leq\E\big[\|\cdot\|^{R}\big|\cF_k\big]$ for any $R\geq1$. Since $\gamma_k$ is bounded and $\lp$ has at most poly-log growth, we finally conclude
\begin{equation}\begin{split}
    \sup_{k\leq Td}\E[(\Delta \cM_{k}^{\grad})^{2p}\mathbf1_{k\leq\tau_Q}]&\leq d^{-2p}(Cp\lp)^{cp},\\
    \sup_{k\leq Td}\E[(\Delta \cM_{k}^{\hess})^{2p}\mathbf1_{k\leq\tau_Q}]&\leq d^{-4p}(Cp\lp)^{cp}.
\end{split}\end{equation}
 To complete the proof, we apply Lemma \ref{lem:BDG_ourcase} with $r=2$ for $\cM^{\grad}$ and $r=4$ for $\cM^{\hess}$.  This actually gives a smaller bound on $\cM^{\hess}$ than what is stated in the Lemma, so it is the contribution from $\cM^{\grad}$ that dominates.
\end{proof}

\subsubsection{Bounding Hessian error term}

\begin{lemma}[Hessian error term]\label{lem:Hess_error} 
Suppose $\{f_i\}_{i=1}^{\ell^\star}$ are $\a$-pseudo-Lipschitz functions $\R^\ell\to\R$ (see Assumption \ref{ass:pseudo_lipschitz}) and let the statistics $Z(W,z)$ be defined as in \eqref{def:Z(X,z)}. Suppose also that $\lp\leq Cd^{\omega}$ for some fixed $C>0$ and some $0\leq\omega<\frac13$. Then, for any  
$T>0$, with overwhelming probability
\[
\sup_{z\in\Omega}\sup_{0\leq t\leq T\wedge\tau_Q}\sum_{k=0}^{\lfloor td\rfloor-1}\|\EE[\cE_k^{\hess}(Z(\cdot,z))|\cF_k]\|\leq d^{-1+3\omega}.
\]
\end{lemma}
In the proof of the lemma above, we make use of the following basic result of the orthogonality property of Gaussian random vectors, which we present here without proof.
   \begin{lemma}[Gaussian Conditioning] \label{lem:conditioning} Let $\ell < d$. Suppose $v \in \R^d$ is distributed $\cN(0, I_d)$ and $U \in \R^{d\times\ell}$ has orthonormal columns. Then 
\begin{equation}
 v \, | \, \ip{U, v}_{\R^d} \sim v - U U^\top  v + U U^\top  v, 
 \end{equation}
where $v - U U^\top  v \sim \cN(0, I_d - U U^\top )$ and $U U^\top  v \sim N(0, U U^\top )$ with $v - U (U^\top  v)$ independent of $U U^\top  v$. 
\end{lemma}
\begin{proof}[Proof of Lemma \ref{lem:Hess_error}]
    The terms that we are trying to bound involve an expectation, conditional on $\cF_{k,i}\defas\sigma(\{W_j\}_{j=0}^k,\{I_{j}\}_{j=1}^{k},I_{k+1}=i)$.  In order to bound these terms, it is useful to first compute expectations considitioned on a larger $\sigma$-algebra, which we define as $\cG_{k,i}\defas\sigma(\cF_{k,i},\{r_{j,I_{j+1}}\}_{j=0}^k)$, and later take the expectation conditioned only on $\cF_{k,i}$. 
We begin by computing the quantity
\begin{equation}\label{eq:Delta2_conditional}\begin{split}
\EE[\Delta_{k,i}^{\otimes2}|\cG_{k,i}]&=\EE[(a_{k+1,i}\otimes\nabla_xf(r_{k,i},\e_{k+1}))^{\otimes2}|\cG_{k,i}]\\
&=\EE[(a_{k+1,i}-\EE[a_{k+1,i}|\cG_{k,i}])^{\otimes2}|\cG_{k,i}]\otimes\EE_{\e_{k+1}}[(\nabla_xf_i(r_{k,i},\e_{k+1}))^{\otimes2}]\\
&\qquad\quad+\EE[a_{k+1,i}|\cG_{k,i}]^{\otimes2}\otimes\EE_{\e_{k+1}}[(\nabla_xf_i(r_{k,i},\e_{k+1}))^{\otimes2}].
\end{split}\end{equation}
In order to bound the condition expectation and variance appearing above, we first need to understand the conditional distribution of $a_{k+1,i}$ and $a_{k+1,i}a_{k+1,i}^\top $.  Recall that \[
a_{k+1,i}=\sqrt{K_i}v_{k+1}+\mu_i,\quad v_{k+1}\sim \cN(0,I_d),
\]
so we have
\[
v_{k+1}|_{r_{k,i},W_k}\eqd v_{k+1}|_{\hat{X}_k^\top (\sqrt{K_i}v_{k+1}+\mu_i),W_k}
\eqd v_{k+1}|_{\hat{X}_k^\top \sqrt{K_i}v_{k+1}}.
\]
We want to be able to apply the conditioning lemma, so we consider the QR-decomposition $\sqrt{K_i}\hat{X}_k=Q_{k,i}R_{k,i}$ where $R_{k,i}\in\R^{\lpp \times\lpp }$ is upper-triangular and invertible while $Q_{k,i}\in\R^{d\times\lpp }$ has columns that are orthonormal (or possibly 0 if not full rank).  Then we have
\begin{equation}
    a_{k+1,i}|r_{k,i},W_k
    =\sqrt{K_i}v_{k+1}+\mu_i|R_{k,i}^\top Q_{k,i}^\top v_{k+1}
    \eqd \sqrt{K_i}v_{k+1}+\mu_i|Q_{k,i}^\top v_{k+1}
\end{equation}
where the last equality comes from invertibility of $R_{k,i}$.
Now, applying Lemma \ref{lem:conditioning} with $\Pi_{k,i}\defas Q_{k,i}Q_{k,i}^\top $, we get
\begin{equation}
    a_{k+1,i}|r_{k,i},W_k
    \eqd \sqrt{K_i}(v_{k+1}-\Pi_{k,i}v_{k+1})+\sqrt{K_i}\Pi_{k,i}v_{k+1}+\mu_i
\end{equation}
where $(I_d-\Pi_{k,i})v_{k+1}\sim \cN(0,I_d-\Pi_{k,i})$ and $\Pi_{k,i}v_{k+1}\sim N(0,\Pi_{k,i})$ with $(I_d-\Pi_{k,i})v_{k+1}$ independent of $\Pi_{k,i}v_{k+1}$.  Thus we get
\begin{equation}
    \EE[a_{k+1,i}|\cG_{k,i}]=\sqrt{K_i}\Pi_{k,i}v_{k+1}+\mu_i,\qquad v_{k+1}\sim \cN(0,I_d).
\end{equation} 
Furthermore, for the conditional variance, we have
\begin{equation}
    \EE[(a_{k+1,i}-\EE[a_{k+1,i}|\cG_{k,i}])^{\otimes2}|\cG_{k,i}]
    =\sqrt{K_i}(I_d-\Pi_{k,i})\sqrt{K_i}.
\end{equation}
Plugging these computations into \eqref{eq:Delta2_conditional}, we get 
\begin{align*}
    \EE[\Delta_{k,i}^{\otimes2}|\cG_{k,i}]&=
    \left(\sqrt{K_i}(I_d-\Pi_{k,i})\sqrt{K_i}+(\sqrt{K_i}\Pi_{k,i}v_{k+1}+\mu_i)^{\otimes2}\right)\otimes\EE_{\e_{k+1}}[(\nabla_xf_i(r_{k,i},\e_{k+1}))^{\otimes2}]\\
    &=(D_{k,i}^{\Delta}+E_{k,i}^{\Delta})\otimes\EE_{\e_{k+1}}[(\nabla_xf_i(r_{k,i},\e_{k+1}))^{\otimes2}],\\
\end{align*}
where
\begin{align}
    D_{k,i}^{\Delta},&\defas K_i+\mu_i\mu_i^\top \\
    E_{k,i}^{\Delta}&\defas
    -\sqrt{K_i}\Pi_{k,i}\sqrt{K_i}+(\sqrt{K_i}\Pi_{k,i}v_{k+1})^{\otimes2}+\sqrt{K_i}\Pi_{k,i}v_{k+1}\mu_i^\top +\mu_iv_{k+1}^\top \Pi_{k,i}\sqrt{K_i}.
\end{align}
The term $D_{k,i}^{\Delta}$ gives the dominant contribution (which appears in the leading order of the Doob's decomposition of SGD) and we see that
\begin{equation}\label{eq:mart_hesserror_singlesummand}
    \EE[\cE_k^{\hess}|\cF_k]
    = \frac{\gamma_k^2}{2d^2} \ipa{\nabla^2\varphi(W_{k}),\sum_{i=1}^{\ls}p_i\left(\EE \big [ E_{k,i}^{\Delta} \otimes \nabla_x f_{i}(r_{k,i},\e_{k+1})^{\otimes2}\,  \mid  \, \mathcal{F}_k \big ]\right)},
\end{equation}
so it remains to control the contributions from $E_{k,i}^{\Delta}$.
This follows a very similar procedure to that of \cite{collinswoodfin2023hitting}, Proposition A4.  The primary differences are the inclusion of $\mu$-dependent terms and that we allow for $\lp$ slowly growing with $d$.  Before proceeding with these bounds, we recall the injective tensor norm $\|\cdot\|_\e$ and introduce its dual, the \textit{projective tensor norm} (sometimes also called the nuclear norm), given by 
    $\|A\|_{\pi}\defas\sup_{\|B\|_{\e}=1}\ip{A,B}$.
  For the contribution from the first term of $E_{k,i}^{\Delta},$ we have
\begin{equation}\begin{split}
    &\left|\ipa{\nabla^2\varphi(X_k),\sqrt{K_i}\Pi_{k,i}\sqrt{K_i}\otimes\nabla_xf_i(r_{k,i},\e_{k+1})^{\otimes2}}\right|\\
    &=\left|\ipa{\ip{\nabla^2\varphi(X_k),\nabla_xf_i(r_{k,i},\e_{k+1})^{\otimes2}}_{\R^{\ell\times\ell}},\sqrt{K_i}\Pi_{k,i}\sqrt{K_i}}\right|\\
&\leq\|\sqrt{K_i}\Pi_{k,i}\sqrt{K_i}\|_{\pi}\|\ip{\nabla^2\varphi(X_k),\nabla_xf_i(r_{k,i},\e_{k+1})^{\otimes2}}_{\R^{\ell\times\ell}}\|_{\e}\\
&\leq\|K\|_{\e}\|\Pi_{k,i}\|_{\pi}\|\nabla^2\varphi(X_k)\|_{\e}\|\nabla_xf_i(r_{k,i},\e_{k+1})\|^2.
\end{split}\end{equation}
From Lemma \ref{lem:gradient_growth_bound}, we have $\EE[\|\nabla_xf_i(r_{k,i},\e_{k+1})\|^2|\cF_k]\leq CL(f)^2(1+\|K\|_{\op}^{1/2}\|W_k\|)^{\max\{1,2\a\}}$. Using Lemma \ref{lem:normequivalence}, we also have $\|\nabla^2\varphi(X_k)\|_{\e}\leq\|\nabla^2Z(W_k,z)\|_{\Omega}\leq\lp.$  We also note that $\|\Pi_{k,i}\|_{\pi}=\text{rank}(\Pi_{k,i})\leq\lp.$Putting it all together, we get
\begin{equation}\begin{split}
    \EE&\left[ \left|\ipa{\nabla^2\varphi(X_k),\sqrt{K_i}\Pi_{k,i}\sqrt{K_i}\otimes\nabla_xf_i(r_{k,i},\e_{k+1})^{\otimes2}}\right|\Big|\cF_k\right]\\
    &\leq (\lp)^2 CL^2(f)(1+\|K\|_{\op}^{1/2}\|W_k\|)^{\max\{1,2\a\}}=O((\lp)^3).
\end{split}\end{equation}
Similarly, for the second term of $E_{k,i}^{\Delta},$ we have
\begin{equation}\label{eq:mart_hesserror_term2bound}\begin{split}
    &\left|\ipa{\nabla^2\varphi(X_k),(\sqrt{K_i}\Pi_{k,i}v_{k+1})^{\otimes2}\otimes\nabla_xf_i(r_{k,i},\e_{k+1})^{\otimes2}}\right|\\
    &\leq \|\nabla^2\varphi(X_k)\|_{\e}\|\nabla_xf_i(r_{k,i},\e_{k+1})\|^2\|\sqrt{K_i}\Pi_{k,i}v_{k+1}\|^2\\
    &\leq \|\nabla^2\varphi(X_k)\|_{\e}\|\nabla_xf_i(r_{k,i},\e_{k+1})\|^2\|K_i\|_{\op}\|\Pi_{k,i}v_{k+1}\|^2.
\end{split}\end{equation}
We can bound $\|\nabla^2\varphi(X_k)\|_{\e}$ and $\EE[\|\nabla_xf_i(r_{k,i},\e_{k+1})\|^2|\cF_k]$ as before and, using the fact that $\Pi_{k,i}$ is a projection, we have $\EE[\|\Pi_{k,i}v_{k+1}\|^2|\cF_{k,i}]=\|\Pi_{k,i}\|^2\leq\lp.$ Then, applying Cauchy-Schwarz, we get the same bound as we did for the first term, namely,
\begin{equation}\begin{split}
    \EE&\left[ \left|\ipa{\nabla^2\varphi(X_k),(\sqrt{K_i}\Pi_{k,i}v_{k+1})^{\otimes2}\otimes\nabla_xf_i(r_{k,i},\e_{k+1})^{\otimes2}}\right|\Big|\cF_{k,i}\right]\\
    &\leq (\lp)^2 CL^2(f)(1+\|K\|_{\op}^{1/2}\|W_k\|)^{\max\{1,2\a\}}=O((\lp)^3).
\end{split}\end{equation}
The third and fourth terms of $E_{k,i}^{\Delta}$ can be bounded in a similar manner to the second term.  The only change is that, in \eqref{eq:mart_hesserror_term2bound}, the quantity $\|\sqrt{K_i}\Pi_{k,i}v_{k+1}\|^2$ is replaced with
$$\|\sqrt{K_i}\Pi_{k,i}v_{k+1}\|\cdot\|\mu_i\|\leq\|K_i\|_{\op}^{1/2}\|\Pi_{k,i}v_{k+1}\|\cdot\|\mu_i\|=O(\sqrt{\lp})$$
and we arrive at 
\begin{equation}
    \EE\left[ \left|\ipa{\nabla^2\varphi(X_k),\sqrt{K_i}\Pi_{k,i}v_{k+1}\mu_i^\top \otimes\nabla_xf_i(r_{k,i},\e_{k+1})^{\otimes2}}\right|\Big|\cF_{k,i}\right]
    =O((\lp)^{5/2}),
\end{equation}
and similarly for the last term of $E_{k,i}^{\Delta}.$ Since all of these bounds hold uniformly in $i$, they hold for the weighted average $\sum_{i=1}^{\ls}p_i(\cdot)$.  This gives us a bound of order $d^{-2}(\lp)^3$ for $\EE[\cE_k^{\hess}|\cF_k]$ in equation \eqref{eq:mart_hesserror_singlesummand}.  Summing over $k$, Lemma \ref{lem:Hess_error} is proved.
\end{proof}

\section{Proof of example: Binary logistic regression}\label{sec:proof_binarylogistic}

We begin by recalling the set-up of the binary case $\ls=2$.  By the symmetry of the problem, we simplify the analysis by taking $X=\left[\begin{array}{cc}
x &
0
\end{array}\right]$ with $x\in \R ^d. $  We then have that 
\begin{equation}
    f_i(r_i)=-y_ir_i+\log(1+\exp(r_i))\quad\text{where }r_{i}=x^\top a|(a\in\text{class }i),\quad y_i=\mathbf{1}_{\{i=1\}}.
\end{equation}
Furthermore, the $\w$ variables become
\[
\w_{11}=\frac{e^{r_1}}{e^{r_1}+1},\quad
\w_{12}=\frac{1}{e^{r_1}+1},\quad
\w_{21}=\frac{e^{r_2}}{e^{r_2}+1},\quad
\w_{22}=\frac{1}{e^{r_2}+1}.
\]
where we note that $\w_{11}+\w_{12}=\w_{21}+\w_{22}=1.$
Using this, the derivatives of $f_i$ are 
\begin{align}
    f'_i(r_i)=-y_i+\w_{i1} 
=\begin{cases}
    -\w_{12}& i=1,\\
    \w_{21}& i=2,
\end{cases}\qquad
    f''_i(r_i)
    = \w_{i1} \w_{i2}.
\end{align}
Following the above simplification, we obtain that, $\csm_{\rho,1} = - \csm_{\rho,2} = \csm_{\rho}  = d x^\top u_{\rho} u_{\rho}^\top  \mu  \in \R$, similarly we reduce $V_\rho  = \left[\begin{array}{cc}
\csV_\rho & 0 \\
0 & 0
\end{array}\right]$ with  $\csV_\rho  =d x^\top  u_\rho u_\rho^\top  x\in  \R$, and $r_i = x^\top  a_i.$

For the risk, we get 
\begin{align}
\csL(t)&= \EE[f_i(r_i)]=-p_1 \csm(t)+p_1\EE_z[\log(1+\exp(\csm(t)+\sqrt{\csB_1(t)}z))] \\\nonumber &+p_2\EE_z[\log(1+\exp(-\csm(t)+\sqrt{\csB_2(t)}z))]   .
\end{align}
with $\csm(t) = \ip{\mu, X}$, and  $\csB_i(t) = \ip{K_i, X^{\otimes 2}}.$ 
Thus, the system of equations \eqref{eq:V_m_rho_2} reduces to
\begin{align} \label{eq:V_m_rho_1}
    \frac{\dif \csV_\rho}{\dif t}
    &=-2\gamma \csV_\rho\sum_{i=1}^2 p_i \lambda_\rho^{(i)}\EE[w_{i1}w_{i2}]+ 2\gamma \csm_{\rho}(p_1\EE[w_{12}]+p_2\EE[w_{21}])+{\gamma^2}\sum_{i=1}^2 p_i\left(\lambda_\rho^{(i)}+(\ip{\mu,u_\rho})^2\right)\EE[w_{i(\sim i)}^2]\\ \nonumber
    \frac{\dif \csm_{\rho}}{\dif t}&=-\gamma \csm_{\rho}\sum_{i=1}^2 p_i\lambda_\rho^{(i)}\EE[w_{i1}w_{i2}]+ {\ip{\mu, u_\rho}^2}\gamma (p_1 \EE[w_{12}]+p_2 \EE[w_{21}])d
\end{align}
where $w_{ij}$ denotes the deterministic equivalent of $\w_{ij}$ and can be written as $w_{12} = 
\left(1+e^{\csm(t)+\sqrt{\csB_1(t)}z}\right)^{-1}$ and $w_{11}
= 1 -w_{12}$ where the expectation is with respect to $z\sim \mathcal{N}(0,1)$. Similarly, $w_{22} = 
\left(1+e^{-\csm(t)+\sqrt{\csB_2(t)}z}\right)^{-1}$ and, $w_{21}=1-w_{22}$. We note that in the identity covariance case $w_{12}\overset{d}{=}w_{22}, $ and $w_{21}\overset{d}{=}w_{11}.$ Finally, we use $\sim i$ to denote ``not $i$'' (so $w_{i(\sim i)}=w_{12},w_{21}$ for $i=1,2$ respectively).

\paragraph{Bounds on $w_{ij}$}
It will be useful to have upper and lower bounds for $\EE w_{12}$ and $\EE w_{21}$, which we provide in the following lemma.
\begin{lemma}\label{lem:w12bounds}
    If $\csm(t)\geq0$, then
    \begin{equation}
        \frac{1}{1+e^{\csm(t)}}\leq\EE w_{12}\leq\min\left\{\frac12,\frac{2}{3+e^{\csm(t)-\csB_1(t)/2}}\right\}.
    \end{equation}
    Likewise, if $\csm(t)\leq0$, then
    $ \max\left\{\frac12,\frac{2}{3+e^{\csm(t)-\csB_1(t)/2}}\right\}\leq\EE w_{12}\leq \frac{1}{1+e^{\csm(t)}}.$
    Finally, the same bounds hold for $\EE w_{21}$ if $\csB_1(t)$ is replaced by $\csB_2(t)$.
\end{lemma}

\begin{proof}
We will work under the assumption that $\csm(t)\geq0$ and prove the bounds for $\EE w_{12}$.  By reversing signs in the argument, one gets the case for $\csm(t)\leq0$ and the proof is similar for $\EE w_{21}.$  Using the shorthand $b=\sqrt{\csB_1(t)},\;m=\csm(t)$, we get
\begin{align}
    &\EE w_{12}=\frac{1}{\sqrt{2\pi}}\int_{-\infty}^\infty\frac{e^{-x^2/2}}{1+e^{bx+m}}\dif x\\
    &=\frac{1}{\sqrt{2\pi}}\int_0^\infty e^{-\frac{x^2}{2}}\left(\frac{1}{1+e^{bx+m}}+\frac{1}{1+e^{-bx+m}}\right)\dif x\\
    &=\frac{1}{1+e^m}+\frac{1}{\sqrt{2\pi}}\int_0^\infty e^{-\frac{x^2}{2}}\left(\left(\frac{1}{1+e^{-bx+m}}-\frac{1}{1+e^m}\right)-\left(\frac{1}{1+e^m}-\frac{1}{1+e^{bx+m}}\right)\right)\dif x\\
    &=\frac{1}{1+e^m}+\frac{1}{\sqrt{2\pi}}\int_0^\infty e^{-\frac{x^2}{2}}\frac{e^m}{1+e^m}(1-e^{-bx})\left(\frac{1}{1+e^{-bx+m}}-\frac{1}{e^{-bx}+e^m}\right)\dif x.
\end{align}
To obtain the lower bound, it suffices to show that the integrand in the last line is positive, so we need to show that $1+e^{-bx+m}\leq e^{-bx}+e^m$ for all $x\geq0$.  This follows from the fact that these two quantities are equal for $x=0$ and $\frac{\dif}{\dif x}(1+e^{-bx+m})\leq \frac{\dif}{\dif x}(e^{-bx}+e^m)$.

For the upper bound, we have
\beq\begin{split}
\EE w_{12}=\EE\frac{1}{1+e^{bx+m}}=\EE\frac{1}{1+e^{-bx+m}}&=\frac12\EE\left(\frac{1}{1+e^{bx+m}}+\frac{1}{1+e^{-bx+m}}\right)\\
&=\frac12\EE\left(1-\frac{e^{2m}-1}{e^{2m}+1+e^m(e^{bx}+e^{-bx})}\right).
\end{split}\eeq
By Jensen's inequality, $\EE\frac{1}{C+Y}\geq\frac{1}{C+\EE Y}$ for any constant $C>0$ and random $Y>0$,.  Thus,
\beq
\EE w_{12}\leq\frac12\left( 1-\frac{e^{2m}-1}{e^{2m}+1+2e^{m+b^2/2}}\right)
=\frac{1+e^{m+b^2/2}}{1+2e^{m+b^2/2}+e^{2m}}.
\eeq
 For all $m\geq0$, this is bounded above by the value at $m=0$, which yields $\EE w_{12}\leq\frac12$.  To obtain a bound depending on $m,b$, we write
\beq\begin{split}
\EE w_{12}\leq\frac{e^{-m}+e^{b^2/2}}{e^{-m}+2e^{b^2/2}+e^m}
&\leq\frac{1+e^{b^2/2}}{1+2e^{b^2/2}+e^m}
=\frac{1}{1+e^m\left(\frac{1+e^{b^2/2-m}}{1+e^{b^2/2}}\right)}\\
&\leq\frac{1}{1+\frac12 e^m(e^{-b^2/2}+e^{-m})}
=\frac{2}{3+e^{m-b^2/2}}.
\end{split}\eeq
\end{proof}

The next assumption controls the distribution of the logistic weights along
the limiting trajectory.
\begin{assumption}[Logistic weight concentration]\label{ass:W_1W_2ab}
    Suppose there 
        exists $C_w(\gamma)\ge 1$  such that\;
        $a(t)\defas \dfrac{\EE w_{12}}{\EE w_{12}-\EE w_{12}^2}<C_w(\gamma)$
\end{assumption}

\begin{figure}[t]
\centering

\begin{minipage}[t]{0.3\textwidth}
    \centering
    \includegraphics[width=\linewidth,keepaspectratio]{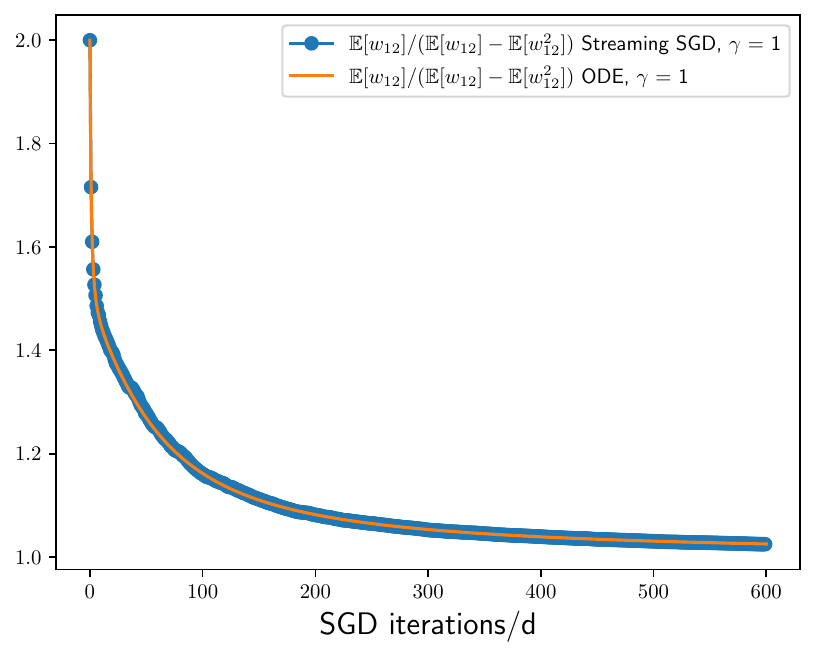}
    \medskip
    \textit{(a) Zero-one model}
\end{minipage}\hspace{0.04\textwidth}%
\begin{minipage}[t]{0.3\textwidth}
    \centering
    \includegraphics[width=\linewidth,keepaspectratio]{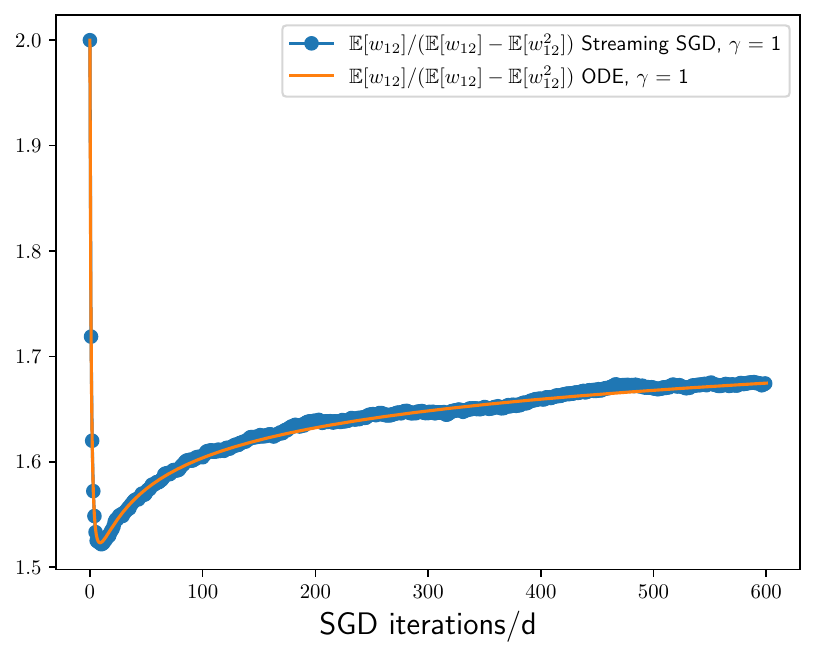}
    \medskip
    \textit{(b) Power-law model with $\alpha_1=1.3,\, \alpha_2=1.3,\, \beta=0.5$}
\end{minipage}

\vspace{0.8em} 

\begin{minipage}[t]{0.3\textwidth}
    \centering
    \includegraphics[width=\linewidth,keepaspectratio]{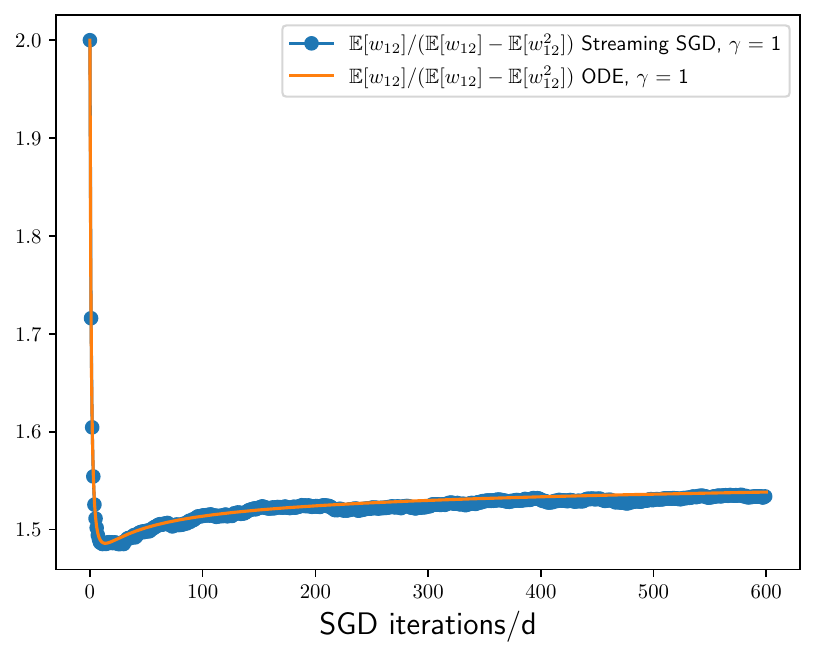}
    \medskip
    \textit{(c) Power-law model with $\alpha_1=1.5,\, \alpha_2=1.5,\, \beta=0.5$}
\end{minipage}\hspace{0.04\textwidth}%
\begin{minipage}[t]{0.3\textwidth}
    \centering
    \includegraphics[width=\linewidth,keepaspectratio]{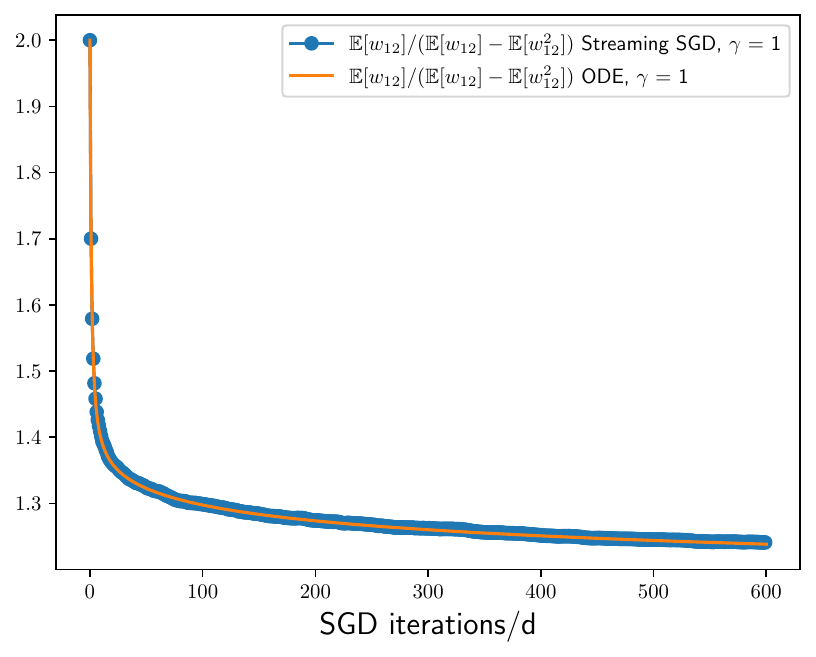}
    \medskip
    \textit{(d) Power-law model with $\alpha_1=1.5,\, \alpha_2=1.5,\, \beta=0$}
\end{minipage}
\vspace{0.1cm}

\caption{\textbf{Visualization of $a(t)$.} Comparison of $a(t)\defas\mathbb{E}[w_{12}]/(\mathbb{E}[w_{12}] - \mathbb{E}[w_{12}^2])$ between Streaming SGD and ODE for various data models with $\gamma=1$.
}
\label{fig:combined_w12_plots}
\vspace{-0.4cm}
\end{figure}

\begin{remark}
We note that  $$a(t) = \frac{W_1(t)}{W_1(t)-W_2(t)}\ge\frac{1}{1-W_1(t)}\ge\frac{1+e^{\csm(t)}}{e^{\csm(t)}}= 1+e^{-\csm(t)}$$ 
using the bound in Lemma \ref{lem:w12bounds}. 
This inequality together with Assumption \ref{ass:W_1W_2ab}(b) implies $\csm(t)\ge -\log(1-\e)$ for $t>t_0$. This bound of the overlap away from zero corresponds to some weak recovery of the true direction $\mu$.
\end{remark}

\paragraph{Heuristic/numeric support for Assumption \ref{ass:W_1W_2ab}}

The assumption says that the ratio $a(t)=\frac{\EE w_{12}}{\EE w_{12}-\EE w_{12}^2}$ remains bounded as $t$ grows.  This assumption serves as a convenient tool for analyzing loss trajectories in the identity, zero-one, and power-law models.  However, we suspect that it actually holds generically in these models (i.e., there should be a way to prove it rather than assuming it).  Numerical simulations suggest the following behavior for the ratio $a(t)$:
\begin{itemize}
    \item \textit{Identity model:} $a(t)$ converges to some $C(\gamma)$, which is typically at least 2 and is larger for larger $\gamma$ (see Figure \ref{fig:combined_plots_gamma_ID_a_loss}).
     \item \textit{Zero-one:} $a(t)$ converges to 1 and the rate of convergence depends on $\gamma$ (see Figure \ref{fig:combined_w12_plots}). 
    \item \textit{Power-law:} $a(t)$ converges to some value between 1 and 2 (see Figure \ref{fig:combined_w12_plots}).
\end{itemize}
By way of a heuristic explanation for this, we recall that, in the context of our original optimization problem,
\[ w_{12}=- f_1'(r_1)=\frac{1}{1+e^{r_1}},\qquad
w_{12}-w_{12}^2=f_1''(r_1)=\frac{e^{r_1}}{(1+e^{r_1})^2}
 \]
where $r_1$ is the inner product of the parameter vector $X$ with a random data point from class 1.  Based on this, $f$ should be optimized when $r_1$ tends toward infinity (which can only happen if $\|X\|$ is unbounded and grows in a direction that has positive overlap with $\mu$).  Then, things roughly break into two cases:
\begin{enumerate}[(C{a}se 1)]
    \item Large variance in the $\mu$ direction causes $X$ to remain bounded and thus $r_1,\EE w_{12}, a(t)$ all remain bounded as well.
    \item There is a direction of zero variance (or small variance) that has positive overlap with $\mu$ and $X$ grows unboundedly in that direction.  Because the variance is sufficiently small in this direction, $\EE w_{12}\to0$ with $\EE w_{12}^2$ shrinking faster, so $a(t)$ remains bounded.
\end{enumerate}
Indeed, Case 1 appears consistent with the behavior in the identity and mild power-law models (where $\csm,\csV$ remain bounded), whereas Case 2 appears consistent with the behavior in the zero-one and extreme power-law models (where $\csm,\csV$ grow unboundedly).

\subsection{Proofs for the zero-one model (Proposition \ref{prop:zero-one})}\label{sec:proof_zero-one}
 In this subsection, we prove Proposition \ref{prop:zero-one}.  We begin by recalling set-up, introducing some additional notations, and proving some preliminary bounds.  Lemmas \ref{lem:m_ij_bounds} and \ref{lem:vij_asymp} will be proved for the general zero-one model.  For Lemma \ref{lem:zero-one_mij_final}, we impose symmetry on the classes as well as Assumption \ref{ass:W_1W_2ab}. At the end of the subsection we will bring these results together to prove Proposition \ref{prop:zero-one}.
Recall that, in the zero-one model, all eigenvectors of $K_1$ and $K_2$ are either 0 or 1 and the index set $[d]$ is partitioned as
\begin{equation}
    [d]=I_{00}\cup I_{01}\cup I_{10}\cup I_{11},\quad\text{where }I_{jk}=\{i\leq d\;|\;\lambda^{(1)}_i=j,\;\lambda^{(2)}_i=k\}.
\end{equation}
We denote the magnitude of $\mu$ in the four key eigenspaces by
\begin{equation}
    \tilde{\mu}_{(ij)}=\sum_{\rho\in I_{ij}}\tilde{\mu}_\rho=\sum_{\rho\in I_{ij}}(\mu^\top u_\rho)^2,\qquad
    \text{thus }\tilde{\mu}_{(01)}+\tilde{\mu}_{(01)}+\tilde{\mu}_{(10)}+\tilde{\mu}_{(11)}=\|\mu\|^2.
\end{equation}
We will eventually take $\tilde{\mu}_{(ij)}\to\frac14$ as $d\to\infty$,  but we initially assume only that there exist positive constants with $c_{00}+c_{01}+c_{10}+c_{11}=1$ such that, for $i,j\in\{0,1\},$ we have 
\begin{equation}
    \lim_{d\to\infty}\frac{|I_{ij}|}{d}=c_{ij},\qquad \text{which implies }\lim_{d\to\infty}\frac{\tilde{\mu}_{(ij)}}{\|\mu\|^2}=c_{ij}.
\end{equation}
One example that falls into this set-up (with high probability) is standard Gaussian $\mu$, in which case $c_{ij}=\frac14$. We define the parts of $\csm(t)$ and $\csV(t)$ associated with these four eigenspaces as
\begin{equation}
    m_{(ij)}(t)=\frac1d\sum_{\rho\in I_{ij}}\csm_\rho(t),\qquad v_{(ij)}(t)=\frac1d\sum_{\rho\in I_{ij}}\csV_\rho(t),
\end{equation}
so we have $m_{(00)}(t)+m_{(01)}(t)+m_{(10)}(t)+m_{(11)}(t)=\csm(t)$ and likewise for $\csV(t)$.  Using the integral equations for $\csm_\rho,\csV_\rho$, and the notation $W_1(t)=p_1\EE w_{12}(t)+p_2\EE w_{21}(t)$, we get
\begin{equation}\begin{split}
    m_{(00)}(t)&=\gamma\tilde{\mu}_{(00)}\int_0^tW_1(s)\dif s\\
    m_{(01)}(t)&=\gamma\tilde{\mu}_{(01)}\int_0^tW_1(s)\exp\left(-\gamma \int_s^tp_2(\EE w_{21}-\EE w_{21}^2)\dif\tau\right)\dif s\\
    m_{(10)}(t)&=\gamma\tilde{\mu}_{(10)}\int_0^tW_1(s)\exp\left(-\gamma \int_s^tp_1(\EE w_{12}-\EE w_{12}^2)\dif\tau\right)\dif s\\
    m_{(11)}(t)&=\gamma\tilde{\mu}_{(11)}\int_0^tW_1(s)\exp\left(-\gamma \int_s^t\Big(p_1(\EE w_{12}-\EE w_{12}^2)+p_2(\EE w_{21}-\EE w_{21}^2)\Big)\dif\tau\right)\dif s.\\
\end{split}\end{equation}
Since $\EE w_{12}(\tau)-\EE w_{12}^2(\tau)\geq0$ for all $\tau$, and likewise for $w_{21}$, the integral equations imply that, for all $t$,
\begin{equation}\label{eq:mij_relations}
    \frac{m_{(00)}(t)}{\tilde{\mu}_{(00)}}\geq \frac{m_{(01)}(t)}{\tilde{\mu}_{(01)}},\frac{m_{(10)}(t)}{\tilde{\mu}_{(10)}}\geq \frac{m_{(11)}(t)}{\tilde{\mu}_{(11)}}.
\end{equation}
Likewise, we see that, for $d$ sufficiently large, $m_{(00)}(t)\geq c_{00}m(t)$.  Using the information above, we arrive at the following bounds (We obtain tighter bounds in a later lemma, subject to a symmetry assumption on the two classes and an upper bound on the learning rate).
\begin{lemma}\label{lem:m_ij_bounds} 
    The quantity $m_{00}(t)$ is unbounded with $t$.  More specifically, it has upper and lower bounds
    \begin{equation}
        m_{(00)}(t)=O(t),\qquad m_{(00)}(t)=\Omega(\log t).
    \end{equation}
    Furthermore, for the other $m_{(ij)}(t)$ bounds (where $i,j$ are not both zero), we have
    \begin{equation}
        m_{(ij)}(t)=O(m_{(00)}(t))=O(t),\qquad m_{(ij)}(t)=\Omega(1).
    \end{equation}
\end{lemma}
\begin{proof}
    Initializing at $X_0=0$, one can deduce from the differential equations that $\csm(t)\geq0$ for all $t$.  Thus, applying Lemma \ref{lem:w12bounds}, we see that, for all $t$, 
    \begin{equation}\label{eq:W1_bds}
        \frac{1}{1+e^{\csm(t)}}\leq w_{12}(t),w_{21}(t)\leq\frac12,\quad\text{so }\frac{1}{1+e^{\csm(t)}}\leq W_1(t)\leq\frac12 .
    \end{equation}
 For the upper bound, this gives us
\begin{equation}
    m_{(00)}(t)\leq\gamma\tilde{\mu}_{(00)}\int_0^t\frac12\dif s=\frac\gamma2\tilde{\mu}_{(00)}t=O(t).
\end{equation}
For the lower bound, we use the observation that $m_{(00)}(t)\geq c_{00}\csm(t)$ to write
\begin{equation}\label{eq:m00intermed-bd}
    m_{(00)}(t)\geq\gamma\tilde{\mu}_{(00)}\int_0^t\frac{1}{1+e^{m_{(00)}(s)/c_{00}}}\dif s.
\end{equation}
Now consider the corresponding integral equation
\begin{equation}
y(t)=\gamma\tilde{\mu}_{(00)}\int_0^t\frac{1}{1+e^{y(s)/c_{00}}}\dif s.
\end{equation}
Differentiating both sides and solving the resulting equation with initial condition $y(0)=0$ gives
    $t=\frac{1}{\gamma\tilde{\mu}_{(00)}}(y+c_{00}(e^{y/c_{00}}-1))$.
This implies that the inequality \eqref{eq:m00intermed-bd} is satisfied for 
\begin{equation}
    \frac{1}{\gamma\tilde{\mu}_{(00)}}\left(m_{(00)}(t)+c_{00}(e^{m_{(00)}(t)/c_{00}}-1)\right)\geq t,\quad\text{thus }m_{(00)}(t)=\Omega(\log t).
\end{equation}

For the bounds on the other $m_{(ij)}(t)$ values, the upper bound comes directly from the integral equations by observing that
    \begin{equation}
        m_{(01)}(t)\leq \frac{c_{01}}{c_{00}} m_{(00)}(t)
    \end{equation}
and similarly for $m_{(10)},m_{(11)}$.  For the lower bound we introduce the notation $\omega_1(s)=\int_0^sW_1(\tau)\dif\tau$ and we get
\begin{equation}\begin{split}
     m_{(01)}(t)&\geq\gamma\tilde{\mu}_{(01)}\int_0^t W_1(s)\exp\left(-\gamma\int_s^t W_1(\tau)\dif\tau\right)\dif s\\
     &=\gamma\tilde{\mu}_{(01)}\int_0^{\omega_1(t)}\exp(-\gamma(\omega_1(t)-u))\dif u\\
     &=\tilde{\mu}_{(01)}(1-e^{-\gamma\omega_1(t)}).
\end{split}\end{equation}
Using the lower bound on $W_1$ given in \eqref{eq:W1_bds} and the fact that $\csm(t)\leq m_{(00)}(t)/c_{00}\leq\frac{\gamma\tilde{\mu}_{(00)}}{2c_{00}}t$, we see that $\omega_1(t)=\int_0^tW_1(\tau)\dif\tau\geq\frac{1}{1+\exp(\frac{\gamma\tilde{\mu}_{(00)}}{2c_{00}})}$ for $t\geq1$. Thus, \[m_{(01)}(t)\geq\tilde{\mu}_{(01)}\left(1-\exp\left(-\frac12\gamma\left(1+\exp(\frac{\gamma\tilde{\mu}_{(00)}}{2c_{00}})\right)^{-1}\right)\right).\]
 For any fixed $\gamma$, this is bounded away from zero.  The lower bounds for $m_{(10)}$, $m_{(11)}$ are similar, replacing $\tilde{\mu}_{(01)}$ by $\tilde{\mu}_{(10)}$, $\tilde{\mu}_{(11)}$, respectively.
\end{proof}

\begin{lemma}\label{lem:vij_asymp}
The asymptotics for $v_{(ij)}$ can be expressed in terms $m_{(ij)}$. In particular, for $i,j\in\{0,1\}$,
\begin{equation}\label{eq:v00_asymp}
    v_{(ij)}(t)\asymp m_{(ij)}^2(t)
\quad\text{ and, for $i=j=0,$ }\quad
    v_{(00)}(t)=\frac{m_{(00)}^2(t)}{\tilde{\mu}_{(00)}}\left(1+O((\log t)^{-1})\right).
\end{equation}
\end{lemma}

\begin{proof}
As with the $m_{(ij)}$ bounds above, we have the following bounds for $v_{(ij)}$, where we denote by $W_2 = p_1\EE w_{12}^2(t)+p_2\EE w_{21}^2(t)$: 
\begin{equation}\begin{split}
    v_{(00)}(t)=\gamma\int_0^t&\left(2m_{(00)}(s)W_1(s)+\gamma\tilde{\mu}_{(00)}W_2(s)\right)\dif s,\\
    v_{(01)}(t)=\gamma\int_0^t&\left(2m_{(01)}(s)W_1(s)+\gamma(p_2\EE w_{21}^2(s)+\tilde{\mu}_{(01)}W_2(s))\right)\\
    &\exp\left(-2\gamma\int_s^tp_2\EE(w_{21}(\tau)-w_{21}^2(\tau))\dif\tau\right)\dif s,\\
    v_{(10)}(t)=\gamma\int_0^t&\left(2m_{(10)}(s)W_1(s)+\gamma(p_1\EE w_{12}^2(s)+\tilde{\mu}_{(10)}W_2(s))\right)\\
    &\exp\left(-2\gamma\int_s^tp_1\EE(w_{12}(\tau)-w_{12}^2(\tau))\dif\tau\right)\dif s,\\
    v_{(11)}(t)=\gamma\int_0^t&\left(2m_{(11)}(s)W_1(s)+\gamma(1+\tilde{\mu}_{(11)})W_2(s)\right)
    \exp\left(-2\gamma\int_s^t(W_1(\tau)-W_2(\tau))\dif\tau\right)\dif s,\\
\end{split}\end{equation}
For $v_{(00)}(t)$, we have a lower bound
\begin{equation}\label{eq:v00_lowerbd}\begin{split}
    v_{(00)}(t)&\geq2\gamma\int_0^tW_1(s)m_{(00)}(s)\dif s
    \;=\;2\gamma^2\tilde{\mu}_{(00)}\iint_{0\leq\tau\leq s\leq t}W_1(\tau)W_1(s)\dif\tau\dif s\\
    &=\gamma^2\tilde{\mu}_{(00)}\left(\int_0^t W_1(s)\dif s\right)^2\;\;
    =\;\;\frac{m_{(00)}^2(t)}{\tilde{\mu}_{(00)}}.
\end{split}\end{equation}
On the other hand, we have an upper bound
\begin{equation}\begin{split}
    v_{(00)}\leq 2\gamma\int_0^tW_1(s)\left(m_{(00)}(s)+\gamma\tilde{\mu}_{(00)}\right)\dif s\;
    =\;\frac{m_{(00)}^2(t)}{\tilde{\mu}_{(00)}}+2\gamma m_{(00)}(t).
\end{split}\end{equation}
Combining the upper and lower bounds, and applying Lemma \ref{lem:m_ij_bounds}, we get
\begin{equation}
    v_{(00)}(t)=\frac{m_{(00)}^2(t)}{\tilde{\mu}_{(00)}}\left(1+O((\log t)^{-1})\right).
\end{equation}
Similarly, for $v_{(01)}(t)$ with sufficiently large $t$, since $m_{(01)}(t)=\Omega(1)$, there exists some $C>0$, which may depend on the learning rate, such that 
\begin{equation}\label{eq:v01_asymp}\begin{split}
    &v_{(01)}(t)\leq2C\gamma\int_0^tW_1(s)\exp\left(-2\gamma\int_s^tp_2\EE(w_{21}-w_{21}^2)\dif\tau\right)m_{(01)}(s)\dif s\\
    &=2C\gamma^2\tilde{\mu}_{(01)}\int_0^tW_1(s)\exp\left(-2\gamma\int_s^tp_2\EE(w_{21}-w_{21}^2)\dif\tau\right)\int_0^sW_1(u)\exp\left(-\gamma\int_u^sp_2\EE(w_{21}-w_{21}^2)\dif\tau\right)\dif u\dif s \\
    &= \gamma^2C\tilde{\mu}_{(01)}\left(\int_0^tW_1(s)\exp\left(-\gamma\int_s^tp_2\EE(w_{21}-w_{21}^2)\dif\tau\right)\dif s\right)^2\\
    &=\frac{C}{\tilde{\mu}_{(01)}}m_{(01)}^2(t).
\end{split}\end{equation}
Furthermore, note that we also have the lower bound $v_{(01)}(t)\geq\frac{m_{(01)}^2(t)}{\tilde{\mu}_{(01)}}$ using the same computation as in \eqref{eq:v01_asymp} but with the inequality in the first line reversed and $C$ replaced by 1.
and likewise for the other quantities, we get (possibly with different $C$ values)
\begin{equation}\label{eq:v10-11_asymp}
   \frac{m_{(10)}^2(t)}{\tilde{\mu}_{(10)}}\leq v_{(10)}(t)\leq \frac{C m_{(10)}^2(t)}{\tilde{\mu}_{(10)}},\qquad
    \frac{m_{(11)}^2(t)}{\tilde{\mu}_{(11)}}\leq v_{(11)}(t)\leq \frac{Cm_{(11)}^2(t)}{\tilde{\mu}_{(11)}}.
\end{equation}
\end{proof}

\begin{lemma}\label{lem:zero-one_mij_final}
    If we impose the additional assumption of symmetry, meaning that $p_1=p_2=1/2$ and $|I_{00}|=|I_{01}|=|I_{10}|=|I_{11}|$ and we take $\tilde{\mu}_{(00)}=\tilde{\mu}_{(01)}=\tilde{\mu}_{(10)}=\tilde{\mu}_{(11)}=\frac14$ and further impose Assumption \ref{ass:W_1W_2ab}, then we get tighter bounds on $m_{(ij)}(t)$ and $W_1(t)$.  In particular, we get that
    \[m_{(01)},m_{(10)},m_{(11)}\asymp1,\qquad m_{(00)} = \log t +O(1),\qquad W_1(t)\asymp t^{-1}.
    \]

\end{lemma}
\begin{proof}
We begin with the differential equation and initial condition
\begin{equation}
    \frac{\dif}{\dif t}m_{(10)}(t)=-\frac\gamma2 m_{(10)}(t)(W_1(t)-W_2(t))+\frac\gamma4W_1(t),\qquad m_{(10)}(0)=0.
\end{equation}
Using Assumption \ref{ass:W_1W_2ab} and continuity, we conclude that
$m_{(10)}\leq2C_w\tilde{\mu}_{(10)}.$
Combining this with the lower bound on $m_{(10)}$ from Lemma \ref{lem:m_ij_bounds}, and applying a similar argument for $m_{(01)},m_{(11)},$ we conclude that
    $m_{(01)},m_{(10)},m_{(11)}\asymp1.$
    
Based on the upper bound of $W_1$ from Lemma \ref{lem:w12bounds}, it will be useful to show that $\csB_1(t)=O(1)$ (and by symmetry it will also hold for $\csB_2$).  We have
\begin{equation}\label{eq:B1bound0-1}
    \csB_1=v_{(10)}+v_{(11)}=O((m_{(10)}^2+m_{(11)}^2))=O(m_{(10)}^2)
\end{equation}
This bound, along with the bounds on $m_{(ij)}$, imply that 
\begin{equation}\label{eq:zero-one_m-B_lb}
    \csm(t)=m_{(00)}(t)+O(1),\qquad \csB_1(t),\csB_2(t)=O(1).
\end{equation} 
Thus, for large $t$, Lemma \ref{lem:w12bounds} implies that
$W_1=e^{-m_{(00)}(t)+O(1)}.$
Combining this with the differential equation $\frac{\dif m_{(00)}}{\dif t}=\frac\gamma4W_1,$ we get
$ \frac{\dif m_{(00)}}{\dif t}=e^{-m_{(00)}(t)+O(1)},$
 and from this we conclude that
 \begin{equation}
     m_{(00)}(t)=\log t+O(1),\qquad W_1(t)\asymp t^{-1}.
 \end{equation} 
\end{proof}
\begin{proof}[Proof of Proposition \ref{prop:zero-one}]
    
     \textit{Part \ref{prop:zero-one_loss}}: Lemma \ref{lem:risk_lub} says that, in the case of $\csB_1=\csB_2=\csB$ (which holds in the symmetric zero-one model), the loss is bounded by
\begin{equation}
    \log(1+e^{-\csm(t)})\leq\csL(t)\leq\log(1+e^{-\csm(t)+\csB(t)/2}).
\end{equation}
Furthermore, Lemma \ref{lem:zero-one_mij_final} and line \eqref{eq:zero-one_m-B_lb} of its proof imply that, for sufficiently large $t$,
    $\csL(t)\asymp t^{-1}.$

    \textit{Part \ref{prop:zero-one_m/v2}}: The asymptotic for $v_{(00)}$ in equation \eqref{eq:v00_asymp} tells us that, for sufficiently large $t$ and $\tilde{\mu}_{(00)}=\frac14$, we have
    \begin{equation}
        \frac{m_{(00)}(t)}{\sqrt{v_{(00)}(t)}}=\frac12\left(1+O((\log t)^{-1})\right)).
    \end{equation}
    From the asymptotics for $m_{(ij)}(t)$ in Lemma \ref{lem:zero-one_mij_final} and the asymptotics for $v_{(ij)}(t)$ in equations \eqref{eq:v00_asymp}, \eqref{eq:v01_asymp}, \eqref{eq:v10-11_asymp}, we see that
    \begin{equation}
        \csm(t)=m_{(00)}(t)(1+O((\log t)^{-1})),\qquad \csV(t)=v_{(00)}(t)(1+O((\log t)^{-2})),
    \end{equation}
    and this implies Part \ref{prop:zero-one_m/v2} of the Proposition.

    \textit{Part \ref{prop:zero-one_mij}}: Direct from Lemma \ref{lem:zero-one_mij_final}.
\end{proof}

\subsection{Proofs for the power-law model \label{sec:power_law_proofs}}

In this section, we prove 
Propositions \ref{prop:power_law_good_regime} and \ref{prop:power_law_bad_regime}.  
We prove them under a milder assumption on the eigenvalues and mean. For this purpose, we present the following kernels:
$$F_{\mu}(x)\defas \sum_{\rho}\tilde{\mu}_{\rho}e^{-\gamma\lambda_{\rho}x}\quad \text{and} \quad \cK_2(x) \defas \frac{1}{d}\sum_{\rho}\lambda_{\rho}^2e^{-2\gamma\lambda_{\rho}x}.
$$ 
We require these kernels to satisfy the following assumption:
\begin{assumption}
\label{ass:F_K_power_law_assmp} $F_\mu(x)\asymp x^{-\kappa_\mu}$ and $\mathcal{K}_2(x)\asymp x^{-\kappa_2}$ for $x\geq1$ with $\kappa_\mu\geq0,$  $\kappa_2> 1.$  
\end{assumption}
Lemma \ref{lem:K_mu_k_map_a_b} shows that Assumption \ref{ass:power_law_mu_lam_con} (power-law set-up) satisfies this assumption with $\kappa_\mu = \frac{\beta+1}{\alpha},$ and $\kappa_2 = \frac{1}{\alpha} +2.$
In addition, in the regime where $F_\mu$ is integrable, i.e., $\kappa_\mu >1$, or in the identity case, taking $\|\cdot\|_1$ to denote integration on $(0,\infty)$, we have
$$\|F_\mu\|_1 = \frac{1}{\gamma }\mu^\top K^{-1} \mu \quad \text{for} \quad \kappa_\mu >1 \,\text{or} \, K_1=K_2=I_d.$$
When $\kappa_\mu\le 1,$ $ F_\mu$ is not integrable. 
Taking $\tilde{\mu}_\rho,$ and $K$ as in Assumption \ref{ass:power_law_mu_lam_con}, we obtain $\|F_\mu\|_1 = \frac{1}{\gamma}\sum_{\rho}\frac{\tilde{\mu}_\rho}{\lambda_\rho} = \frac{1}{\gamma (\beta-\alpha +1)}+O(d^{-1})$ when $\frac{\beta+1}{\a}=\kappa_\mu >1$ and $\|F_\mu\|_1=\frac{\|\mu\|^2}{\gamma }$ when $K_1=K_2 =I_d$. Similarly $\|\cK_2\|_1= \frac{1}{2d \gamma}\tr(K).$

As in the zero-one section \ref{sec:proof_zero-one}, we will use the notation 
\[
W_p \defas \EE[w_{12}^p] \quad \text{and} \quad \omega_p(t) \defas \int_0^t\EE[w_{12}^p](s)\dif s \quad \text{for} \quad p=\{1,2\}.
\]
For simplicity, we assume  $K_1=K_2$ and $p_1 = \frac{1}{2}.$  Using symmetry of $w_{12}$ and $w_{21}$, we rewrite the equations for $ \csm_{\rho}, \csV_{\rho}$ for convenience in this notation: 
\begin{align}\label{eq:v_rho_one_K}
\frac{\dif \csV_{\rho}}{\dif t}&=-2\gamma\lambda_{\rho}\csV_{\rho}(W_{1}-W_{2})+2\gamma \csm_{\rho}W_{1}+\gamma^{2}\lambda_{\rho}W_{2}+\gamma^{2}\tilde{\mu}_{\rho}W_{2}\\
\label{eq:m_rho_one_K}
\frac{\dif \csm_{\rho}}{\dif t}&=-\gamma \csm_{\rho}\lambda_{\rho}(W_{1}-W_{2})+\gamma d\tilde{\mu}_{\rho}W_{1}.
\end{align}

\begin{proof}[Proof of Proposition \ref{prop:power_law_good_regime}]
 Throughout the proof, we assume without loss of generality that $\csm(t)\ge0$
(otherwise, we flip the direction of $\mu$). We start by analyzing the overlap (first statement of the Proposition). 
Taking the sum over $\rho$, in Eq. \eqref{eq:m_rho_lb_final} (in Lemma \ref{lem:ulb_m_v_rho}): 
\[
\csm(t) = \frac{1}{d}\sum_{\rho=1}^d \csm_\rho(t)\ge \gamma\sum_{\rho}\tilde{\mu}_{\rho}\int_{0}^{\omega_{1}(t)}e^{-\gamma\lambda_{\rho}\left(\omega_{1}(t)-u\right)}\dif u.
\]
By the power-law property of the kernel
$F_{\mu}(x)\asymp x^{-\kappa_{\mu}},$
\begin{align}
\gamma \sum_{\rho} \tilde{\mu}_{\rho}\int_{0}^{u} e^{-\gamma\lambda_{\rho}(u-v)} \dif v
&= \gamma \sum_{\rho} \tilde{\mu}_{\rho}\int_{0}^{u}e^{-\gamma\lambda_{\rho}x} \dif x
\\\nonumber &= \gamma \int_0^u F_\mu(x)\dif x =  
\gamma(\|F_\mu\|_1 - \int_u^\infty F_\mu(x)\dif x ).
\end{align}
We then obtain that for some constant $c>0$ for $\kappa_\mu >1$
\begin{align}\label{eq:m_lb}
\csm(t)\geq \gamma\|F_\mu\|_1 - c\gamma\omega_{1}(t)^{-\kappa_{\mu}+1}   
\end{align}
and $\csm(t)\geq \|\mu\|^2-e^{-\gamma \omega_1(t)},$ when $K_1=K_2 =I_d.$ Next, we derive an upper bound for $\csm(t).$ 
Using Assumption \ref{ass:W_1W_2ab},
we have that
\begin{align}\label{eq:w_b_max_bound}
 \int_{0}^{t} (\EE[w_{12}](\tau)- \EE[w_{12}^2](\tau)) \dif \tau 
 \ge  
 \frac{\omega_1(t)}{C_w}. 
\end{align}
Using Lemma \ref{lem:ulb_m_v_rho} with $\bar{c}(t) = 1/
C_w
$ and $t_0=0$
in particular averaging over $\rho$ in Eq. \eqref{eq:m_rho_ub_final} 
: 
\begin{align}
\csm(t)\le\gamma\int_{0}^{\omega_1(t)}F_{\mu}\left(\frac{\omega_1(t)-v_{1}}{C_w}\right)dv_{1}= C_w\gamma\int_{0}^{\omega_1(t)/C_w}F_{\mu}\left(y\right)\dif y.      
\end{align}

It then follows that,  
\begin{align}\label{eq:m_ub}
\csm(t)\leq {C_w}{\gamma }\left( \|F_\mu\|_1 - g_F(\omega_1(t);C_w)\right) \le  {C_w}\gamma  \|F_\mu\|_1    
\end{align}
with $g_F(\omega_1(t);C_w) = \left(\frac{\omega_{1}(t)}{C_w}\right)^{-\kappa_{\mu}+1}$ when $\kappa_\mu>1$ and $g_F(\omega_1(t);C_w) =e^{-\gamma \frac{\omega_{1}(t)}{C_w}}$ when $K_1=K_2 =I_d$. Thus we have $\csm(t)=O(1)$ which, by Lemma \ref{lem:w12bounds}, implies $\EE w_{12}=\Omega(1)$ and thus $\omega_1(t)\asymp t$.
Hence, the upper bound above matches the lower bound in Eq. \eqref{eq:m_lb} up to constants. Furthermore, $\{\csm_\rho(t)\}_{\rho=1}^d$ are monotone increasing functions and therefore $\csm(t)$ converges to a limit. This proves the first statement of the Lemma. 
To show the second statement of the Lemma, we move to bound $\csB$. By Lemma \ref{lem:ulb_m_v_rho} (Eq. \eqref{eq:m_rho_ub_final}    
 and      
Eq. \eqref{eq:v_rho_ub}), multiplying by $\lambda_\rho$ and averaging over $\rho$: 
\begin{align}
    \csB(t)&
    \le 
    C_w \gamma \frac{1}{d}\sum_{\rho}{\tilde{\mu}_{\rho} d}\int_{0}^{\omega_1(t)}e^{-\gamma\lambda_{\rho}{\left(\omega_{1}(t)-u\right)/C_w}}\dif u
    +
    {\gamma^{2}}\frac{1}{d}\sum_{\rho}\lambda_{\rho}^2\int_{0}^{\omega_1(t)}e^{-2\gamma\lambda_{\rho}{\left(\omega_{1}(t)-u\right)/C_w}}\dif u +O(d^{-1})
    \\\nonumber 
    &=C_w^2\gamma \int_{0}^{\omega_1(t)/C_w}F_{\mu}\left(u\right)\dif u
    +
   {\gamma^{2}}C_w\int_{0}^{\omega_1(t)/C_w}\cK_2\left(u\right)\dif u +O(d^{-1}).
\end{align}
It then follows that when $\kappa_\mu>1$,  
\begin{align}
    \csB(t)\le {C_w^2}{\gamma }\left( \|F_\mu\|_1 - g_F(\omega_1(t))\right) + \gamma ^2C_w\left(\|\cK_2\|_1- g_\cK(\omega_1(t))\right)\le {C_w^2}{\gamma } \|F_\mu\|_1  + \gamma ^2 C_w\|\cK_2\|_1
\end{align}
with $g_\cK(\omega_1(t)) = \left(\frac{\omega_{1}(t)}{C_w}\right)^{-\kappa_{2}+1}$ for $ \kappa_\mu>1$ and $g_\cK(\omega_1(t)) = e^{-2\gamma \frac{\omega_{1}(t)}{C_w}}$ when $K_1=K_2 =I_d.$
This then shows that both for $\kappa_\mu>1$ and when $K_1=K_2 =I_d,$ $\csB$ is bounded as stated in the second statement of the Lemma. Next, following Eq. \eqref{eq:m_ub} together with Lemma \ref{lem:risk_lub}   we have that 
\begin{align}
    \log(1+e^{-C_w\gamma\|F_\mu\|_1}
    )\le \csL(t).
\end{align}
As $\csm$ and $\csB$ are bounded, there exists $t_0>0$ such that for all $t>t_0,$ the loss is also bounded from above by Lemma \ref{lem:risk_lub}
\begin{align}
   \csL(t)\le  \log(1+e^{-(1-\frac{C_w^2}{2})\gamma\|F_\mu\|_1+\frac{\gamma^2C_w}{2}\|\cK_2\|_1}
    ).
\end{align}
This then completes the proof. 
\end{proof}

\begin{proof}[Proof of Proposition \ref{prop:power_law_bad_regime}] 
    Throughout the proof, we assume without loss of generality that $\csm(t)\ge0$
(otherwise, we flip the direction of $\mu$). 
We start with a lower bound on the overlap $\csm(t).$ Using Lemma \ref{lem:ulb_m_v_rho}, in particular Eq. \eqref{eq:m_rho_lb_final} it follows that 
    \[
\csm_{\rho}(t)\ge\gamma\tilde{\mu}_{\rho} d\int_{0}^{t}W_{1}(s)e^{-\gamma\lambda_{\rho}\left(\omega_{1}(t)-\omega_{1}(s)\right)}\dif s = \tilde{\mu}_{\rho} d\frac{(1-e^{-\gamma\lambda_{\rho}\omega_{1}(t)})}{\lambda_{\rho}}.
\]
 By Assumption \ref{ass:F_K_power_law_assmp} (or Assumption \ref{ass:power_law_mu_lam_con} and Lemma \ref{lem:K_mu_k_map_a_b}), $F_{\mu}(x)=\sum_{\rho}\tilde{\mu}_{\rho}e^{-\gamma\lambda_{\rho}x}\asymp x^{-\kappa_{\mu}}$
 (with $\kappa_{\mu}=\frac{\beta+1}{\alpha}$ in the case of Assumption \ref{ass:power_law_mu_lam_con}). 
Averaging over the eigenmodes in the above equation,
\begin{align}\label{eq:mt_lb_pl}
  \csm(t)&\ge\gamma\int_{0}^{t}W_{1}(s)F_{\mu}(\omega_{1}(t)-\omega_{1}(s))\dif s
  \\\nonumber &=\gamma \int_{0
  }^{\omega_{1}(t)}  F_\mu(\omega_{1}(t)-v) \dif v = \gamma \int_{0
  }^{\omega_{1}(t)}  F_\mu(y) \dif y \ge \gamma c_\mu^-\int_{\e}^{\omega_{1}(t)}y^{-\kappa_{\mu}}\dif y 
\end{align}
for some constant $c_\mu^->0$ and $\e>0.$ 
This then yields the following lower bounds:  
\begin{equation}
 \csm(t)=\Omega(\omega_{1}(t)^{1-\kappa_{\mu}}) \quad \text{for} \quad \kappa_{\mu}< 1, \quad \csm(t)=\Omega(\log(\omega_{1}(t)
 )) \quad 
\text{for} \quad  \kappa_{\mu}=1. \label{eq:m_lb_extrem_pl}  
\end{equation}
Next, using the lower bound in Lemma \ref{lem:w12bounds}
$$ \frac{\dif \omega_1(t)}{\dif t} \ge \frac{1}{1+e^{\csm(t)}}.$$ 
In addition, we note that a crude upper bound on $\csm$ is obtained from Eq. \eqref{eq:m(t)_rho}, $\csm(t) \le \gamma \|\mu\|^2\omega_1(t).$ 
Plugging this upper bound on $\csm(t)$ it then follows that: 
\begin{align}
 \omega_1(t) = \Omega(\log(t)). 
\end{align}
Plugging these in Eq. \eqref{eq:m_lb_extrem_pl}, the first claim of the Lemma is obtained. 

We proceed by deriving a tighter upper bound on $\csm(t)$. By Assumption \ref{ass:W_1W_2ab} and Lemma \ref{lem:ulb_m_v_rho} averaging over $\rho$: 
\begin{align}
   \csm(t)&\le \gamma \frac{1}{d}\sum_\rho {\tilde{\mu}_{\rho} d}\int_{0}^{\omega_1(t)}e^{-\gamma\lambda_{\rho}{\left(\omega_{1}(t)-u\right)/C_w}}\dif u  
   \\\nonumber &= \gamma  \int_{0
  }^{\omega_{1}(t)}  F_\mu((\omega_{1}(t)-v)/C_w) \dif v = \gamma C_w\int_{0
  }^{\omega_{1}(t)/C_w}  F_\mu(y) \dif y.
\end{align}
Assumption \ref{ass:F_K_power_law_assmp} on $F_\mu$ yields the following upper bounds that match the lower bounds on $\csm(t)$ up to multiplicative constants:  
\begin{equation}\label{eq:m_ub_extrem_pl}
 \csm(t)=O(\omega_{1}(t)^{1-\kappa_{\mu}}) \quad \text{for} \quad \kappa_{\mu}< 1, \quad \csm(t)=O(\log(\omega_{1}(t)
 )) \quad 
\text{for} \quad  \kappa_{\mu}=1.   
\end{equation}
We move to find a bound on $\csB$ in order to obtain the asymptotic behavior of $\omega_1(t)$ as a function of $t$. We will then use this relation to show that the risk is decreasing to zero. By Lemma \ref{lem:ulb_m_v_rho} (Eq. \eqref{eq:v_rho_ub} and Eq. \eqref{eq:m_rho_ub_final}) multiplying by $\lambda_\rho$ and summing over $\rho$: 
\begin{align}\label{eq:B_extrem_pl}
    \csB(t)&    
    \le \frac{1}{d}\sum_{\rho} \frac{\lambda_\rho}{\tilde{\mu}_\rho d} {\csm_\rho^2} +     {\gamma^2}\left(\lambda_{\rho}+\tilde{\mu}_{\rho}\right)\int_{0}^{\omega_1(t)} e^{-2\gamma\lambda_{\rho}\bar{c}(t)(\omega_{1}(t) - u)}\dif u.    
\end{align}
 Suppose Assumption \ref{ass:W_1W_2ab} holds and there exists $\gamma, t_0>0$ such that $C_w\le 2-\e$ for some $\e = \e(\gamma,\kappa_\mu)\in (0,1)$. By Lemma \ref{lem:ulb_m_v_rho} (Eq, \ref{eq:m_rho_ub_final}) it holds that for all $\rho\in[1,d]$,  
\begin{equation}
\csm_\rho(t)\le \gamma\tilde{\mu}_{\rho} d\int_{0}^{\omega_1(t)} e^{-\frac{\gamma\lambda_{\rho}}{2-\e}(\omega_{1}(t) - u)}\dif u   = (2-\e)\frac{\tilde{\mu}_{\rho} d }{\lambda_\rho}(1-e^{-\frac{\gamma\lambda_{\rho}}{2-\e}\omega_{1}(t) } )\le(2-\e)\frac{\tilde{\mu}_{\rho} d }{\lambda_\rho}. 
\end{equation}
The first term in Eq. \eqref{eq:B_extrem_pl} is then
$$\frac{1}{d}\sum_{\rho = 1}^{d} \csm_\rho^2 (t) \frac{\lambda_\rho}{\tilde{\mu}_{\rho}d} \le 2(1-\e/2)\frac{1}{d}\sum_{\rho = 1}^{d} \csm_\rho (t) = 2(1-\e/2)\csm(t). $$
We move to the second term in Eq. \eqref{eq:B_extrem_pl} 
\begin{align*}
     &{\gamma^2}\sum_\rho \lambda_{\rho}\left(\lambda_{\rho}+\tilde{\mu}_{\rho}\right)\int_{0}^{\omega_1(t)} e^{-2\gamma\lambda_{\rho}\bar{c}(t)(\omega_{1}(t) - u)}\dif u 
     \\ &=\gamma^2 C_w \frac{1}{d}\sum_{\rho} \lambda_\rho^2\int_{0}^{\omega_1(t)/C_w}e^{-2\gamma\lambda_{\rho}y}\dif u +O(d^{-1}) = \gamma^2 C_w\int_{0}^{\omega_1(t)/C_w}\cK_2(y) \dif y +O(d^{-1}).
\end{align*}
Combining the above, we therefore have that, 
\begin{align}\label{eq:B_up_k_lt_1}
    \csB(t)&    
    \le  2(1-\e/2)\csm(t) + \gamma \frac{1}{d}\sum_{\rho}\frac{ \lambda_\rho^2}{\tilde{\mu}_\rho d} \csm_\rho +O(d^{-1}).
\end{align}
In addition, the difference 
\begin{align*} 
\csm(t)-\frac{\csB(t)}{2}&
\geq  \frac{\e}{2}\csm(t)-\frac{\gamma^2}{2} C_w\int_0^{\omega_1(t)/C_w} \cK_2(z) \dif z+ O(d^{-1}).  
\end{align*} 

By Assumption, $\beta < 2\a$, which implies that $\kappa_\mu<\kappa_2$. This then shows that the second term is sub-leading with respect to the first. Then, by assumption \ref{ass:F_K_power_law_assmp} 
\begin{align}\label{eq:m_B_2_k_lth_1} 
\csm(t)-\frac{\csB(t)}{2}=\Omega(\omega_1(t)^{1-\kappa_\mu}) \quad \text{for} \quad \kappa_\mu<1, \quad  \csm(t)-\frac{\csB(t)}{2}=\Omega(\log(\omega_1(t))) \quad  \text{for}\quad \kappa_\mu=1.
\end{align}
The above bound together with Eq. \eqref{eq:B_up_k_lt_1}, $\b< 2\a$, and the lower bound on $\csV_\rho(t)$ in Eq. \eqref{eq:v_rho_ub} we obtain that $\csB(t)\asymp \csm(t).$
Next, following the upper and lower bounds in Lemma \ref{lem:w12bounds}, this, together with Eq. \eqref{eq:m_B_2_k_lth_1}, yields the following inequalities 
$$\frac{2}{3+e^{c\csm(t)}}\ge \frac{\dif \omega_1(t)}{\dif t}
\ge \frac{1}{1+e^{\csm(t)}}
$$  
for some $c>0$. Next, using the upper bound \eqref{eq:m_ub_extrem_pl} and lower bound \eqref{eq:m_lb_extrem_pl} on $\csm(t).$ We obtain that $$\omega_1(t) \asymp (\log(t))^{\frac{1}{1-\kappa_\mu}} 
\quad \text{for} \quad \kappa_\mu<1, \quad  t^{a_1}<\omega_1(t)<t^{a_2} \quad  \text{for}\quad \kappa_\mu=1 \text{ and } a_1, a_2\in (0,1),$$ 
where we used the fact that $\int_0^{\omega_1(t)}e^{cu^{1-\kappa_\mu}}\dif u  =  \frac{(-c)^{-\frac{1}{1-\kappa_\mu}}}{1-\kappa_\mu}\left(\Gamma(\frac{1}{1-\kappa_\mu}) - \Gamma(\frac{1}{1-\kappa_\mu},-c\omega(t)^{1-\kappa_\mu}) \right) $ and that $\Gamma(s,x)\sim x^{s-1}e^{-x}$ as $x\to\infty$ with $\Gamma(s,x)$ being the upper incomplete Gamma function.
Together with \eqref{eq:m_lb_extrem_pl} and \eqref{eq:m_ub_extrem_pl} we obtain that $\csm(t)\asymp \log(t)$ as claimed. 
Finally, using Lemma \ref{lem:risk_lub}, \eqref{eq:m_B_2_k_lth_1}, \eqref{eq:m_ub_extrem_pl} and \eqref{eq:m_lb_extrem_pl} for large enough $t$ the risk decreases to zero for some constants $c_1(\gamma,\kappa_\mu), c_2(\gamma,\kappa_\mu)>0$,  
\[
t^{-c_1}\le \csL(t)\le t^{-c_2}.\]
\end{proof}

\subsection{Proof for Identity model}\label{sec:identity_proof}
In this section we present a proof of Proposition \ref{prop:identity_model} that does not rely on Assumption \ref{ass:W_1W_2ab}. We suspect that this bound is too loose, which is why the result holds only for a small range of learning rates and $\|\mu\|$. 
The proof of the identity model that uses Assumption \ref{ass:W_1W_2ab} is presented in Proposition \ref{prop:power_law_good_regime} ($\a =0$). 
This provides larger ranges for the learning rates and $\mu$. We accompany these two results with a numerical simulation showing that the constant in Assumption \ref{ass:W_1W_2ab} and the decrease of the loss indeed change significantly with the learning rate (Figure \ref{fig:combined_plots_gamma_ID_a_loss}).
\begin{figure}[t]
\centering
\begin{minipage}[t]{0.4\textwidth}
    \centering
    \includegraphics[width=\linewidth,keepaspectratio]{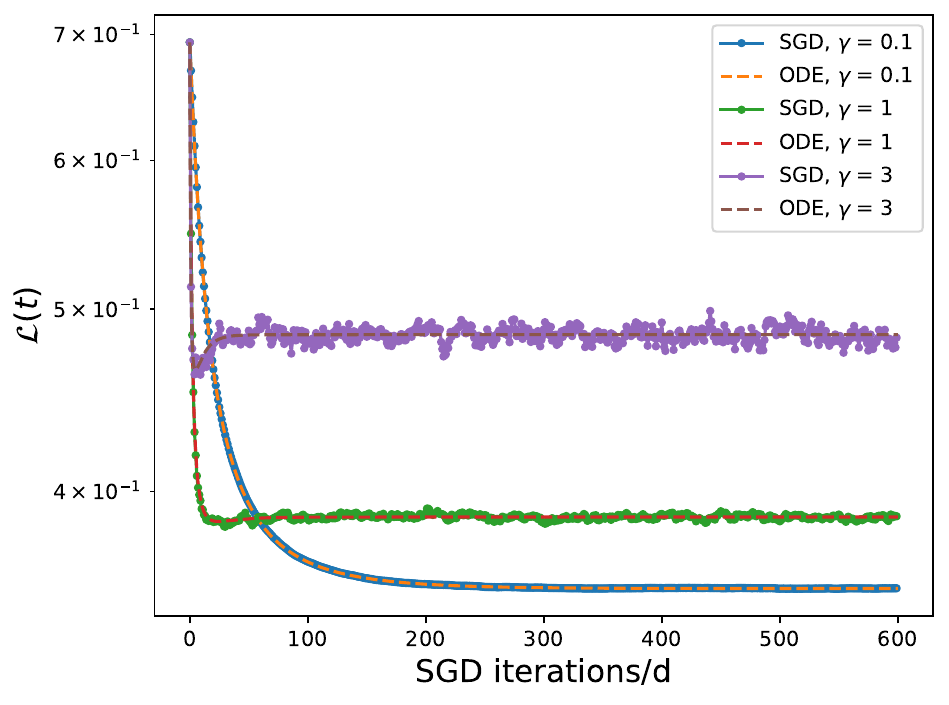}
    \medskip
    \textit{(a) Loss vs. SGD iterations/$d$.}
\end{minipage}
\begin{minipage}[t]{0.4\textwidth}
    \centering
    \includegraphics[width=\linewidth,keepaspectratio]{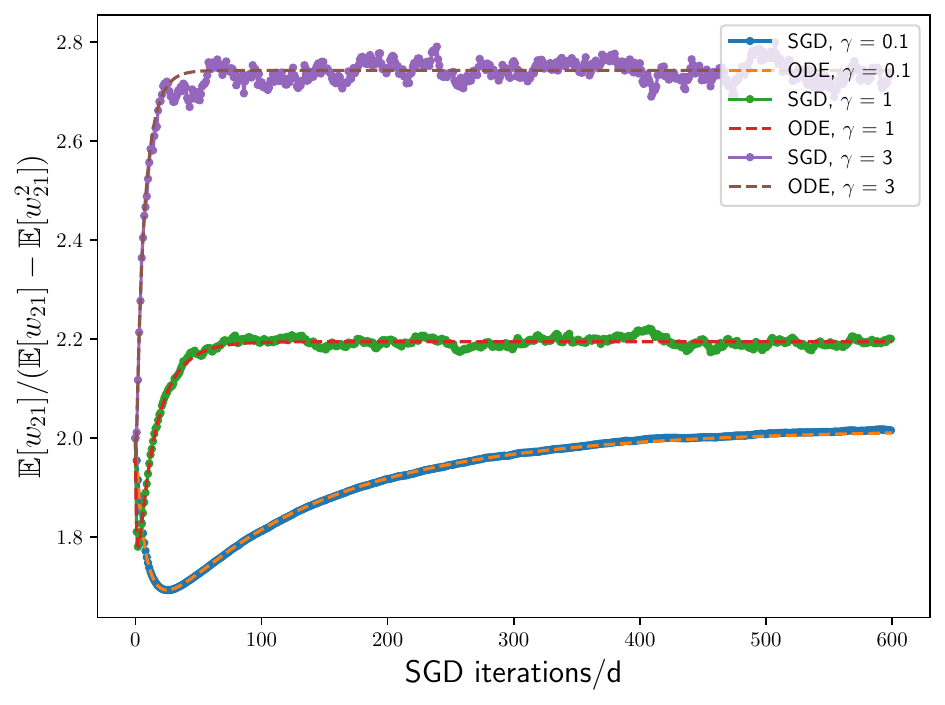}
    \medskip
    \textit{(b) $a(t)$ vs. SGD iterations/$d$.}
\end{minipage}

\caption{\textbf{Learning curve and $a(t)$ for the identity model.} Comparison of simulation results (loss and $a(t) = \frac{\EE w_{12}}{\EE w_{12}-\EE w_{12}^2}$) for different learning rates $\gamma$.}
\label{fig:combined_plots_gamma_ID_a_loss}
\vspace{-0.5cm}
\end{figure}

\begin{proposition}[Identity covariance]\label{prop:identity_model}
Suppose $p_1=\frac{1}{2}$, $K_1 =K_2 = I_d$ 
, and
the learning rate satisfies the relations $\|\mu\|^2\le \frac{(\gamma -1)^2}{8(2-\gamma)}$ and $\gamma\le 1.$  
Then for $t\ge 1$  
\begin{enumerate}
    \item $\csm(t)\to C\|\mu\|^2$ for some $C>1$
    \item  $\csV(t) \le \frac{1-\gamma - 8\|\mu\|^2 + \sqrt{(\gamma-1)^2-8\|\mu\|^2(2-\gamma)}}{2\|\mu\|^2}$   
    \item $\csL(t)$ converges to a non-zero constant.
\end{enumerate}
\end{proposition}
\begin{remark} In the proof below, we demonstrate that the conditions $\|\mu\|^2\le \frac{(\gamma -1)^2}{8(2-\gamma)}$ and $\gamma\le 1$ are sufficient to guarantee that $\dif\csV/\dif t\leq0$ which, in turn, implies the three results stated in the proposition.  This is overly restrictive in the sense that one could likely allow a larger range of $\gamma,\mu$ in which $\dif\csV/\dif t>0$ but nevertheless $\csm,\csV$ remain bounded.  Our goal with this proposition is not to obtain the most general result, but simply to illustrate that one can find a range of allowable parameters under which bounds can be proven without resorting to Assumption \ref{ass:W_1W_2ab}.
\end{remark}
\begin{proof}[Proof of Proposition \ref{prop:identity_model}]
 The equations for identity covariance are simplified to two equations for the norm of the iterates and the overlap 
\begin{align}\label{eq:v_identity}
\frac{\dif \csV}{\dif t}&=-2\gamma\csV(W_{1}-W_{2})+2\gamma \csm W_{1}+\gamma^{2}W_{2}+\gamma^{2}\frac{\|\mu\|^2}{d}W_{2}\\
\frac{\dif \csm}{\dif t}&=-\gamma \csm (W_{1}-W_{2})+\gamma \|\mu\|^2 W_{1}\label{eq:m_identity}.
\end{align}
Solving for $\csm$ we obtain the following lower bound: 
\begin{align}\label{eq:m_sol_identity}
\csm(t)  = \gamma \|\mu\|^2 \int_{0}^{t}\dot{\omega_{1}}(s)e^{-\gamma \int_s^t (\EE[w_{12}]- \EE[w_{12}^2]) \dif \tau}\dif s\ge \|\mu\|^2 (1- e^{-\gamma \omega_1 (t)}). 
\end{align}
As we initialize $0\leq\csm(0)\leq \|\mu\|^2,$ it follows that for all $t$ and all $\gamma,$
\begin{align} \label{eq:stab_m}
    \csm(t) \le \frac{\|\mu\|^2 W_1(t)}{W_1(t)-W_2(t)}.
\end{align}
This is because the curve $y(t)=\frac{\|\mu\|^2 W_1(t)}{W_1(t)-W_2(t)}$ is an attracting stable curve for $\csm(t)$.
We note that $W_1(t)-W_2(t)> 0,$ and $\frac{ W_1(t)}{W_1(t)-W_2(t)}> 1$ by construction for all $t> 0.$ By Lemma \ref{lem:w12_2_ub} we also have that for all $t> 0$ \begin{align}\label{eq:GP_W_12}
 \frac{ W_1(t)}{W_1(t)-W_2(t)} \le \frac{1}{2}\csV(t)+2.  
\end{align} 
Together with Eq. \eqref{eq:stab_m}, we conclude that \begin{align}\label{eq:m_V_linear_relation}
    \frac{\csm(t)}{\|\mu\|^2}\le  \frac{1}{2}\csV(t)+2  .
\end{align}
Next, we impose a stability condition for $\csV(t)$ in \eqref{eq:v_identity}. In particular, we want to find a sufficient condition such that $\dif\csV(t)/\dif t\leq0$ for all $t\ge t_0$ for some $t_0\ge 0$. This will show that both $\csm(t)$ and $\csV(t)$ are bounded by a constant independent of $t$.  
We observe that
\begin{align}  \label{eq:stab_V}  
&\left(2\gamma\csm W_1+\gamma^2 W_2+\gamma^2\frac{\|\mu\|^2}{d}W_2\right)\Big/\Big(2\gamma(W_1-W_2)\Big)\\\nonumber
&=\csm(t)\frac{W_1(t)}{W_1(t)-W_2(t)}+\frac{1}{2}\gamma(1+ \frac{\|\mu\|^2}{d})\frac{W_2(t)}{W_1(t)-W_2(t)}
\\\nonumber &\le 
\left(\frac{W_{1}}{W_{1}(t)-W_{2}(t)}\right)^2\|\mu\|^2+ \frac{1}{2}\gamma(1+ \frac{\|\mu\|^2}{d})\left(\frac{W_1(t)}{W_1(t)-W_2(t)} -1\right)
\\\nonumber &\le 
4\left(\frac{\csV(t)}{4}+1\right)^2{\|\mu\|^2} + \gamma(1+ \frac{\|\mu\|^2}{d}) (\frac{\csV(t)}{4}+\frac{1}{2})
\end{align}
where the first transition is using the condition on $\csm(t)$ in Eq. \eqref{eq:stab_m}, the second transition is using Eq. \eqref{eq:GP_W_12}.
Next, we impose a stability condition on $\csV$ to be non-exploding for any $t\ge t_0$, following Eq. \eqref{eq:m_identity}
\begin{align}
\csV(t)\ge    4\left(\frac{\csV(t)}{4}+1\right)^2\|\mu\|^2 + \gamma(1+ \frac{\|\mu\|^2}{d}) (\frac{\csV(t)}{4}+\frac{1}{2}).
\end{align}
Rearranging the terms, and denoting by $z = \frac{\csV(t)}{4}+1,$ 
We conclude that for large enough $d$
\begin{align}
      4z^2\|\mu\|^2+(\gamma-1)z +1 -\frac{\gamma}{2}  \le  0 .
\end{align}
Solving for $z\ge 0$ we obtain that, 
\begin{align}
    \csV(t) \le \frac{1-\gamma - 8\|\mu\|^2 + \sqrt{(\gamma-1)^2-8\|\mu\|^2(2-\gamma)}}{2\|\mu\|^2}
\end{align}
provided that $\|\mu\|^2\le \frac{(\gamma -1)^2}{8(2-\gamma)}$ and that $\gamma\le 1$ is such that the above term is positive. This condition implies also that $\|\mu\|^2 \le 1/4.$ 
Plugging it back in Eq. \eqref{eq:m_sol_identity}, we obtain that $\csm$ converges to a fraction of $\|\mu\|^2$. 
Similarly, solving Eq. \eqref{eq:v_identity}, we obtain that $\csV$ also converges  
to an $O(1)$ constant which depends on $\|\mu\|^2$. Here, the convergence of $\csm(t)$ is due to the fact that we have an upper bound on $\csm$ and it is a monotonically increasing function.  Likewise, $\csV$ is a positive function and we are working in a range of parameters where, beyond some $t_0$, it is monotonically decreasing.
This completes the proof of the first two statements of the Proposition. The third statement follows from Lemma \ref{lem:risk_lub}. 
\end{proof}

\subsection{Technical deterministic lemmas}
\begin{lemma}[Risk bounds] \label{lem:risk_lub}
In the binary logistic regression setting for $p_1 = 1/2$, if $\csB_1(t)=\csB_2(t)=\csB(t)$ (e.g. in the case of $K_1=K_2$), then the risk is bounded from below and above as follows: \begin{equation}
    \log(1+e^{-\csm(t)})\leq\csL(t)\leq\log(1+e^{-\csm(t)+\csB(t)/2}).
\end{equation} 
\end{lemma} 
\begin{proof}
  The cross-entropy loss for a Gaussian mixture is as follows: 
\begin{align}
\csL(t)&= \EE[f_i(r_i)]=-p_1 \csm(t)+p_1\EE_z[\log(1+e^{\csm(t)+\sqrt{\csB_1(t)}z})] \\\nonumber &+p_2\EE_z[\log(1+e^{-\csm(t)+\sqrt{\csB_2(t)}z})]   .
\end{align}
In the case where $\csB_1=\csB_2=\csB$,  the loss function simplifies to 
\begin{align}
\csL(t)&= \EE_z[\log(1+e^{-\csm(t)+\sqrt{\csB(t)}z})]   .
\end{align}
Since the function $\log x$ is concave, while the function $\log(1+ce^{kx})$ with $c>0$ is convex.  Applying Jensen's inequality with each of these two functions, we obtain upper and lower bounds, respectively, as stated.    
\end{proof}

\begin{lemma}\label{lem:w12_2_ub} Define $w_{12}=1/(e^{bz+m}+1)$ with $z\sim \cN(0,1)$, and $m,b\in \R$, then
\begin{equation}
   \label{eq:w12_1_2_bound}
\EE [w_{12}^2]\le 
\frac{b^2+2}{b^2+4}
\EE [w_{12}]
\end{equation}
\end{lemma}
\begin{proof}

Using the Gaussian Poincaré inequality, we have that 
\begin{align}\label{eq:poin_w12}
\text{Var}(w_{12}) = \EE [w_{12}^2]-(\EE [w_{12}])^2\le\EE [(w_{12}')^2],
\end{align}
where $w_{12}'= -b \frac{e^{bz+m}}{(e^{bz+m}+1)^2}= -bw_{12}(1-w_{12})$, taking the square and expectation: 
$$\EE[(w_{12}')^2]= b^2 \EE [(w_{12}(1-w_{12}))^2] \le \frac{b^2}{4} \EE [(w_{12}(1-w_{12}))] = \frac{b^2}{4} (\EE [w_{12}]-\EE [w_{12}^2])$$
where the second transition is since $w_{12}\le 1$ and therefore $w_{12}(1-w_{12})\le 1/4.$ 
Plugging in Eq. \eqref{eq:poin_w12} and using the bound $\EE[w_{12}]\le \frac{1}{2}$, 
\begin{align}
\EE [w_{12}^2]\le \frac{b^2}{4} (\EE [w_{12}]-\EE [w_{12}^2])+\frac{1}{2}\EE [w_{12}],
\end{align}
Rearranging terms, we obtain the statement of the lemma.
\end{proof}
The next Lemma studies the regime at which $w$ concentrates around its expectation. In this regime:  
\begin{lemma}[Power-law mapping]\label{lem:K_mu_k_map_a_b}
 Suppose $\a>0,\, \b\ge 0$, and \begin{equation}\label{eq:power_law_mu_lam_con}
    \lambda_\rho=\left(\frac{\rho}{d}\right)^\a;\qquad
    \tilde{\mu}_\rho=\frac{1}{d}\left(\frac{\rho}{d}\right)^\b.
\end{equation}   
Then Assumption \ref{ass:F_K_power_law_assmp} is satisfied with $\kappa_\mu = \frac{\beta+1}{\alpha},$ and $\kappa = \frac{1}{\alpha} +1.$ When $\alpha = 0$ (identity matrix),
$   
\cK(x) = e^{-2\gamma x},   
$ and $F_{\mu}(x) = \|\mu\|^2 e^{- \gamma x}.$
\end{lemma}
\begin{proof}  Following the definition of $F_\mu(x) 
$ 
\begin{align}   F_\mu(x)&=\sum_{\rho=1}^d\tilde{\mu}_\rho\exp(-\gamma x\lambda_\rho) =
    \int_0^1y^\b\exp(-\gamma xy^\a)\dif y +O(d^{-1})
    \\ \nonumber &=(\gamma x)^{-\frac{\b+1}{\a}}\boldsymbol{\gamma}\left(\frac{\b+1}{\a},\gamma x\right) +O(d^{-1}).
\end{align}
where $\boldsymbol{\gamma}$ denotes the incomplete gamma function.  Since $\boldsymbol{\gamma}(s,t)\to\Gamma(s)$ as $t\to\infty$, we have $F_\mu(x)
\asymp (\gamma x)^{-\frac{\b+1}{\a}}$ in the power-law case.  In particular, this means that there is a transition in the integrability of $F_\mu$ when $\a=\b+1$.  Similarly for \begin{align}    
\cK(x) &= \frac{1}{d}\sum_{\rho}\lambda_{\rho}e^{-2\gamma\lambda_{\rho}x} = 
\int_0^1y^\alpha\exp(-2\gamma xy^\a)\dif y +O(d^{-1}) \\ \nonumber & = (\gamma x)^{-\frac{\alpha+1}{\a}}\boldsymbol{\gamma}\left(\frac{\alpha+1}{\a},\gamma x\right) +O(d^{-1}).
\end{align} 
The claim for $\alpha = 0$ is by definition. 
\end{proof}

\begin{lemma}\label{lem:ulb_m_v_rho}
Suppose $K_1=K_2 = K$, $X_0 = 0$. Then, for any $\rho\in [d],$   \begin{align}     \csm_\rho(t) &\ge \frac{d\tilde{\mu}_{\rho}}{\lambda_\rho}(1-e^{-\gamma \lambda_\rho\omega_1(t)})  \label{eq:m_rho_lb_final}  \\       
\label{eq:v_rho_ub}
\frac{\csm_{\rho}^2(t)}{\tilde{\mu}_\rho d}\le \csV_{\rho}(t)
&\le \frac{1}{\tilde{\mu}_\rho d} \csm_{\rho}^2(t)
+{\gamma^2}\left(\lambda_{\rho}+\tilde{\mu}_{\rho}\right)\int_{0}^{\omega_1(t)} e^{-2\gamma\lambda_{\rho}\bar{c}(t)(\omega_{1}(t) - u)}\dif u
\\\nonumber &\le \frac{1}{\tilde{\mu}_\rho d} \csm_{\rho}^2(t)
+\frac{\gamma}{\tilde{\mu}_{\rho} d}\left(\lambda_{\rho}+\tilde{\mu}_{\rho}\right)\csm_{\rho}(t).  
\end{align} 
Denote by $c(t)\in (0,1)$ some constant such that 
for all $t\ge 0$, it holds that $\E w_{12}^2(t) \le c(t) \E w_{12}(t)$. 
Then,  
\begin{align}
\csm_\rho(t)&\le \gamma\tilde{\mu}_{\rho} d\int_{0}^{\omega_1(t)} e^{-\gamma\lambda_{\rho}\bar{c}(t)(\omega_{1}(t) - u)}\dif u  \label{eq:m_rho_ub_final} 
\end{align} 
with $\bar{c}(t) = 1 - \sup_{s\in [t_0,t]} c(s)$. 
If in addition Assumption \ref{ass:W_1W_2ab} holds then,  $\bar{c}(t) =1/C_w  $ independent of $t$. 
\end{lemma}
\begin{proof}[Proof of Lemma \ref{lem:ulb_m_v_rho}]
Let us start by analyzing the overlap (first statement of the Proposition). $\csm(t)$ projected into the eigenbasis of $K$, satisfies Eq. \eqref{eq:m_rho_one_K}, its solution is then:
\begin{align}\label{eq:m(t)_rho}
\csm_{\rho}(t)=\gamma\tilde{\mu}_{\rho} d\int_{0}^{t}\dot{\omega_{1}}(s)e^{-\gamma\lambda_{\rho}\int_{s}^{t} (\EE[w_{12}](\tau)- \EE[w_{12}^2](\tau)) \dif \tau} \dif s
\end{align}
The lower bound is obtained by $e^{\gamma(\omega_{2}(t)-\omega_{2}(s))}\ge 1,$ changing variables from $s\to u_1 = \omega_1(s)$ and integrating over time. 
Next, we derive an upper bound on $\csm_\rho(t).$ 
By the assumption of the lemma, we have that for any $s>T$
\begin{align}\label{eq:w_c_bound}
 \int_{s}^{t} (\EE[w_{12}](\tau)- \EE[w_{12}^2](\tau)) \dif \tau \ge \bar{c}(t){(\omega_1(t)-\omega_1(s))}  
\end{align}
where $\bar{c}(t) \defas 1-\sup_{\tau\in [T,t]}c(\tau).$
Plugging the above inequality in Eq. \eqref{eq:m(t)_rho},
and changing integration variables $s\to u = \omega_1(s) = \int_0^t \EE[w_{12}](s)\dif s$ we obtain the upper bound, 
\begin{align}\label{eq:m_rho_ub_p2}
\csm_\rho(t)=&\gamma\tilde{\mu}_{\rho} d\int_{0}^{t}\dot{\omega_{1}}(s)e^{-\gamma\lambda_{\rho}\int_s^t (\EE[w_{12}]- \EE[w_{12}^2]) \dif \tau}\dif s
\\\nonumber & \le \gamma\tilde{\mu}_{\rho} d\int_{0}^{\omega_1(t)} e^{-\gamma\lambda_{\rho}\bar{c}(t)(\omega_{1}(t) - u)}\dif u  
\end{align}
Next, we turn to bound the projected norm. The implicit solution to Eq. \eqref{eq:v_rho_one_K}:
\begin{multline}
\label{eq:v_rho_sol}
\csV_{\rho}(t)=2\gamma\int_{0}^{t}\csm_{\rho}(s)\dot{\omega_{1}}(s)e^{-2\gamma\lambda_{\rho}\int_s^t (\EE[w_{12}]- \EE[w_{12}^2]) \dif \tau}\dif s\\
+\gamma^{2}\left(\lambda_{\rho}+\tilde{\mu}_{\rho}\right)\int_{0}^{t}\dot{\omega_{2}}(s)e^{-2\gamma\lambda_{\rho}\int_s^t (\EE[w_{12}]- \EE[w_{12}^2]) \dif \tau}\dif s
\end{multline}
Next, we move to derive an upper bound on $\csV_\rho(t)$,  
we use the upper bound on $\csm_\rho$, in Eq. \eqref{eq:m_rho_ub_p2},
\begin{align}
\csm_{\rho}(t)&\le\frac{\tilde{\mu}_{\rho} d}{\lambda_\rho\bar{c}(t)}(1-e^{-\gamma\lambda_{\rho}\bar{c}(t)\omega_1(t) }) 
\end{align}
We start by analyzing the first term using the equation for $\csm_\rho$ and changing the order of integration, we obtain that, 
\begin{align}
    &2\gamma\int_{0}^{t}\csm_{\rho}(s)\dot{\omega_{1}}(s)e^{-2\gamma\lambda_{\rho}\int_s^t (\EE[w_{12}]- \EE[w_{12}^2]) \dif \tau}\dif s
\\\nonumber &=  2\gamma^2 d \tilde{\mu}_{\rho}\int_{0}^{t}\int_0^s\dot{\omega_{1}}(x)e^{-\gamma\lambda_{\rho}\int_x^s (\EE[w_{12}]- \EE[w_{12}^2]) \dif y} \dot{\omega_{1}}(s)e^{-\gamma\lambda_{\rho}\int_s^t (\EE[w_{12}]- \EE[w_{12}^2]) \dif \tau}\dif s  \\\nonumber &= \frac{\csm_{\rho}^2(t)}{\tilde{\mu}_\rho d}.
\end{align}
Plugging it back into the equation for $\csV_\rho(t)$ and using $\dot{\omega}_2(r)\le \dot{\omega}_1(r)$, yields the upper bound on $\csV_\rho$ in the lemma.
As the second term in Eq. \eqref{eq:v_rho_sol} is positive, we observe also that $\csm_{\rho}^2(t)\le {\tilde{\mu}_\rho} d\csV_\rho(t)$. This completes the proof of the Lemma. 
\end{proof}

\end{document}